\documentclass{article}
\usepackage{mathrsfs}
\usepackage{amsthm}
\usepackage{main,times}

\usepackage{amsmath,amsfonts,bm}

\def\eqref#1{equation~\ref{#1}}

\def\1{\bm{1}}

\DeclareMathAlphabet{\mathsfit}{\encodingdefault}{\sfdefault}{m}{sl}
\SetMathAlphabet{\mathsfit}{bold}{\encodingdefault}{\sfdefault}{bx}{n}

\newcommand{\E}{\mathbb{E}}

\newcommand{\R}{\mathbb{R}}

\newcommand{\Var}{\mathrm{Var}}

\newcommand{\Cov}{\mathrm{Cov}}

\usepackage{url}
\usepackage{amssymb}
\usepackage{enumitem}
\usepackage{float}
\usepackage{graphicx}
\graphicspath{{figures/}}
\usepackage{hyperref}
\newcommand{\Prob}{\mathbb{P}}
\newcommand{\Unif}{\operatorname{Unif}}

\newtheorem{theorem}{Theorem}[section]
\newtheorem{lemma}[theorem]{Lemma}
\newtheorem{proposition}[theorem]{Proposition}
\newtheorem{corollary}[theorem]{Corollary}
\newtheorem{remark}[theorem]{Remark}
\newtheorem{assumption}[theorem]{Assumption}

\title{Learning Decision-Stump Thresholds\\in Context: Dynamics of Softmax Attention}

\author{
Hong Ha Le$^{1}$ \quad
Jackie Lok$^{2}$ \quad
Atsushi Nitanda$^{3}$ \quad
Yan Shuo Tan$^{1}$\thanks{Corresponding author: \texttt{yanshuo@nus.edu.sg}.}\\
\normalfont
$^{1}$National University of Singapore\\
$^{2}$Princeton University\\
$^{3}$Nanyang Technological University
}

\iclrfinalcopy 
\begin{document}

\maketitle
\lhead{}
\renewcommand{\headrulewidth}{0pt}

\begin{abstract}
Estimating a decision threshold requires locating observations near an unknown boundary. We study how gradient-based pretraining learns this statistical rule in a two-parameter softmax-attention model with a fixed feature and inequality direction. Pretraining uses labeled contexts and their true thresholds; a fresh threshold must be inferred from context alone. Under a large-resolution initialization, constant-step gradient descent on $m$ tasks with $n$ examples each produces a frozen estimator with error $\widetilde O((m\wedge n)^{-1}+N^{-1})$ for each fixed interior threshold and every fresh-context size $N$. The two terms separate finite-pretraining accuracy from fresh-context localization. The mechanism is coordinated parameter divergence: population training calibrates the relative label and feature scores, then increases the attention scale as $t^{1/4}$, giving population threshold error $O(t^{-1/4})$. To transfer this mechanism to a fixed finite corpus, we control gradient errors relative to the shrinking directions of progress at successive parameter scales. This certifies a growing training interval without requiring long-time tracking of the population trajectory. We also identify the boundary limitation of the one-head model and explain statistically what a reflected symmetrization could achieve.
\end{abstract}

\section{Introduction}
In-context learning (ICL) is the ability of a pretrained model to perform a new task from examples supplied in its input, without updating its parameters \citep{brown2020language}.
\citet{garg2022can} cast ICL as a statistical learning problem: pretraining across many datasets produces a reusable prediction \emph{procedure}, which is applied to a new dataset through the context alone.
In their formulation, as in much of the subsequent theory, each task is a supervised-prediction problem: the context contains labeled examples, and the model predicts the response at a new query point.
A growing body of work explains how gradient-based pretraining acquires such procedures, characterizing convergence, phases of learning, and the evolution of the in-context predictor for linear regression and, more recently, induction and sequence tasks \citep{zhang2024trained,chen2024training,zhang2025dynamics,chen2024unveiling,musat2025emergence}.

The same template applies to estimating a parameter of the data-generating process.
In \emph{amortized estimation}, one repeatedly draws a parameter from a prior and a dataset given that parameter, trains a single network to map datasets to parameter estimates, and applies the frozen network to new datasets \citep{zammitmangion2025neural}.
Under squared loss, the population objective is the Bayes risk, so the network approximates the posterior mean.
Neural Bayes estimators use permutation-invariant architectures for this purpose \citep{sainsburydale2024likelihood}; prior-fitted networks apply the same idea to predictive distributions and underlie tabular foundation models \citep{muller2025priorfitted,hollmann2025accurate}; and recent work estimates dataset-level quantities such as mutual information with permutation-invariant neural models \citep{peyrard2025meta}, or uses TabPFN for semi-supervised parameter estimation \citep{zhang2025tabpfn}.
A trained amortized estimator is an in-context estimator: it infers a latent parameter from its input without parameter updates.

Recent theory for amortized estimators primarily studies the statistical properties of the trained estimator rather than its optimization dynamics. \citet{roedder2025theoretical} decompose neural-estimator risk into Bayes, approximation, generalization, and algorithm-dependent terms, while leaving the algorithm-dependent term outside their analysis; related work studies the statistical accuracy of neural posterior and likelihood estimators \citep{frazier2024statistical}. The ICL training-dynamics works discussed above instead concern supervised prediction. \emph{How gradient-based pretraining constructs an amortized estimator, and how the number and size of the simulated datasets limit its accuracy, have, to our knowledge, not been analyzed.}

We study this question for decision thresholds.
Thresholds define the splits of decision trees, which are used both as interpretable models and as building blocks of accurate tabular ensembles \citep{tan2022cautionary,erickson2025tabarena}.
Tabular foundation models (TFMs) are now competitive with tree ensembles on standard tabular benchmarks, particularly on smaller datasets \citep{erickson2025tabarena,hollmann2025accurate}, yet their synthetic pretraining priors themselves contain decision trees and tree ensembles \citep{hollmann2025accurate,qu2025tabicl,zhang2025mitra}.
This motivates isolating computations that locate boundaries induced by latent splits.
We isolate one such computation: estimating a single threshold when the split feature and inequality direction are known.

Specifically, we consider the following model.
Given a threshold $\alpha \in (0, 1)$, a task supplies a context $D_{N,\alpha}=\{(X_i,Y_i)\}_{i=1}^N$ with $N$ independent features $X_i\sim\Unif(0,1)$ and labels $Y_i=\mathbf 1\{X_i\le\alpha\}$.
Write $M_N^+=\max\{X_i:Y_i=1\}$, with value zero if no label is positive. For $0\le\varepsilon<\alpha$,
$\Prob(\alpha-M_N^+>\varepsilon)=(1-\varepsilon)^N\le e^{-N\varepsilon}$, and hence, $N(\alpha-M_N^+)\Rightarrow\operatorname{Exp}(1)$.
Thus, the estimator $M_N^+$ localizes the threshold at the nonregular $N^{-1}$ scale, faster than the $N^{-1/2}$ scale typical of regular smooth parametric models \citep{hansen2000sample}.
Because the benchmark estimator is explicit, the setting lets us compare what gradient training learns against a known statistical rule.
We approximate this estimator with a single softmax-attention head:
\begin{equation}
\label{eq:intro-estimator}
F_N(\lambda,\mu;D_{N,\alpha})
    =\frac{\sum_{i=1}^N X_i e^{\lambda Y_i+\mu X_i}}{\sum_{i=1}^N e^{\lambda Y_i+\mu X_i}}.
\end{equation}
The parameters $\lambda$ and $\mu$ favor positive labels and larger feature values, respectively, and will be trained by gradient descent on simulated training tasks. Once the relative scale $\lambda / \mu$ is calibrated, increasing their common scale concentrates attention on the largest positive observation.

Implementing a sharp statistical rule with smooth attention requires increasing parameter scale. In the infinite-context limit, within-class concentration leaves $O(\mu^{-1})$ threshold bias; in finite contexts, the output must also approximate $M_N^+$.
Thus, improving resolution requires parameter growth even as the gradient driving it weakens, creating a tension between statistical accuracy and optimization with finite pretraining data.
In our model, it leads to two concrete questions:
(1) Does gradient descent coordinate the relative scores while continuing to increase the attention scale?
(2) As the useful gradient signal vanishes, how far can this sharpening process remain reliable when the population gradient is replaced by an empirical one computed from only $m$ tasks with $n$ observations each?

Our contributions answer both questions:
\begin{itemize}
\item \textbf{A learned amortized estimator with an explicit budget.}
Constant-step gradient descent on $m$ simulated tasks with $n$ observations each learns a frozen estimator with error $\widetilde O((m \wedge n)^{-1}+N^{-1})$ for every fresh-context size $N$ and each fixed threshold separated from the lower endpoint of the task range (Theorem~\ref{thm:learned-estimation}).
The first error term is the finite-pretraining error, and the second is the statistical error of the fresh context.
The bound attains the $N^{-1}$ benchmark, up to logarithmic factors, when $m\wedge n=\widetilde\Omega(N)$; in particular, the certified accuracy is limited by the simulated context size $n$, however large the fresh context.

\item \textbf{Balance, then sharpen.}
Population gradient flow first moves the ratio $\lambda/\mu$ to a boundary-aligned value, then increases $\mu$ as $t^{1/4}$.
The estimator converges uniformly at rate $O(t^{-1/4})$ while both parameters $\lambda$ and $\mu$ diverge (Theorem~\ref{thm:main-population}).
The radial drift $\dot{\mu}(t)$ scales as the product of the $\mu^{-1}$ bias and the $\mu^{-2}$ tilted variance, which explains why sharpening never stops but steadily slows.
This mechanism remains stable under discretization, and therefore transfers to population gradient descent with sufficiently small step size.

\item \textbf{Finite-corpus control along a diverging trajectory.}
Normalizing gradient errors by the shrinking progress margins at each parameter scale certifies continued sharpening for empirical gradient descent up to scale $(m \wedge n)/\log(mn)$ (Corollary~\ref{cor:finite-summary-v3}).
The argument follows the accuracy-relevant properties while avoiding full population-trajectory tracking over the growing horizon.
\end{itemize}

The paper is structured as follows.
The model is precisely introduced in Section~\ref{sec:model}.
Next, we describe our main results: the statistical results of the pretrained model are presented in Section~\ref{sec:main-results}, and Sections~\ref{sec:parameter-statistics} and~\ref{sec:training-v3} establish how it is learned by gradient training dynamics.

Numerical experiments are provided in Section~\ref{sec:experiments} to test the predictions of the theory.
A further discussion of related works is given in Section~\ref{sec:related-work}, before we conclude with a discussion of extensions in Section~\ref{sec:discussion}.
The technical details and proofs are deferred to the appendices.

\section{Problem and attention estimator}
\label{sec:model}
\textbf{Tasks and training data.}
Fix $\rho\in(0,1/2)$, and write $\ell=1-\rho$ and $\Delta=1-2\rho$. Each pretraining task draws a threshold $\alpha\sim\Unif([\rho,\ell])$, followed by $n$ independent features $X_i\sim\Unif(0,1)$ with corresponding labels $Y_i=\mathbf1\{X_i\le\alpha\}$. We observe $m$ independent blocks $Z_j=(\alpha_j,D_{n,\alpha_j})$.
A fresh task supplies an independent context of size $N$, but its threshold is withheld. Thus $m$ counts pretraining tasks, $n$ counts observations per pretraining task, and $N$ is the fresh-context size.

\textbf{Statistical benchmark.}
Let $M_N^+=\max\{X_i:Y_i=1\}$, with value zero if no label is positive. For $0\le\varepsilon<\alpha$, $\Prob(\alpha-M_N^+>\varepsilon)=(1-\varepsilon)^N\le e^{-N\varepsilon}$. The natural localization scale is therefore $N^{-1}$. This rule can be specified without pretraining; our question is whether the stated gradient procedure learns an attention estimator with comparable fresh-context accuracy. Section~\ref{subsec:bayes-comparison} compares the resulting rate with the Bayes risk, and also discusses a symmetrized estimator.

\textbf{One-head architecture.}
Introduce context tokens $z_i=(X_i,Y_i,0)^\top$ and a readout token $z_{\rm r}=(0,0,1)^\top$. A self-attention head with $Q=(0,0,1)$, $K_\theta=(\mu,\lambda,0)$, $V=(1,0,0)$, and $O=1$ computes \eqref{eq:intro-estimator} when the readout attends only to context tokens. Only $\theta=(\lambda,\mu)$ is trained. Equivalently, writing $\mathbf D_{n,\alpha}=[(X_i,Y_i)^\top]_{i=1}^n$,
\begin{equation}
\label{eq:main-heads}
 F_n(\theta; D_{n,\alpha})=\operatorname{softmax}\!\left([\mu\ \lambda]\mathbf D_{n,\alpha}\right)
 \mathbf D_{n,\alpha}^{\top}(1,0)^\top.
\end{equation}
For $\mu>0$, the ratio $k=\lambda/\mu$ measures the relative preference for labels and features, and $\mu$ measures concentration.

\textbf{Objective and optimization.}
We use the empirical squared-error objective and its infinite-task, infinite-context counterpart, given by, respectively,
\begin{equation}
 \widehat{\mathcal R}_{m,n}(\theta)=\frac1{2m}\sum_{j=1}^m
 (F_n(\theta;D_{n,\alpha_j})-\alpha_j)^2,
 \qquad
 \mathcal R_\rho(\theta)=\frac12\E_\alpha(F(\theta;\alpha)-\alpha)^2,
\end{equation}
where $F$ is the infinite-context limit of $F_n$. Empirical training is ordinary full-batch gradient descent with a constant step size $\gamma > 0$:
\begin{equation}
 \widehat\theta_{s+1} = \widehat\theta_s-\gamma\nabla\widehat{\mathcal R}_{m,n}(\widehat\theta_s).
\end{equation}
In our analysis, population gradient flow provides the reference dynamics.

\begin{assumption}[Large-resolution initialization]
\label{ass:initialization}
Fix $\kappa>0$. The deterministic initialization satisfies $\mu_0\ge\mu_{\rm init}$ and $k_0=\lambda_0/\mu_0\ge\kappa$, where $\mu_{\rm init}$ depends only on $\rho$, $\kappa$, and fixed model constants. It is independent of the pretraining corpus and is shared by population and empirical training.
\end{assumption}
This assumption requires a large initial concentration while allowing the initial score ratio to be inaccurate. Training must still learn the appropriate balance of scores.

\section{Learned estimation and transfer}
\label{sec:main-results}
\label{subsec:main-transfer}
We begin by stating the statistical outcome of the pretrained model. The next sections will explain which parameters give an accurate estimator and how gradient training reaches them.
Define $A_{m,n}:=(m\wedge n)/{\log(C_{\rm stat}mn)},$
where $C_{\rm stat}>1$ is a fixed constant. All constants below may depend on the fixed initialization, model constants, and explicitly subscripted fixed quantities such as $\eta$, but not on $m,n,N$.

\begin{theorem}[Learned threshold estimation]
\label{thm:learned-estimation}
Suppose Assumption~\ref{ass:initialization} holds, and fix $\eta\in(0,\Delta)$. For every sufficiently small fixed step size $\gamma>0$ and all sufficiently large $A_{m,n}$, there exist a nonempty deterministic interval of gradient-descent iteration indices $\mathcal I_{m,n,\gamma}$ with both endpoints are of order $\Theta(A_{m,n}^4/\gamma)$, and a pretraining event $\mathcal E_{\rm train}$ with $\Prob_{\rm pre}(\mathcal E_{\rm train})\ge1-Ce^{-cm}-C(mn)^{-c}$. Select any $\mathcal F_{\rm pre}$-measurable $\widehat s\in\mathcal I_{m,n,\gamma}$ and freeze $\widehat\theta:=\widehat\theta_{\widehat s}^\gamma$. On $\mathcal E_{\rm train}$, for each fixed $\alpha^{\rm fr}\in[\rho+\eta,\ell]$, every $N\ge1$, and every $u\ge1$,
\begin{equation}
\label{eq:learned-estimation-v3}
 \Prob_{\rm fr}\!\left(
 |F_N(\widehat\theta;D_N^{\rm fr})-\alpha^{\rm fr}|
 \le C_\eta\left[\frac{1+\log(1+CA_{m,n}u)}{A_{m,n}}+\frac{u}{N}\right]
 \,\middle|\,\mathcal F_{\rm pre}\right)\ge1-4e^{-u},
\end{equation}
where $\mathcal F_{\rm pre}$ is the $\sigma$-field generated by the pretraining corpus; thus $\widehat s$ is selected before the fresh threshold, context size, confidence level, or observations are specified.
\end{theorem}

The first term measures the finite attention resolution certified by pretraining, and the second measures localization from a fresh context. Removing the conditioning gives total failure probability at most $Ce^{-cm}+C(mn)^{-c}+4e^{-u}$. In particular, the error is $\widetilde O((m\wedge n)^{-1}+N^{-1})$, with high probability.
Note that the same trained parameters can be applied to any fresh task, and the probability guarantee is pointwise in both $\alpha^{\rm fr}$ and $N$.

Theorem~\ref{thm:learned-estimation} is the statistical consequence of the detailed continuation and transfer statements, Theorems~\ref{thm:main-empirical} and~\ref{thm:main-transfer}, in Appendix~\ref{app:statements-framework}, where the iteration interval is defined precisely. Its proof is assembled in Appendix~\ref{app:proof-learned-estimation}. We remark that the constants that appear are theoretical; the result does not supply a calibrated practical stopping rule or assert that accuracy deteriorates beyond the certified interval.

\textbf{Fresh-query classification.}
For an independent $X_q\sim\Unif(0,1)$, predict $\widehat Y_q=\mathbf1\{X_q\le F_N\}$. Since $F_N\in[0,1]$, conditional misclassification probability is exactly $|F_N-\alpha^{\rm fr}|$. The estimation guarantee therefore gives the same rate for fresh-query classification; see Corollary~\ref{cor:main-query-classification}.

\section{From parameters to statistical accuracy}
\label{sec:parameter-statistics}
Two parameter properties suffice for the accuracy of the estimator: the positive class must receive enough weight, and attention within that class must concentrate near the threshold. We first quantify these properties without assuming that the parameters were learned.

\textbf{Class separation and concentration.}
Let $p_\theta(\alpha)$ be the population attention mass on $[0,\alpha]$, and let $q_\theta(\alpha):=1-p_\theta(\alpha)$. Exact integration gives
\begin{equation}
 \frac{p_\theta(\alpha)}{q_\theta(\alpha)}
 =e^{\mu(k+\alpha-1)}\frac{1-e^{-\mu\alpha}}{1-e^{-\mu(1-\alpha)}},
 \qquad
 F(\theta;\alpha)-\alpha
 =q_\theta(\alpha)(1-\alpha)-\mu^{-1}+O(e^{-\rho\mu}).
\end{equation}
The error has two sources: mass assigned to the negative class and finite concentration within the positive class. The first is controlled by $k$ and the second by $\mu$. Across tasks, the two classes exchange dominance near $\alpha=1-k$; this is a transition in task space, not the unknown threshold within a particular context.

\textbf{A sufficient progress region.}
For suitable fixed $c_0$, $C_{\rm band}>1$, and sufficiently large $\mu_*$, define
\begin{equation}
\label{eq:main-progress-region}
 \mathcal G:=\left\{(\lambda,\mu):\mu\ge\mu_*,\quad
 \ell+\frac{\log\mu+c_0}{\mu}\le\frac\lambda\mu
 \le\ell+C_{\rm band}\frac{\log\mu}{\mu}\right\}.
\end{equation}
Throughout this region, negative-class mass is small enough that
$\sup_{\alpha\in[\rho,\ell]}|F(\theta;\alpha)-\alpha|\le C/\mu$.
The logarithmic correction supplies sufficient class separation even at the lower endpoint $\rho$. The upper boundary facilitates the training analysis, but accurate prediction may also occur above it.

\begin{proposition}[Accuracy of a fixed attention estimator]
\label{prop:fixed-estimator-v3}
Fix $\eta\in(0,\Delta)$. For $\theta\in\mathcal G$ with sufficiently large $\mu$, each fixed $\alpha\in[\rho+\eta,\ell]$, every $N\ge1$, and every $u\ge1$,
\begin{equation}
\label{eq:fixed-estimator-v3}
 \Prob_{\rm fr}\!\left(
 |F_N(\theta;D_{N,\alpha})-\alpha|
 >C_\eta\left[\frac{1+\log(1+C_{\rm fr}\mu u)}{\mu}+\frac uN\right]
 \right)\le4e^{-u}.
\end{equation}
\end{proposition}
Appendix~\ref{app:proof-fixed-estimator} proves Proposition~\ref{prop:fixed-estimator-v3}. For shorter context sizes $N$, the proof compares softmax pooling with the largest positive observation. For $N$ large enough to resolve the concentration scale, it controls the weighted empirical average directly around its population limit. Combining the two arguments avoids a logarithmic dependence on $N$ in the pretraining term.

Proposition~\ref{prop:fixed-estimator-v3} is the interface between statistics and optimization. Thus, to obtain Theorem~\ref{thm:learned-estimation}, we show that training produces parameters in $\mathcal G$ with $\mu\asymp A_{m,n}$. Once a checkpoint is frozen, conditioning on its pretraining corpus makes the fixed-parameter bound in this result applicable.

\section{From pretraining to parameter guarantees}
\label{sec:training-v3}
Our proof separates deterministic optimization geometry from finite-corpus perturbations. We first use population gradient flow to identify a scale-dependent funnel in which the concentration parameter increases and the induced estimator becomes more accurate. Then, we transfer this geometry to population gradient descent and use scale-adapted empirical-gradient estimates to preserve the radial progress, inward boundary motion, and step-length control needed by empirical gradient descent. This avoids pointwise tracking of the empirical trajectory over a training horizon that grows with the pretraining sample size.

\subsection{Population capture and increasing concentration}
\label{subsec:main-population}

We study population gradient flow $\dot\theta(t)=-\nabla\mathcal R_\rho(\theta(t))$.
For population gradient descent with fixed step size $\gamma > 0$, write $\theta_{s+1}^\gamma=\theta_s^\gamma-\gamma\nabla\mathcal R_\rho(\theta_s^\gamma)$ with initialization $\theta_0^\gamma=\theta_0$, and let $\mu_s^\gamma$ denote its second coordinate.
For empirical gradient descent, write $\widehat\theta_{s+1}^\gamma=\widehat\theta_s^\gamma-\gamma\nabla\widehat{\mathcal R}_{m,n}(\widehat\theta_s^\gamma)$ with $\widehat\theta_0^\gamma=\theta_0$, and let $\widehat\mu_s^\gamma$ denote its second coordinate.

Population training must first learn the balance of scores and then increase their scale. The gradient identities make this coupling explicit. Here $\Cov_\theta$ and $\Var_\theta$ refer to the distribution tilted by the attention weights.

\begin{lemma}[Gradient identities]
\label{lem:main-gradient-identities}
For every $\theta\in\mathbb R^2$ and $\alpha\in[\rho,\ell]$, $\partial_\lambda F(\theta;\alpha)=\Cov_\theta(X,Y\mid\alpha)$ and $\partial_\mu F(\theta;\alpha)=\Var_\theta(X\mid\alpha)$. Consequently, $\nabla\mathcal R_\rho(\theta)=\frac1\Delta\int_\rho^\ell\{F(\theta;\alpha)-\alpha\}\bigl(\Cov_\theta(X,Y\mid\alpha),\Var_\theta(X\mid\alpha)\bigr)^\top d\alpha$.
\end{lemma}

\begin{theorem}[Population capture and sharpening]
\label{thm:main-population}
Under Assumption~\ref{ass:initialization}, population gradient flow enters $\mathcal G$ in finite time and remains there. Writing $\tau_{\mathcal G}:=\inf\{t\ge0:\theta(t)\in\mathcal G\}$, there exists $c_{\rm rad}\in(0,1)$ such that $\inf_{0\le t\le\tau_{\mathcal G}} \mu(t) \ge c_{\rm rad}\mu(0)$.
Moreover, there exist constants $0<c<C<\infty$ such that, for every $t\ge\tau_{\mathcal G}$,
\begin{equation}
\label{eq:main-population-sharpening}
    \frac{c}{\mu(t)^3}
    \le
    \dot\mu(t)
    \le
    \frac{C}{\mu(t)^3},
    \qquad
    \sup_{\alpha\in[\rho,\ell]}
    |F(\theta(t);\alpha)-\alpha|
    \le
    \frac{C}{\mu(t)},
    \qquad
    \mathcal R_\rho(\theta(t))
    \le
    \frac{C}{\mu(t)^2}.
\end{equation}
Writing $\mu_{\rm ent}:=\mu(\tau_{\mathcal G})$, it follows that
\begin{equation}
\label{eq:main-population-growth}
    \mu_{\rm ent}^4+c(t-\tau_{\mathcal G})
    \le
    \mu(t)^4
    \le
    \mu_{\rm ent}^4+C(t-\tau_{\mathcal G}).
\end{equation}
Finally,
\begin{equation}
\label{eq:main-ratio-convergence}
    \frac{\lambda(t)}{\mu(t)}
    =
    \ell+
    O\!\left(
        \frac{\log\mu(t)}{\mu(t)}
    \right).
\end{equation}
Hence, $\lambda(t),\mu(t)\to\infty$, while the induced estimator converges uniformly:
\begin{equation}
\label{eq:main-population-time-rates}
    \sup_{\alpha\in[\rho,\ell]}
    |F(\theta(t);\alpha)-\alpha|
    =
    O(t^{-1/4}),
    \qquad
    \mathcal R_\rho(\theta(t))
    =
    O(t^{-1/2}).
\end{equation}
\end{theorem}

This theorem shows that population training reaches the progress region $\mathcal G$ required by Proposition~\ref{prop:fixed-estimator-v3}. The concentration parameter may decrease before capture, but it can only shrink by at most a fixed multiplicative factor beforehand. Once this region is reached, the radial drift $\dot\mu(t)$ becomes positive and optimization acts primarily by sharpening the learned softmax rule. Appendix~\ref{sec:population-dynamics} gives the proof.

This mechanism is stable under discretization. For every sufficiently small fixed step size $\gamma$, population gradient descent reaches the interior of the same scale-dependent funnel at a deterministic robust-entry index $s_0=s_0(\gamma)$, and thereafter satisfies
\begin{equation}
\label{eq:main-population-gd-summary}
\begin{aligned}
    (\mu_{s_0}^\gamma)^4+c\gamma(s-s_0)
    &\le (\mu_s^\gamma)^4
    \le (\mu_{s_0}^\gamma)^4+C\gamma(s-s_0),
    \\
    \sup_{\alpha\in[\rho,\ell]}
    |F(\theta_s^\gamma;\alpha)-\alpha|
    &\le
    \frac{C}{\{(\mu_{s_0}^\gamma)^4+c\gamma(s-s_0)\}^{1/4}}.
\end{aligned}
\end{equation}
Before robust entry, the discrete trajectory shadows the flow for only a fixed physical time. After entry, a one-step barrier argument replaces pointwise trajectory tracking. This post-entry formulation is the one that survives finite pretraining. Proposition~\ref{prop:population-gd-reference} in Appendix~\ref{sec:finite-step-gd} proves \eqref{eq:main-population-gd-summary}.

\subsection{Finite pretraining and the shrinking gradient signal}
\label{subsec:main-finite-pretraining}
Population training can increase $\mu$ indefinitely. A finite corpus perturbs the gradient, and the permissible perturbation becomes smaller as the population gradient weakens. In $\mathcal G$, radial progress has size $\mu^{-3}$ and the inward velocities at the two boundaries have size at least a constant times $\mu^{-4}$. Determining whether these directions remain reliable therefore requires relative gradient control.

On a parameter range $\mu\asymp R$, we therefore normalize the radial error by $\mu^3$ and the boundary-normal errors by $\mu^4$. Let $\mathcal E(R)$ denote the maximum of these normalized errors and the normalized speed error (defined precisely in Appendix~\ref{app:empirical-technical-estimates-auxiliary}, \eqref{eq:empirical-gd-shell-errors}). The empirical-process argument gives
\begin{equation}
\label{eq:main-relative-concentration}
 \mathcal E(R)\le C\left\{\sqrt{\frac{RL}{m}}+\frac{RL}{m}+\frac Rn\right\},
 \qquad L=\log\frac{C(1+R)mn}{\delta},
\end{equation}
with probability at least $1-\delta$, provided $n\ge CRL$. The terms distinguish sampling across tasks from finite-context bias. Their scale dependence comes from moment bounds adapted to concentrated attention; uniform bounds on individual task gradients would obscure this dependence.

Keeping \eqref{eq:main-relative-concentration} below the fixed normalized progress margins permits $R$ up to a sufficiently small constant times $A_{m,n}$. This yields the following result:
\begin{corollary}[Certified concentration from finite pretraining]
\label{cor:finite-summary-v3}
Under Assumption~\ref{ass:initialization}, for sufficiently small fixed $\gamma$ and sufficiently large $A_{m,n}$, with probability at least $1-Ce^{-cm}-C(mn)^{-c}$, empirical gradient descent enters $\mathcal G$ and remains there throughout a deterministic training interval. Every checkpoint in its late subinterval satisfies
\[
 c_A A_{m,n}\le\widehat\mu\le C_A A_{m,n},
 \qquad \sup_{\alpha\in[\rho,\ell]}|F(\widehat\theta;\alpha)-\alpha|\le C/A_{m,n}.
\]
The endpoints of the late subinterval occur after $\Theta(A_{m,n}^4/\gamma)$ post-entry gradient steps.
\end{corollary}
This is a summary of Theorem~\ref{thm:main-empirical}, with details given in Appendix~\ref{app:proof-finite-summary}.

\textbf{Why continuation does not require trajectory tracking.}
Before entry, a compact approximation argument compares empirical and population training over a fixed time interval. After entry, the proof uses overlapping parameter ranges. On each range, the normalized estimates preserve positive radial drift, inward boundary motion, and a bound on the step length. A one-step barrier argument prevents exit from the useful region, including on the step entering the next range. Intersecting the events for all deterministic ranges makes this argument valid for the data-dependent trajectory on the same reused corpus.

Thus the proof follows the properties sufficient for accurate prediction over the growing time horizon. The abstract formulation appears in Appendix~\ref{app:statements-framework}.
The certified upper scale only marks the reach of the present estimates, but the achievable accuracy may improve with further empirical training.

\section{Experiments}
\label{sec:experiments}

We examine the accuracy of a frozen estimator and the dynamics producing it. The transfer experiment uses one checkpoint per pretraining corpus across all fresh-context sizes. Population experiments examine capture and increasing concentration. Appendix~\ref{sec:experimental-details} gives reproducibility details, initialization checks, and exploratory departures from the model assumptions.

\textbf{Finite pretraining and frozen transfer.}

We use two complementary corpus sweeps to probe the predictions of Corollary~\ref{cor:finite-summary-v3} and Theorem~\ref{thm:learned-estimation}. Let $\mu_s^{\rm pop}$ be the deterministic population-GD reference path and set $A_{m,n}^{\rm proxy}:=(m\wedge n)/\log(mn)$. The theory-guided stopping diagnostic fixes $s_q:=\min\{s:\mu_s^{\rm pop}\ge qA_{m,n}^{\rm proxy}\}$ for $q\in\{0.5,0.75,1\}$ before any corpus is sampled. For each of seven balanced and axis-swapped configurations, we run full-batch empirical GD with $\gamma=400$ on 30 corpora. From the shared post-entry start through $s_1$, all 210 empirical radii remain within $[0.7,1.3]$ times their population references, and 209 of 210 runs have population prediction error at most $3/A_{m,n}^{\rm proxy}$ at every $s_q$; full results appear in Appendix~\ref{sec:experimental-details}. These predetermined checkpoints provide a controlled diagnostic of the radius and prediction-error bounds in Corollary~\ref{cor:finite-summary-v3}. The theoretical step-size threshold and comparison constants remain existential, so practical stopping-rule calibration and failure-boundary estimation lie outside the scope of this experiment.

\begin{figure}[t]
    \centering
    \includegraphics[width=\textwidth]{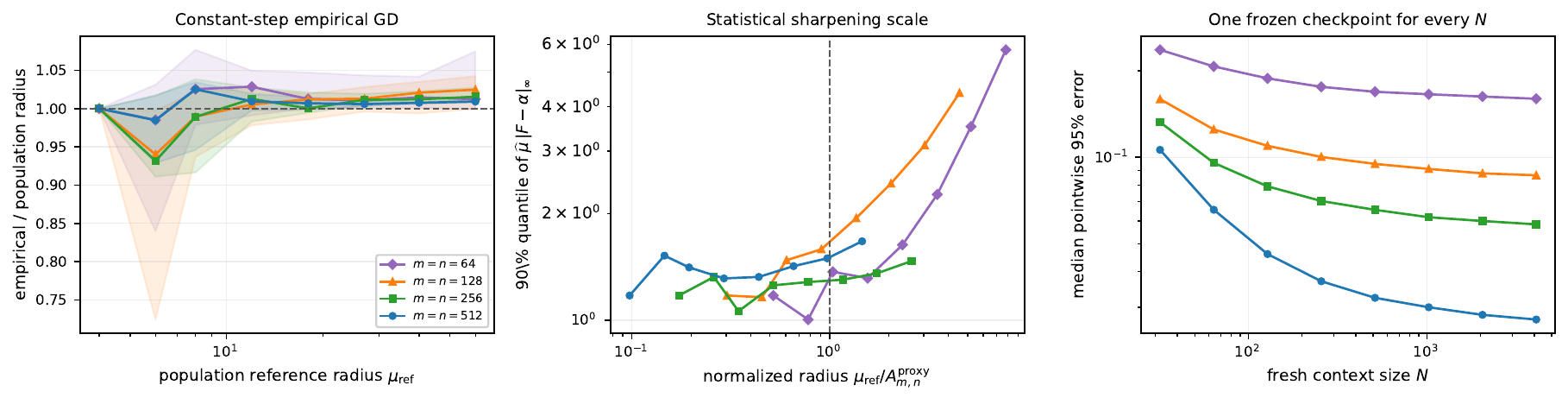}
    \caption{Finite pretraining followed by frozen transfer. The first two panels test the radius and prediction-error controls of Corollary~\ref{cor:finite-summary-v3}; the right panel tests the two-term error prediction of Theorem~\ref{thm:learned-estimation} using one checkpoint per pretraining corpus, frozen and reused for every $N$. The proxy indicates the order of the theoretical stopping scale; its comparison constants are not calibrated.}
    \label{fig:main-pretraining-transfer}
\end{figure}

For transfer, on the fixed reference-radius grid, we choose the largest $r\le A_{m,n}^{\rm proxy}$ and freeze each empirical iterate at the precomputed iteration index $\min\{s:\mu_s^{\rm pop}\ge r\}$. The grid and population path are fixed in advance, so checkpoint selection precedes observation of the empirical corpus. We reuse the resulting parameters for every $N\in\{32,64,\ldots,4096\}$ and test the two additive terms in \eqref{eq:learned-estimation-v3}. Figure~\ref{fig:main-pretraining-transfer} reports, for each corpus, the maximum over the threshold grid of the 95th-percentile absolute error, followed by the median of this quantity across corpora. The curves decay to a pretraining-limited floor. At $N=4096$, the error falls from $0.160$ to $0.0272$ as balanced pretraining grows from $m=n=64$ to $m=n=512$, with a log--log slope of $-1.02$ against $A_{m,n}^{\rm proxy}$. The supplementary axis swaps give nearly identical errors. These results are consistent with Theorem~\ref{thm:learned-estimation} at the prescribed checkpoints.

\begin{figure}[t]
    \centering
    \includegraphics[width=\textwidth]{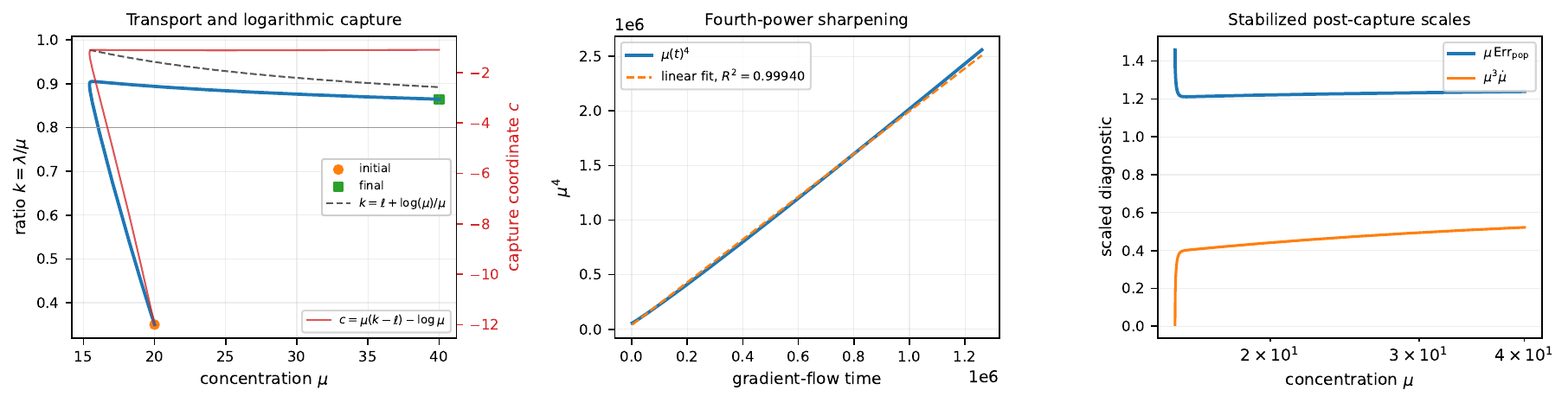}
    \caption{Exact one-head population dynamics. The ratio and rescaled endpoint coordinate $c=\mu(k-\ell)-\log\mu$ display capture into the logarithmic tube (left), $\mu^4$ is nearly affine after entry (center), and both $\mu\sup_\alpha|F-\alpha|$ and $\mu^3\dot\mu$ stabilize at constant scale (right).}
    \label{fig:main-one-head-trajectory}
\end{figure}

\textbf{Population dynamics and finite-step robustness.}
Theorem~\ref{thm:main-population} predicts capture into a logarithmic tube, $\mu(t)^4=\Theta(t)$ after capture, and uniform error $O(\mu(t)^{-1})$. We set $\rho=0.2$ and integrate the exact population vector field from $(k_0,\mu_0)=(0.35,20)$. Figure~\ref{fig:main-one-head-trajectory} follows the complete transient: the radius first falls from $20$ to $15.42$ as the ratio crosses the task range, then grows in the endpoint layer. At $\mu=40$, the rescaled endpoint coordinate $\mu(k-\ell)-\log\mu$ is $-1.10$, while $\mu^3\dot\mu=0.522$ and $\mu\sup_\alpha|F(\theta;\alpha)-\alpha|=1.239$. After capture, a linear fit of $\mu^4$ against flow time has $R^2=0.9994$, and the tested error is at most $1.51/\mu(t)$. For the discrete-time prediction \eqref{eq:main-population-gd-summary}, we initialize population GD in the post-entry regime and test $\gamma\in\{50,100,200,400\}$. Over the plotted range $\mu\ge14$, the relative radial increment is below $0.4\%$, and $(\mu_{s+1}^4-\mu_s^4)/\gamma$ is nearly independent of $\gamma$. For additional details, see Figure~\ref{fig:app-finite-pretraining} in Appendix~\ref{sec:experimental-details}.

\section{Related Work}
\label{sec:related-work}

\textbf{Algorithmic and statistical accounts of ICL.}
Prior work connects transformer predictions to gradient-based and closed-form regression algorithms \citep{akyurek2022learning,pmlr-v202-von-oswald23a}, establishes algorithm selection and generalization for constrained pretraining ERM \citep{bai2023transformers}, interprets ICL as latent-concept inference \citep{xie2021explanation}, and constructs attention-based solvers for posterior predictive quantities \citep{kang2026posterior}. These characterize computations or statistical solutions; gradient training of transformer weights requires separate analysis.

\textbf{Learned statistical estimators and tabular priors.}
Transformer-based statistical ICL covers posterior predictive distributions, density estimation, and parameter estimation \citep{kang2026posterior,hollmann2025accurate,zhang2025tabpfn}. Meta-statistical learning similarly treats an entire dataset as input to a permutation-invariant network trained across datasets \citep{peyrard2025meta}. We analyze the optimization and finite-corpus transfer of such an estimator for a threshold. This choice reflects tree-based generators in TFM priors: TabICL includes XGBoost tasks, while MITRA combines tree and ensemble priors with structural causal models \citep{qu2025tabicl,zhang2025mitra}. Full query-label objectives and multidimensional structure remain outside our scope.

\textbf{Training dynamics and invariant structures.}
Convergence analyses use Polyak--\L{}ojasiewicz (PL) or related gradient-dominance arguments for linear, looped, and representation-based models \citep{zhang2024trained,gatmiry2024looped,yang2024representations}, and also give classification guarantees \citep{shen2024training}. Trajectory analyses track competing attention weights, spectral dynamics, and learning phases \citep{pmlr-v235-huang24d,chen2024training,zhang2025dynamics,he2025context}; finite-step and mini-batch effects can create moving stability boundaries in linear self-attention \citep{balasubramanian2026large}. For induction, \citet{chen2024unveiling} study staged feature learning and attention-scale growth against finite-context errors. Related work analyzes dynamics in a symmetry-motivated low-dimensional reduction \citep{musat2025emergence} and establishes invariant population-dynamics subspaces \citep{musat2026invariant}. Our framework combines regional capture, inward drift, increasing resolution, and scale-adapted estimates that preserve these properties on a finite corpus, thereby linking gradient training to threshold-estimation guarantees.

\textbf{Boundary estimation and attention resolution.}
Nonlinear and nonparametric analyses study learned low-dimensional structure and local prediction rules \citep{oko2024pretrained,collins2024lipschitz,li2024one}; experiments with discrete functions include thresholds \citep{bhattamishra2024discrete}. Directly observed threshold models exhibit inverse-sample-size localization \citep{hansen2000sample}, whereas indirect or high-dimensional boundary-estimation rates depend on ill-posedness, smoothness, and sparsity \citep{goldenshluger2006optimal,feng2022nonregular}. Parameter divergence with functional convergence appears in implicit bias and max-margin attention selection \citep{soudry2018implicit,tarzanagh2023maxmargin,vasudeva2025implicit}, and low-entropy attention is studied by \citet{varre2026polarizes}. Pretraining guarantees distinguish task and context counts \citep{wu2024many,oko2024pretrained}; related work treats finite-token approximation \citep{bohbot2026token} and finite-data softmax optimization with preconditioning and regularization \citep{goel2026training}. We instead connect ordinary constant-step descent to boundary estimation through relative gradient control along a diverging trajectory.

\section{Discussion}
\label{sec:discussion}
The analysis connects a statistical rule to a training mechanism. Calibrating the relative label and feature scores suppresses the wrong label class, and increasing concentration localizes the boundary within the correct class. Population training acquires this calibration and continues to increase resolution. Relative gradient estimates show how much of that progress remains certified from a fixed finite pretraining corpus.

\subsection{Bayes benchmark and symmetrization}
\label{subsec:bayes-comparison}
Let $M_N^-:=\min\{X_i:Y_i=0\}$, with value one if no label is negative, and define the posterior bracket
\[
 L_N:=\max\{\rho,M_N^+\},
 \qquad
 U_N:=\min\{\ell,M_N^-\}.
\]
Under the uniform prior on $[\rho,\ell]$, the posterior distribution of $\alpha$ given the context is $\Unif([L_N,U_N])$. Consequently, its midpoint $B_N:=(L_N+U_N)/2$ is the exact Bayes estimator under both absolute and squared loss. The corresponding prior-averaged Bayes risks satisfy
\[
 R_{N,1}^{\rm Bayes}
 =\frac14\E (U_N-L_N)
 =\frac{1+o(1)}{2N},
 \qquad
 R_{N,2}^{\rm Bayes}
 =\frac1{12}\E (U_N-L_N)^2
 =\frac{1+o(1)}{2N^2}.
\]
These constants follow from the local spacing limit $N(\alpha-L_N,U_N-\alpha)\Rightarrow(E_L,E_R)$ away from the prior endpoints, where $E_L,E_R$ are independent $\operatorname{Exp}(1)$ variables; the endpoint layers have vanishing prior mass. In its hard-attention limit, the trained one-head rule approaches the left endpoint $M_N^+$. For fixed interior $\alpha$, $\E|M_N^+-\alpha|=(1+o(1))N^{-1}$ and $\E(M_N^+-\alpha)^2=(2+o(1))N^{-2}$. Prior-averaging these fixed-interior asymptotics gives the same leading constants because the prior-endpoint layers contribute only lower-order terms. Thus, the one-head rule attains the Bayes-optimal rates, while its asymptotic absolute-risk and squared-risk constants are, respectively, two and four times the Bayes constants. The learned softmax estimator inherits these rates whenever its finite-pretraining resolution term is no larger than the fresh-context term.

For comparison, define two oppositely oriented heads with tied parameters $(k,\mu)$:
\[
 F_{\pm,N}(k,\mu;D_N)
 :=\frac{\sum_{i=1}^N X_i\exp\{\pm\mu(kY_i+X_i)\}}
          {\sum_{i=1}^N \exp\{\pm\mu(kY_i+X_i)\}},
 \qquad
 G_N:=\frac{F_{+,N}+F_{-,N}}2.
\]
When both label classes are resolved and the two softmax endpoint errors are $o(N^{-1})$, $F_{+,N}$ approaches $M_N^+$ and $F_{-,N}$ approaches $M_N^-$ closely enough that $G_N$ attains the leading Bayes constant. This statement concerns statistical representation. The optimization and finite-pretraining theorems concern the one-head objective, while training under the symmetrized loss is left for future work.

\subsection{Boundary coverage}
\label{subsec:main-boundary}

On the training interval $[\rho,\ell]$, the one-head population predictor is uniformly accurate all the way to both endpoints. The all-$N$ transfer theorem above imposes the fixed margin $\alpha^{\rm fr}\ge\rho+\eta$ only because short fresh contexts may not resolve the weak score separation near the lower endpoint. Once the fresh context is large enough relative to the frozen concentration scale, the self-normalized transfer argument applies on the full closed interval $[\rho,\ell]$; see Corollary~\ref{cor:one-head-full-boundary-large-N}.

A genuine architectural obstruction appears when the physical task family is extended from $[\rho,\ell]$ to $[0,1]$.

\begin{proposition}[Exact one-head boundary obstruction]
\label{prop:main-boundary-obstruction}
For $S_{k,\mu}(x,y):=\mu(ky+x)$, let $F_+(k,\mu;\alpha)$ and $F_-(k,\mu;\alpha)$ be the population softmax averages of $X$ under scores $S_{k,\mu}(X,Y)$ and $-S_{k,\mu}(X,Y)$, respectively. For fixed $k=\lambda/\mu$,
\begin{equation}
\label{eq:main-boundary-exact}
    F_+(k,\mu;0)
    =
    1-\frac1\mu+\frac1{e^\mu-1},
    \qquad
    F_-(k,\mu;1)
    =
    \frac1\mu-\frac1{e^\mu-1}.
\end{equation}
Consequently,
\begin{equation}
\label{eq:main-boundary-limits}
    F_+(k,\mu;0)\to1,
    \qquad
    F_-(k,\mu;1)\to0
    \qquad
    \text{as }
    \mu\to\infty.
\end{equation}
Hence no single fixed orientation in this score family can be uniformly consistent on the full physical task interval $[0,1]$.
\end{proposition}
Appendix~\ref{app:proof-boundary-obstruction} gives the proof of Proposition~\ref{prop:main-boundary-obstruction}.

\textbf{Scope and further questions.}
The theory assumes a fixed split feature, noiseless labels, a uniform design, and large-resolution initialization; the two orientations are symmetric. Appendix~\ref{sec:experimental-details} explores label noise and Beta designs, while Proposition~\ref{prop:relative-radial-loss} and Lemma~\ref{lem:integrated-radial-loss} bound the pre-capture scale loss used in proving Theorem~\ref{thm:main-population}. Extending the theory to small initialization, noisy or nonuniform settings, learned split features, and sharp finite-pretraining limits remains open.

\newpage

\subsection*{Acknowledgements} 
The authors would like to thank Subhro Ghosh, Wei Huang, and Mingyuan Bai for insightful comments. We used generative AI tools to assist in developing proof ideas and drafting parts of the proofs. The authors independently verified all AI-assisted arguments and take full responsibility for the mathematical claims and proofs. Y.S. Tan was supported by NUS Startup Grant A-8000448-00-00 and the Singapore Ministry of Education (MOE) AcRF Tier 1 Grants A-8002498-00-00 and A-8004458-00-00.

\bibliography{main}

@article{chen2024unveiling,
  title={Unveiling Induction Heads: Provable Training Dynamics and Feature Learning in Transformers},
  author={Chen, Siyu and Sheen, Heejune and Wang, Tianhao and Yang, Zhuoran},
  journal={arXiv preprint arXiv:2409.10559},
  year={2024}
}

@InProceedings{bohbot2026token,
  title={Token Sample Complexity of Attention},
  author={Bohbot, L{\'e}a and Letrouit, Cyril and Peyr{\'e}, Gabriel and Vialard, Fran{\c{c}}ois-Xavier},
  booktitle={Proceedings of the 43rd International Conference on Machine Learning},
  year={2026},
  volume={306},
  series={Proceedings of Machine Learning Research},
  publisher={PMLR},
  address={Seoul, South Korea}
}

@article{goel2026training,
  title={Training Dynamics of Softmax Self-Attention: Fast Global Convergence via Preconditioning},
  author={Goel, Gautam and Soltanolkotabi, Mahdi and Bartlett, Peter},
  journal={arXiv preprint arXiv:2603.01514},
  year={2026}
}

@article{musat2025emergence,
  title={On the Emergence of Induction Heads for In-Context Learning},
  author={Musat, Tiberiu and Pimentel, Tiago and Noci, Lorenzo and Stolfo, Alessandro and Sachan, Mrinmaya and Hofmann, Thomas},
  journal={arXiv preprint arXiv:2511.01033},
  year={2025}
}

@article{musat2026invariant,
  title={Invariant Learning Dynamics of Transformers in Inductive Reasoning Tasks},
  author={Musat, Tiberiu and Pimentel, Tiago and Zucchet, Nicolas and Hofmann, Thomas},
  journal={arXiv preprint arXiv:2607.11875},
  year={2026}
}

@article{kang2026posterior,
  title={Transformers Can Learn Posterior Predictive Distributions In-Context},
  author={Kang, Gyeonghun and Lee, Changwoo J. and Cheng, Xiang},
  journal={arXiv preprint arXiv:2605.26713},
  year={2026}
}

@InProceedings{chen2024training,
  title={Training Dynamics of Multi-Head Softmax Attention for In-Context Learning: Emergence, Convergence, and Optimality (Extended Abstract)},
  author={Chen, Siyu and Sheen, Heejune and Wang, Tianhao and Yang, Zhuoran},
  booktitle={Proceedings of the Thirty-Seventh Conference on Learning Theory},
  pages={4573--4573},
  year={2024},
  volume={247},
  series={Proceedings of Machine Learning Research},
  publisher={PMLR}
}

@inproceedings{akyurek2022learning,
  title={What Learning Algorithm Is In-Context Learning? Investigations with Linear Models},
  author={Aky{\"u}rek, Ekin and Schuurmans, Dale and Andreas, Jacob and Ma, Tengyu and Zhou, Denny},
  booktitle={International Conference on Learning Representations},
  year={2023}
}

@InProceedings{pmlr-v202-von-oswald23a,
  title = 	 {Transformers Learn In-Context by Gradient Descent},
  author =       {Von Oswald, Johannes and Niklasson, Eyvind and Randazzo, Ettore and Sacramento, Joao and Mordvintsev, Alexander and Zhmoginov, Andrey and Vladymyrov, Max},
  booktitle = 	 {Proceedings of the 40th International Conference on Machine Learning},
  pages = 	 {35151--35174},
  year = 	 {2023},
  editor = 	 {Krause, Andreas and Brunskill, Emma and Cho, Kyunghyun and Engelhardt, Barbara and Sabato, Sivan and Scarlett, Jonathan},
  volume = 	 {202},
  series = 	 {Proceedings of Machine Learning Research},
  month = 	 {23--29 Jul},
  publisher =    {PMLR},
  url = 	 {https://proceedings.mlr.press/v202/von-oswald23a.html}
}

@InProceedings{pmlr-v235-huang24d,
  title = 	 {In-context Convergence of Transformers},
  author =       {Huang, Yu and Cheng, Yuan and Liang, Yingbin},
  booktitle = 	 {Proceedings of the 41st International Conference on Machine Learning},
  pages = 	 {19660--19722},
  year = 	 {2024},
  editor = 	 {Salakhutdinov, Ruslan and Kolter, Zico and Heller, Katherine and Weller, Adrian and Oliver, Nuria and Scarlett, Jonathan and Berkenkamp, Felix},
  volume = 	 {235},
  series = 	 {Proceedings of Machine Learning Research},
  month = 	 {21--27 Jul},
  publisher =    {PMLR},
  url = 	 {https://proceedings.mlr.press/v235/huang24d.html}
}

@article{zhang2024trained,
  title={Trained Transformers Learn Linear Models In-Context},
  author={Zhang, Ruiqi and Frei, Spencer and Bartlett, Peter L.},
  journal={Journal of Machine Learning Research},
  volume={25},
  number={49},
  pages={1--55},
  year={2024}
}

@InProceedings{gatmiry2024looped,
  title={Can Looped Transformers Learn to Implement Multi-step Gradient Descent for In-context Learning?},
  author={Gatmiry, Khashayar and Saunshi, Nikunj and Reddi, Sashank J. and Jegelka, Stefanie and Kumar, Sanjiv},
  booktitle={Proceedings of the 41st International Conference on Machine Learning},
  pages={15130--15152},
  year={2024},
  volume={235},
  series={Proceedings of Machine Learning Research},
  publisher={PMLR}
}

@article{yang2024representations,
  title={In-Context Learning with Representations: Contextual Generalization of Trained Transformers},
  author={Yang, Tong and Huang, Yu and Liang, Yingbin and Chi, Yuejie},
  journal={Advances in Neural Information Processing Systems},
  volume={37},
  year={2024}
}

@InProceedings{zhang2025dynamics,
  title={Training Dynamics of In-Context Learning in Linear Attention},
  author={Zhang, Yedi and Singh, Aaditya K. and Latham, Peter E. and Saxe, Andrew M.},
  booktitle={Proceedings of the 42nd International Conference on Machine Learning},
  pages={76047--76087},
  year={2025},
  volume={267},
  series={Proceedings of Machine Learning Research},
  publisher={PMLR}
}

@article{brown2020language, title={Language models are few-shot learners}, author={Brown, Tom and Mann, Benjamin and Ryder, Nick and Subbiah, Melanie and Kaplan, Jared D and Dhariwal, Prafulla and Neelakantan, Arvind and Shyam, Pranav and Sastry, Girish and Askell, Amanda and others}, journal={Advances in neural information processing systems}, volume={33}, pages={1877--1901}, year={2020} }

@article{hollmann2025accurate,
  title={Accurate Predictions on Small Data with a Tabular Foundation Model},
  author={Hollmann, Noah and M{\"u}ller, Samuel and Purucker, Lennart and Krishnakumar, Arjun and K{\"o}rfer, Max and Hoo, Shi Bin and Schirrmeister, Robin Tibor and Hutter, Frank},
  journal={Nature},
  volume={637},
  pages={319--326},
  year={2025},
  doi={10.1038/s41586-024-08328-6}
}

@InProceedings{qu2025tabicl,
  title={{TabICL}: A Tabular Foundation Model for In-Context Learning on Large Data},
  author={Qu, Jingang and Holzm{\"u}ller, David and Varoquaux, Ga{\"e}l and Le Morvan, Marine},
  booktitle={Proceedings of the 42nd International Conference on Machine Learning},
  pages={50817--50847},
  year={2025},
  volume={267},
  series={Proceedings of Machine Learning Research},
  publisher={PMLR}
}

@article{peyrard2025meta,
  title={Meta-Statistical Learning: Supervised Learning of Statistical Estimators},
  author={Peyrard, Maxime and Cho, Kyunghyun},
  journal={arXiv preprint arXiv:2502.12088},
  year={2025},
  url={https://arxiv.org/abs/2502.12088}
}

@article{sainsburydale2024likelihood,
  title={Likelihood-Free Parameter Estimation with Neural Bayes Estimators},
  author={Sainsbury-Dale, Matthew and Zammit-Mangion, Andrew and Huser, Rapha{\"e}l},
  journal={The American Statistician},
  volume={78},
  number={1},
  pages={1--14},
  year={2024},
  doi={10.1080/00031305.2023.2249522}
}

@article{zammitmangion2025neural,
  title={Neural Methods for Amortized Inference},
  author={Zammit-Mangion, Andrew and Sainsbury-Dale, Matthew and Huser, Rapha{\"e}l},
  journal={Annual Review of Statistics and Its Application},
  volume={12},
  pages={311--335},
  year={2025},
  doi={10.1146/annurev-statistics-112723-034123}
}

@inproceedings{muller2025priorfitted,
  title={Position: The Future of Bayesian Prediction Is Prior-Fitted},
  author={M{\"u}ller, Samuel and Reuter, Arik and Hollmann, Noah and R{\"u}gamer, David and Hutter, Frank},
  booktitle={Proceedings of the 42nd International Conference on Machine Learning},
  series={Proceedings of Machine Learning Research},
  volume={267},
  pages={81861--81875},
  year={2025},
  url={https://proceedings.mlr.press/v267/muller25d.html}
}

@inproceedings{zhang2025mitra,
  title={{MITRA}: Mixed Synthetic Priors for Enhancing Tabular Foundation Models},
  author={Zhang, Xiyuan and Maddix, Danielle C. and Yin, Junming and Erickson, Nick and Ansari, Abdul Fatir and Han, Boran and Zhang, Shuai and Akoglu, Leman and Faloutsos, Christos and Mahoney, Michael W. and Hu, Cuixiong and Rangwala, Huzefa and Karypis, George and Wang, Bernie},
  booktitle={Advances in Neural Information Processing Systems},
  year={2025},
  url={https://proceedings.neurips.cc/paper_files/paper/2025/hash/177d68f4adef163b7b123b5c5adb3c60-Abstract-Conference.html}
}

@article{zhang2025tabpfn,
  title={{TabPFN}: One Model to Rule Them All?},
  author={Zhang, Qiong and Tan, Yan Shuo and Tian, Qinglong and Li, Pengfei},
  journal={arXiv preprint arXiv:2505.20003},
  year={2025},
  url={https://arxiv.org/abs/2505.20003}
}

@article{li2024one,
  title={One-layer transformer provably learns one-nearest neighbor in context},
  author={Li, Zihao and Cao, Yuan and Gao, Cheng and He, Yihan and Liu, Han and Klusowski, Jason M and Fan, Jianqing and Wang, Mengdi},
  journal={Advances in Neural Information Processing Systems},
  volume={37},
  pages={82166--82204},
  year={2024}
}

@InProceedings{shen2024training,
  title={On the Training Convergence of Transformers for In-Context Classification of {G}aussian Mixtures},
  author={Shen, Wei and Zhou, Ruida and Yang, Jing and Shen, Cong},
  booktitle={Proceedings of the 42nd International Conference on Machine Learning},
  pages={54732--54771},
  year={2025},
  volume={267},
  series={Proceedings of Machine Learning Research},
  publisher={PMLR}
}

@article{garg2022can,
  title={What can transformers learn in-context? a case study of simple function classes},
  author={Garg, Shivam and Tsipras, Dimitris and Liang, Percy S and Valiant, Gregory},
  journal={Advances in neural information processing systems},
  volume={35},
  pages={30583--30598},
  year={2022}
}

@inproceedings{xie2021explanation,
  title={An Explanation of In-Context Learning as Implicit Bayesian Inference},
  author={Xie, Sang Michael and Raghunathan, Aditi and Liang, Percy and Ma, Tengyu},
  booktitle={International Conference on Learning Representations},
  year={2022}
}

@article{bai2023transformers,
  title={Transformers as statisticians: Provable in-context learning with in-context algorithm selection},
  author={Bai, Yu and Chen, Fan and Wang, Huan and Xiong, Caiming and Mei, Song},
  journal={Advances in neural information processing systems},
  volume={36},
  pages={57125--57211},
  year={2023}
}

@InProceedings{he2025context,
  title={In-Context Linear Regression Demystified: Training Dynamics and Mechanistic Interpretability of Multi-Head Softmax Attention},
  author={He, Jianliang and Pan, Xintian and Chen, Siyu and Yang, Zhuoran},
  booktitle={Proceedings of the 42nd International Conference on Machine Learning},
  pages={22686--22742},
  year={2025},
  volume={267},
  series={Proceedings of Machine Learning Research},
  publisher={PMLR}
}

@article{balasubramanian2026large,
  title={Unbiased Gradients, Moving Stability Boundaries: Exact Mini-Batch Geometry in Linear Self-Attention},
  author={Balasubramanian, Krishnakumar},
  journal={arXiv preprint arXiv:2605.21292},
  year={2026},
  url={https://arxiv.org/abs/2605.21292}
}

@article{oko2024pretrained,
  title={Pretrained transformer efficiently learns low-dimensional target functions in-context},
  author={Oko, Kazusato and Song, Yujin and Suzuki, Taiji and Wu, Denny},
  journal={Advances in Neural Information Processing Systems},
  volume={37},
  pages={77316--77365},
  year={2024}
}

@inproceedings{wu2024many,
  title={How many pretraining tasks are needed for in-context learning of linear regression?},
  author={Wu, Jingfeng and Zou, Difan and Chen, Zixiang and Braverman, Vladimir and Gu, Quanquan and Bartlett, Peter},
  booktitle={International Conference on Learning Representations},
  year={2024},
  url={https://proceedings.iclr.cc/paper_files/paper/2024/hash/fa3c139cf8084de7bfd944f1c90c8695-Abstract-Conference.html}
}

@article{goldenshluger2006optimal,
author = {A. Goldenshluger and A. Tsybakov and A. Zeevi},
title = {{Optimal change-point estimation from indirect observations}},
volume = {34},
journal = {The Annals of Statistics},
number = {1},
publisher = {Institute of Mathematical Statistics},
pages = {350 -- 372},
year = {2006},
doi = {10.1214/009053605000000750},
URL = {https://doi.org/10.1214/009053605000000750}
}

@article{feng2022nonregular,
  title={Nonregular and minimax estimation of individualized thresholds in high dimension with binary responses},
  author={Feng, Huijie and Ning, Yang and Zhao, Jiwei},
  journal={The Annals of Statistics},
  volume={50},
  number={4},
  pages={2284--2305},
  year={2022},
  publisher={Institute of Mathematical Statistics}
}

@inproceedings{tan2022cautionary,
  title={A cautionary tale on fitting decision trees to data from additive models: generalization lower bounds},
  author={Tan, Yan Shuo and Agarwal, Abhineet and Yu, Bin},
  booktitle={International Conference on Artificial Intelligence and Statistics},
  pages={9663--9685},
  year={2022},
  organization={PMLR}
}

@inproceedings{bhattamishra2024discrete,
  title={Understanding In-Context Learning in Transformers and {LLM}s by Learning to Learn Discrete Functions},
  author={Bhattamishra, Satwik and Patel, Arkil and Blunsom, Phil and Kanade, Varun},
  booktitle={International Conference on Learning Representations},
  year={2024},
  url={https://proceedings.iclr.cc/paper_files/paper/2024/hash/f7e7fabd73b3df96c54a320862afcb78-Abstract-Conference.html}
}

@inproceedings{collins2024lipschitz,
  title={In-Context Learning with Transformers: Softmax Attention Adapts to Function Lipschitzness},
  author={Collins, Liam and Parulekar, Advait and Mokhtari, Aryan and Sanghavi, Sujay and Shakkottai, Sanjay},
  booktitle={Advances in Neural Information Processing Systems},
  year={2024},
  url={https://proceedings.neurips.cc/paper_files/paper/2024/hash/a8633d27d782f66fe660c2fb4bae446e-Abstract-Conference.html}
}

@article{soudry2018implicit,
  title={The Implicit Bias of Gradient Descent on Separable Data},
  author={Soudry, Daniel and Hoffer, Elad and Nacson, Mor Shpigel and Gunasekar, Suriya and Srebro, Nathan},
  journal={Journal of Machine Learning Research},
  volume={19},
  number={70},
  pages={1--57},
  year={2018},
  url={https://jmlr.org/papers/v19/18-188.html}
}

@inproceedings{tarzanagh2023maxmargin,
  title={Max-Margin Token Selection in Attention Mechanism},
  author={Tarzanagh, Davoud Ataee and Li, Yingcong and Zhang, Xuechen and Oymak, Samet},
  booktitle={Advances in Neural Information Processing Systems},
  volume={36},
  year={2023},
  url={https://proceedings.neurips.cc/paper_files/paper/2023/hash/970f59b22f4c72aec75174aae63c7459-Abstract-Conference.html}
}

@article{vasudeva2025implicit,
  title={Implicit Bias and Fast Convergence Rates for Self-attention},
  author={Vasudeva, Bhavya and Deora, Puneesh and Thrampoulidis, Christos},
  journal={Transactions on Machine Learning Research},
  year={2025},
  url={https://openreview.net/forum?id=pKilnjQsb0}
}

@article{varre2026polarizes,
  title={Gradient Flow Polarizes Softmax Outputs towards Low-Entropy Solutions},
  author={Varre, Aditya and Rofin, Mark and Flammarion, Nicolas},
  journal={arXiv preprint arXiv:2603.06248},
  year={2026},
  url={https://arxiv.org/abs/2603.06248}
}

@article{roedder2025theoretical,
  title={Theoretical Guarantees for Neural Estimators in Parametric Statistics},
  author={R{\"o}dder, Almut and Hentschel, Manuel and Engelke, Sebastian},
  journal={arXiv preprint arXiv:2506.18508},
  year={2025},
  url={https://arxiv.org/abs/2506.18508}
}

@article{frazier2024statistical,
  title={The Statistical Accuracy of Neural Posterior and Likelihood Estimation},
  author={Frazier, David T. and Kelly, Ryan and Drovandi, Christopher and Warne, David J.},
  journal={arXiv preprint arXiv:2411.12068},
  year={2024},
  url={https://arxiv.org/abs/2411.12068}
}

@inproceedings{erickson2025tabarena,
  title={{TabArena}: A Living Benchmark for Machine Learning on Tabular Data},
  author={Erickson, Nick and Purucker, Lennart and Tschalzev, Andrej and Holzm{\"u}ller, David and Desai, Prateek Mutalik and Salinas, David and Hutter, Frank},
  booktitle={Advances in Neural Information Processing Systems, Datasets and Benchmarks Track},
  year={2025},
  doi={10.52202/085713-0519},
  url={https://proceedings.neurips.cc/paper_files/paper/2025/hash/1697e3fb412da11dc9488249f9e7bbc9-Abstract-Datasets_and_Benchmarks_Track.html}
}

@article{hansen2000sample,
author = {Hansen, Bruce E.},
title = {Sample Splitting and Threshold Estimation},
journal = {Econometrica},
volume = {68},
number = {3},
pages = {575-603},
doi = {10.1111/1468-0262.00124},
url = {https://onlinelibrary.wiley.com/doi/abs/10.1111/1468-0262.00124},
eprint = {https://onlinelibrary.wiley.com/doi/pdf/10.1111/1468-0262.00124},
year = {2000}
}
\bibliographystyle{main}

\appendix
\section*{Supplementary Material}
\textbf{Roadmap.}
The appendix is organized in the dependency order of the arguments:
\begin{itemize}[leftmargin=*]
\item Appendix~\ref{app:statements-framework} gives the proof framework, the exact deterministic training-time window, and the detailed finite-pretraining and freeze-then-transfer theorems. It also assembles Theorem~\ref{thm:learned-estimation} and records the short proofs of the gradient identities, the finite-pretraining summary, and the boundary obstruction.
\item Appendix~\ref{sec:structural-identities} derives the two-component mixture representation, the bias and ratio identities, and global well-posedness of the population flow.
\item Appendix~\ref{sec:population-dynamics} proves population capture and sharpening (Theorem~\ref{thm:main-population}) through outer ratio correction, interior transport, logarithmic endpoint capture, and radial persistence.
\item Appendix~\ref{sec:finite-step-gd} transfers the population-flow geometry to constant-step population gradient descent using finite-horizon Euler shadowing and discrete barrier arguments.
\item Appendix~\ref{app:empirical-technical-estimates-auxiliary} establishes the compact and scale-adapted empirical-process bounds used to compare finite-corpus and population gradients.
\item Appendix~\ref{sec:empirical-pretraining} proves the finite-pretraining theorem by combining compact pre-entry tracking with continuation across overlapping statistical shells; it then derives the certified radius and deterministic training-time window.
\item Appendix~\ref{sec:one-head-interior-transfer} proves the fixed-parameter fresh-context bound (Proposition~\ref{prop:fixed-estimator-v3}), the freeze-then-transfer theorem, and the large-context extension to the full training interval.
\item Appendix~\ref{sec:experimental-details} provides the reproducibility protocol, detailed finite-step and transfer diagnostics, initialization and matched-design scope checks, and exploratory label-noise experiments.
\end{itemize}

\section{Proof framework and technical statements}
\label{app:statements-framework}
This section records the exact training-time window construction and the two technical theorems used to prove the streamlined learned-estimation theorem in the main manuscript. Main-text results appear here only through their assembly and short proofs. The subsequent appendices provide the population, empirical-process, and transfer arguments in dependency order.

\subsection{Continuation geometry}
The continuation argument uses the concrete region $\mathcal G$ in
\eqref{eq:main-progress-region}. Its resolution coordinate is $\mu$ and its
shape coordinate is $k=\lambda/\mu$. Within this region, the population radial
drift is $\Theta(\mu^{-3})$, the vector-field speed is $O(\mu^{-3})$, the
inward margins at the side barriers $h_-$ and $h_+$ are
$\Omega(\mu^{-4})$, and the functional error is $O(\mu^{-1})$; hence the
intrinsic training clock is $\mu^4$. The population analysis establishes
finite-time entry and subsequent invariance. The empirical continuation
argument then preserves the radial and boundary signs on overlapping shells by
normalizing their errors by $\mu^3$ and $\mu^4$, respectively.

\subsection{Deterministic training-time window}
Fix $C_{\rm stat}>1$ as in Appendix~\ref{subsec:one-boundary-empirical-scales}. The certified radius is $M_{\rm stat}(m,n)\asymp A_{m,n}$.
Let $T_0>\tau_{\mathcal G}$ be the deterministic post-capture population-flow time chosen in Appendix~\ref{sec:finite-step-gd}, and let $s_0(\gamma):=\lceil T_0/\gamma\rceil$ be the corresponding population-GD robust-entry index. Define $\mathbb T_\gamma := \{T_0+\gamma j:j=0,1,2,\ldots\}$. For $T=T_0+\gamma(s-s_0)$, write $\widehat\theta^\gamma(T):=\widehat\theta_s^\gamma$ and $\widehat\mu^\gamma(T):=\widehat\mu_s^\gamma$.
The empirical continuation argument provides deterministic entry radii
$r_-^{\rm gd}\le r_+^{\rm gd}$ and constants
$0<c_-^{\rm gd}\le c_+^{\rm gd}$, which define the radial envelopes
\begin{equation}
\label{eq:main-radial-envelopes}
M_\pm^{\rm gd}(T)
:=
\left\{
(r_\pm^{\rm gd})^4
+
c_\pm^{\rm gd}(T-T_0)
\right\}^{1/4}.
\end{equation}
Choose $c_{\rm rate}>0$ sufficiently small, depending only on the deterministic envelope constants. Define
\begin{align*}
    T_{\rm stat}^{\rm gd}(m,n,\gamma)
    &:=
    \max
    \left\{
        T\in\mathbb T_\gamma:
        M_+^{\rm gd}(T)
        \le
        M_{\rm stat}(m,n)
    \right\},\\
    T_{\rm rate}^{\rm gd}(m,n,\gamma)
    :=
    \min
    &\left\{
        T\in\mathbb T_\gamma:
        T\le T_{\rm stat}^{\rm gd},
        \;
        M_-^{\rm gd}(T)
        \ge
        c_{\rm rate}M_{\rm stat}(m,n)
    \right\}.
\end{align*}
The resulting deterministic statistical-scale training-time window is
\begin{equation}
\label{eq:main-statistical-window}
\mathcal T_{\rm stat}(m,n,\gamma)
:=
[
T_{\rm rate}^{\rm gd}(m,n,\gamma),
T_{\rm stat}^{\rm gd}(m,n,\gamma)
]
\cap
\mathbb T_\gamma.
\end{equation}
Both endpoints depend only on $(m,n,\gamma)$ and fixed model constants; they are determined before the empirical training trajectory is observed. The theorem below shows that, once the pretraining sizes are sufficiently large, these definitions are well posed and the resulting window is nonempty.

\begin{theorem}[Finite-pretraining sharpening horizon]
\label{thm:main-empirical}
Under Assumption~\ref{ass:initialization}, there exist constants $\gamma_0>0$ and $A_0<\infty$ such that, for every fixed $0<\gamma\le\gamma_0$ and every $m,n$ satisfying $A_{m,n}\ge A_0$, there exists a pretraining event $\mathcal E_{\rm train}$ with $\Prob(\mathcal E_{\rm train}) \ge 1-Ce^{-cm}-C(mn)^{-c}$ on which, simultaneously for every
$T\in[T_0,T_{\rm stat}^{\rm gd}]\cap\mathbb T_\gamma$,
\begin{equation}
\label{eq:main-empirical-horizon}
\widehat\theta^\gamma(T)
\in
\mathcal G,
\qquad
M_-^{\rm gd}(T)
\le
\widehat\mu^\gamma(T)
\le
M_+^{\rm gd}(T).
\end{equation}
For the same $A_0$, the training-time window
$\mathcal T_{\rm stat}(m,n,\gamma)$ is nonempty. Moreover, there exist constants
$0<c_A<C_A<\infty$ such that every
$T\in\mathcal T_{\rm stat}(m,n,\gamma)$ satisfies
\begin{equation}
\label{eq:main-empirical-scale}
c_A A_{m,n}
\le
\widehat\mu^\gamma(T)
\le
C_A A_{m,n},
\qquad
\sup_{\alpha\in[\rho,\ell]}
\left|
F(
\widehat\theta^\gamma(T);
\alpha
)
-
\alpha
\right|
\le
\frac{C}{A_{m,n}}.
\end{equation}
Finally, $T_{\rm rate}^{\rm gd}-T_0 \asymp T_{\rm stat}^{\rm gd}-T_0 \asymp A_{m,n}^{4}$.
Equivalently, finite pretraining certifies order $A_{m,n}^{4}/\gamma$ post-entry gradient steps.
\end{theorem}

Appendix~\ref{sec:empirical-pretraining} proves the theorem and the well-posedness of the late window $\mathcal T_{\rm stat}$ used below for transfer.

\subsection{Checkpoint-conditioned fresh-task guarantee}
\begin{theorem}[Checkpoint-conditioned transfer]
\label{thm:main-transfer}
Fix $\eta\in(0,\Delta)$. There exists a constant $q_\eta^{\rm fr}<\infty$ such that the following holds whenever the conditions of Theorem~\ref{thm:main-empirical} are satisfied and $A_{m,n} \ge q_\eta^{\rm fr}.$ Let $\mathcal F_{\rm pre}$ denote the $\sigma$-field generated by the pretraining corpus, and let $\widehat T$ be any deterministic or $\mathcal F_{\rm pre}$-measurable training time satisfying $\widehat T \in \mathcal T_{\rm stat}(m,n,\gamma) \ \text{on } \mathcal E_{\rm train}.$ Freeze $\widehat\theta:=\widehat\theta^\gamma(\widehat T)$ before the fresh threshold, fresh-context size, confidence parameter, or fresh observations are revealed.

Then, on $\mathcal E_{\rm train}$, for every fixed $\alpha^{\rm fr}\in[\rho+\eta,\ell]$, every integer $N\ge1$, and every $u\ge1$,
\begin{equation}
\label{eq:main-transfer-conditional}
\begin{aligned}
\Prob_{\rm fr}\Bigg(
    &\left.
    \left|
        F_N\!\left(
            \widehat\theta;
            D_N^{\rm fr}
        \right)
        -
        \alpha^{\rm fr}
    \right|
    \le
    C_\eta
    \left[
        \frac{
            1+\log\!\left(
                1+
                C_{\rm fr}
                \widehat\mu^\gamma(\widehat T)u
            \right)
        }{
            \widehat\mu^\gamma(\widehat T)
        }
        +
        \frac{u}{N}
    \right]
    \,\right|\,
    \mathcal F_{\rm pre}
\Bigg)
\ge
1-4e^{-u}.
\end{aligned}
\end{equation}
\end{theorem}

The checkpoint is selected entirely from pretraining information and is independent of the fresh-context size $N$. The first term in \eqref{eq:main-transfer-conditional} depends on the realized frozen concentration $\widehat\mu$, whereas the second is the fresh-context localization error. Substituting the scale bounds from Theorem~\ref{thm:main-empirical} yields the $A_{m,n}$-level statement in Theorem~\ref{thm:learned-estimation}; Appendix~\ref{app:proof-learned-estimation} performs this assembly.

Appendix~\ref{appsubsec:freeze-then-transfer} proves the conditional statement from Proposition~\ref{prop:fixed-estimator-v3} and a conditioning argument.

\begin{corollary}[Fresh-query classification]
\label{cor:main-query-classification}
Under the setting of Theorem~\ref{thm:learned-estimation}, let $X_q^{\rm fr}\sim\Unif(0,1)$ be an independent fresh query and define $Y_q^{\rm fr}:=\mathbf 1\{X_q^{\rm fr}\le\alpha^{\rm fr}\}$ and $\widehat Y_q^{\rm fr}:=\mathbf 1\{X_q^{\rm fr}\le F_N(\widehat\theta;D_N^{\rm fr})\}$. Then
\begin{equation}
\label{eq:main-query-classification-exact}
    \Prob\!\left(
        \widehat Y_q^{\rm fr}
        \neq
        Y_q^{\rm fr}
        \,\middle|\,
        \widehat\theta,
        D_N^{\rm fr},
        \alpha^{\rm fr}
    \right)
    =
    \left|
        F_N(
            \widehat\theta;
            D_N^{\rm fr}
        )
        -
        \alpha^{\rm fr}
    \right|.
\end{equation}
Hence Theorem~\ref{thm:learned-estimation} yields the same $\widetilde O((m\wedge n)^{-1}+N^{-1})$ rate for fresh-query classification.
\end{corollary}

\begin{proof}[Proof of Corollary~\ref{cor:main-query-classification}]
Condition on $\widehat\theta$, $D_N^{\rm fr}$, and $\alpha^{\rm fr}$, so the two thresholds $F_N(\widehat\theta;D_N^{\rm fr})$ and $\alpha^{\rm fr}$ are fixed points in $[0,1]$. The indicators $\mathbf1\{X_q^{\rm fr}\le F_N\}$ and $\mathbf1\{X_q^{\rm fr}\le\alpha^{\rm fr}\}$ disagree exactly when $X_q^{\rm fr}$ lies between these two points. Since $X_q^{\rm fr}\sim\Unif(0,1)$, the conditional probability of this interval is their absolute difference, proving \eqref{eq:main-query-classification-exact}. The rate follows from Theorem~\ref{thm:learned-estimation}.
\end{proof}

\subsection{Assembly of the main learned-estimation theorem}
\label{app:proof-learned-estimation}
\begin{proof}[Proof of Theorem~\ref{thm:learned-estimation}]
Let $\mathcal I_{m,n,\gamma}$ contain the gradient-descent iteration indices corresponding to $\mathcal T_{\rm stat}(m,n,\gamma)$ under $T=T_0+\gamma(s-s_0)$. Theorem~\ref{thm:main-empirical} makes this interval deterministic and nonempty and gives the stated scale and step counts. On $\mathcal E_{\rm train}$, substitute $c_A A_{m,n}\le\widehat\mu\le C_A A_{m,n}$ into \eqref{eq:main-transfer-conditional}, absorbing fixed constants. This gives \eqref{eq:learned-estimation-v3} with its conditional probability unchanged. Adding the failure probability of $\mathcal E_{\rm train}$ gives the joint guarantee.
\end{proof}

\subsection{Short proofs of main-text results}
\label{app:short-main-proofs}

\subsubsection{Gradient identities}
\label{app:proof-main-identities}
\begin{proof}[Proof of Lemma~\ref{lem:main-gradient-identities}]
For fixed $\alpha$, let $\E_\theta[\,\cdot\mid\alpha]$ denote expectation under the density on $[0,1]$ proportional to
$\exp\{\lambda Y(x)+\mu x\}$, where $Y(x)=\mathbf1\{x\le\alpha\}$. Then $F(\theta;\alpha)=\E_\theta[X\mid\alpha]$. Differentiating this bounded exponential-family expectation gives
\[
 \partial_\lambda F(\theta;\alpha)
 =\Cov_\theta(X,Y\mid\alpha),
 \qquad
 \partial_\mu F(\theta;\alpha)
 =\Cov_\theta(X,X\mid\alpha)
 =\Var_\theta(X\mid\alpha).
\]
Because $X,Y\in[0,1]$, differentiation may be passed through the integral defining the population risk. Since $\alpha\sim\Unif([\rho,\ell])$ and $\Delta=\ell-\rho$, differentiating
$\mathcal R_\rho(\theta)=\frac1{2\Delta}\int_\rho^\ell(F(\theta;\alpha)-\alpha)^2\,d\alpha$
and substituting the two identities yields the claimed gradient formula.
\end{proof}

\subsubsection{Finite-pretraining summary}
\label{app:proof-finite-summary}
\begin{proof}[Proof of Corollary~\ref{cor:finite-summary-v3}]
Apply Theorem~\ref{thm:main-empirical} on its event $\mathcal E_{\rm train}$. Its interval $[T_0,T_{\rm stat}^{\rm gd}]\cap\mathbb T_\gamma$ gives entry into and persistence within $\mathcal G$. Choose the late subinterval to be $\mathcal T_{\rm stat}(m,n,\gamma)$. Equation~\eqref{eq:main-empirical-scale} gives the two asserted bounds at every checkpoint in this subinterval, and the theorem gives both the stated event probability and the $\Theta(A_{m,n}^4/\gamma)$ post-entry step count.
\end{proof}

\subsubsection{One-head boundary obstruction}
\label{app:proof-boundary-obstruction}
\begin{proof}[Proof of Proposition~\ref{prop:main-boundary-obstruction}]
At $\alpha=0$, one has $Y=0$ almost surely, so the factor involving $k$ disappears and
\[
 F_+(k,\mu;0)
 =\frac{\int_0^1 x e^{\mu x}\,dx}{\int_0^1 e^{\mu x}\,dx}
 =1-\frac1\mu+\frac1{e^\mu-1}.
\]
At $\alpha=1$, one has $Y=1$ almost surely. Under the reflected score, the constant factor $e^{-\mu k}$ cancels from the normalized average, giving
\[
 F_-(k,\mu;1)
 =\frac{\int_0^1 x e^{-\mu x}\,dx}{\int_0^1 e^{-\mu x}\,dx}
 =\frac1\mu-\frac1{e^\mu-1}.
\]
Taking $\mu\to\infty$ proves \eqref{eq:main-boundary-limits}. Thus each fixed orientation incurs order-one error at one physical endpoint, which rules out uniform consistency on $[0,1]$.
\end{proof}

\section{Structural Reduction and Basic Identities}
\label{sec:structural-identities}
Throughout, $c,C>0$ denote generic constants that may change from line to line, while named constants remain fixed after they are introduced. Dependence on fixed parameters is indicated by subscripts when needed.

\subsection{Two-component structural reduction}
\label{subsec:mixture-identities}
For $\mu>0$, define the component masses
\begin{equation}
\label{eq:component-masses}
    A_{1,\mu}(\alpha)
    :=
    \int_0^\alpha e^{\mu x}\,dx
    =
    \frac{e^{\mu\alpha}-1}{\mu},
    \qquad
    A_{2,\mu}(\alpha)
    :=
    \int_\alpha^1 e^{\mu x}\,dx
    =
    \frac{e^\mu-e^{\mu\alpha}}{\mu},
\end{equation}
and the corresponding mixture weights
\begin{equation}
\label{eq:mixture-weights}
    p_\theta(\alpha)
    :=
    \frac{
        e^\lambda A_{1,\mu}(\alpha)
    }{
        e^\lambda A_{1,\mu}(\alpha)
        +
        A_{2,\mu}(\alpha)
    },
    \qquad
    q_\theta(\alpha)
    :=
    1-p_\theta(\alpha).
\end{equation}

Let $m_{j,\mu}(\alpha)$ and $s_{j,\mu}^2(\alpha)$ denote the mean and
variance of the density proportional to $e^{\mu x}$ on
$I_1=[0,\alpha]$ and $I_2=(\alpha,1]$, respectively. Then
\begin{align*}
    m_{1,\mu}(\alpha)
    &=
    \alpha-\frac1\mu+
    \frac{\alpha}{e^{\mu\alpha}-1},
    &
    m_{2,\mu}(\alpha)
    &=
    1-\frac1\mu+
    \frac{1-\alpha}{e^{\mu(1-\alpha)}-1},\\
    s_{1,\mu}^2(\alpha)
    &=
    \frac1{\mu^2}
    -
    \frac{\alpha^2e^{-\mu\alpha}}
    {(1-e^{-\mu\alpha})^2},
    &
    s_{2,\mu}^2(\alpha)
    &=
    \frac1{\mu^2}
    -
    \frac{(1-\alpha)^2e^{-\mu(1-\alpha)}}
    {(1-e^{-\mu(1-\alpha)})^2}.
\end{align*}
Finally, define the component mean separation $d_\mu(\alpha):=m_{2,\mu}(\alpha)-m_{1,\mu}(\alpha).$
The exact representation and its large-$\mu$ behavior are summarized in the following structural lemma.

\begin{lemma}[Two-component representation and asymptotics]
\label{lem:mixture-representation}
For every $\mu>0$ and $\alpha\in(0,1)$,
\begin{equation}
\label{eq:mixture-F}
    F(\theta;\alpha)
    =
    p_\theta(\alpha)m_{1,\mu}(\alpha)
    +
    q_\theta(\alpha)m_{2,\mu}(\alpha),
\end{equation}
\begin{equation}
\label{eq:mixture-covariance}
    \Cov_\theta(X,Y\mid\alpha)
    =
    -p_\theta(\alpha)
    q_\theta(\alpha)
    d_\mu(\alpha),
\end{equation}
and
\begin{align}
\label{eq:mixture-variance}
    \Var_\theta(X\mid\alpha)
    &=
    p_\theta(\alpha)s_{1,\mu}^2(\alpha)
    +
    q_\theta(\alpha)s_{2,\mu}^2(\alpha)
    +
    p_\theta(\alpha)q_\theta(\alpha)
    d_\mu(\alpha)^2.
\end{align}
Moreover, whenever $k=\lambda/\mu$,
\begin{equation}
\label{eq:mixture-odds}
    \frac{p_\theta(\alpha)}{q_\theta(\alpha)}
    =
    e^{\mu(k+\alpha-1)}
    \frac{
        1-e^{-\mu\alpha}
    }{
        1-e^{-\mu(1-\alpha)}
    }.
\end{equation}

Uniformly over $\alpha\in[\rho,\ell]$, as $\mu\to\infty$,
\begin{equation}
\label{eq:uniform-mean-asymptotics}
    m_{1,\mu}(\alpha)
    =
    \alpha-\frac1\mu+O(e^{-\rho\mu}),
    \qquad
    m_{2,\mu}(\alpha)
    =
    1-\frac1\mu+O(e^{-\rho\mu}),
\end{equation}
\begin{equation}
\label{eq:uniform-variance-asymptotics}
    s_{1,\mu}^2(\alpha)
    =
    \frac1{\mu^2}+O(e^{-\rho\mu}),
    \qquad
    s_{2,\mu}^2(\alpha)
    =
    \frac1{\mu^2}+O(e^{-\rho\mu}),
\end{equation}
and
\begin{equation}
\label{eq:uniform-separation-asymptotics}
    d_\mu(\alpha)
    =
    1-\alpha+O(e^{-\rho\mu}).
\end{equation}
In particular, there exist constants $0<c_s<C_s<\infty$ and $\mu_s>0$, depending only on $\rho$, such that
\begin{equation}
\label{eq:component-variance-bounds}
    \frac{c_s}{\mu^2}
    \le
    s_{j,\mu}^2(\alpha)
    \le
    \frac{C_s}{\mu^2},
    \qquad
    j\in\{1,2\},
\end{equation}
for every $\mu\ge\mu_s$, uniformly over $\alpha\in[\rho,\ell]$.
\end{lemma}

\begin{proof}
The total tilted mass on $[0,\alpha]$ is $e^\lambda A_{1,\mu}(\alpha)$, while the total tilted mass on $[\alpha,1]$ is $A_{2,\mu}(\alpha)$. Hence $p_\theta(\alpha)$ and $q_\theta(\alpha)$ are the probabilities of the left and right components, respectively, which proves \eqref{eq:mixture-F}.

For the covariance identity, suppress the common dependence on $\alpha$. Since $Y=1$ on the left component and $Y=0$ on the right component,
\begin{align*}
    \Cov_\theta(X,Y\mid\alpha)
    &=
    \E_\theta[XY\mid\alpha]
    -
    \E_\theta[X\mid\alpha]\E_\theta[Y\mid\alpha]\\
    &=
    p_\theta m_{1,\mu}
    -
    p_\theta
    \left(
        p_\theta m_{1,\mu}
        +
        q_\theta m_{2,\mu}
    \right)\\
    &=
    p_\theta q_\theta
    (m_{1,\mu}-m_{2,\mu})\\
    &=
    -p_\theta q_\theta{d_\mu}.
\end{align*}
This proves \eqref{eq:mixture-covariance}. The variance formula \eqref{eq:mixture-variance} follows from the law of total variance. Finally,
\begin{align*}
    \dfrac{p_\theta(\alpha)}{q_\theta(\alpha)}
    &=
    e^\lambda
    \dfrac{A_{1,\mu}(\alpha)}
    {A_{2,\mu}(\alpha)}
    =
    e^{\lambda+\mu\alpha-\mu}
    \dfrac{1-e^{-\mu\alpha}}
    {1-e^{-\mu(1-\alpha)}}.
\end{align*}
Using $\lambda=k\mu$ gives \eqref{eq:mixture-odds}.

Since $\alpha\ge\rho$ and $1-\alpha\ge\rho$ on $[\rho,\ell]$, the correction terms in $m_{1,\mu}(\alpha)$ and $m_{2,\mu}(\alpha)$ are uniformly $O(e^{-\rho\mu})$. This proves the mean expansions. The variance expansions follow similarly from $s^2_{1,\mu}(\alpha)$ and $s^2_{2,\mu}(\alpha)$. Subtracting the two mean expansions gives the expansion for ${d_\mu}$. The two-sided variance bounds follow after increasing $\mu_s$.
\end{proof}

The representation immediately yields the bias expansion used throughout the subsequent dynamics.

\begin{corollary}[Uniform bias expansion]
\label{cor:bias-expansion}
Uniformly over $\alpha\in[\rho,\ell]$, as $\mu\to\infty$,
\begin{equation}
\label{eq:uniform-bias-expansion}
    F(\theta;\alpha)-\alpha
    =
    q_\theta(\alpha)(1-\alpha)
    -
    \frac1\mu
    +
    O(e^{-\rho\mu}).
\end{equation}
The remainder is uniform over $\lambda\in\R$.
\end{corollary}

\begin{proof}
Using \eqref{eq:mixture-F} and \eqref{eq:uniform-mean-asymptotics},
\begin{align*}
    F(\theta;\alpha)
    &=
    p_\theta(\alpha)
    \left(
        \alpha-\dfrac1\mu
    \right)
    +
    q_\theta(\alpha)
    \left(
        1-\dfrac1\mu
    \right)
    +
    O(e^{-\rho\mu})\\
    &=
    p_\theta(\alpha)\alpha
    +
    q_\theta(\alpha)
    -
    \dfrac1\mu
    +
    O(e^{-\rho\mu}).
\end{align*}
Since $p_\theta(\alpha)+q_\theta(\alpha)=1$, subtracting $\alpha$ gives the claim. Uniformity in $\lambda$ follows from $0\le p_\theta(\alpha),q_\theta(\alpha)\le1$.
\end{proof}

The odds identity \eqref{eq:mixture-odds} gives, uniformly over $\alpha\in[\rho,\ell]$,
\begin{equation}
\label{eq:mixture-odds-asymptotic}
    \frac{p_\theta(\alpha)}{q_\theta(\alpha)}
    =
    e^{\mu(k+\alpha-1)}
    \left\{
        1+O(e^{-\rho\mu})
    \right\}.
\end{equation}
Hence the two components exchange dominance near $\alpha_c:=1-k.$ Thus $k$ determines the location of the attention transition across the task family, while $\mu$ determines its sharpness. This is the structural reduction underlying the population dynamics.

\subsection{Effective coordinates and ratio dynamics}
\label{subsec:basic-dynamics}

The preceding representation shows that $(k,\mu)$ are the natural coordinates for the population flow: $k$ moves the transition location $\alpha_c$, while $\mu$ controls concentration.

\begin{lemma}[Ratio dynamics and sign structure]
\label{lem:ratio-identity}
Along the population gradient flow, whenever $\mu(t)>0$,
\begin{equation}
\label{eq:ratio-identity}
    \dot k(t)
    =
    -
    \frac{
        S(\lambda(t),\mu(t))
    }{
        \mu(t)^2
    },
\end{equation}
where $S(\lambda,\mu) := \mu\partial_\lambda\mathcal R_\rho(\lambda,\mu) - \lambda\partial_\mu\mathcal R_\rho(\lambda,\mu).$ Equivalently,
\begin{equation}
\label{eq:S-integral}
    S(\lambda,\mu)
    =
    \frac1\Delta
    \int_\rho^\ell
    \bigl(F(\theta;\alpha)-\alpha\bigr)
    H_\theta(\alpha)
    \,d\alpha,
\end{equation}
where $H_\theta(\alpha):=\mu\Cov_\theta(X,Y\mid\alpha)-\lambda\Var_\theta(X\mid\alpha).$ For every $\lambda,\mu\ge0$ and $\alpha\in(0,1)$, $H_\theta(\alpha)\le0.$
\end{lemma}

\begin{proof}
Since $k=\lambda/\mu$, $\dot k = ({\mu\dot\lambda-\lambda\dot\mu})/{\mu^2}.$ Using $\dot\lambda=-\partial_\lambda \mathcal R_\rho$ and $\dot\mu=-\partial_\mu \mathcal R_\rho,$ gives \eqref{eq:ratio-identity} and $S(\lambda,\mu)$. Formula \eqref{eq:S-integral} follows directly from Lemma~\ref{lem:main-gradient-identities}.

Suppose first that $\mu>0$. By \eqref{eq:mixture-covariance},
$\Cov_\theta(X,Y\mid\alpha)=-p_\theta(\alpha)q_\theta(\alpha){d_\mu}(\alpha)\le0$
because ${d_\mu}(\alpha)\ge0$. Since $\Var_\theta(X\mid\alpha)\ge0, \qquad \lambda,\mu\ge0,$ it follows that $H_\theta(\alpha)\le0$. If $\mu=0$, then $H_\theta(\alpha)=-\lambda\Var_\theta(X\mid\alpha)\le0.$
\end{proof}

\subsection{Well-posedness of the population flow}
\label{subsec:population-well-posedness}

The population dynamics used below are globally well posed. We record the standard ODE fact and its application here for completeness. The derivative identities needed below are those of Lemma~\ref{lem:main-gradient-identities} in the main text.

\begin{lemma}[Bounded locally Lipschitz gradient flows]
\label{lem:bounded-lipschitz-gradient-flow}
Let $G:\R^2\to\R$ be continuously differentiable, and suppose that $\nabla G$ is locally Lipschitz. If there exists $B<\infty$ such that $\|\nabla G(\theta)\| \le B$ for every $\theta\in\R^2$, then, for every initial condition $\theta(0)\in\R^2$, the gradient flow $\dot\theta(t) = -\nabla G(\theta(t))$ admits a unique solution defined for all $t\ge0$.
\end{lemma}
\begin{proof}
Since $-\nabla G$ is locally Lipschitz, the Picard--Lindelöf theorem gives a unique maximal solution $\theta:[0,T_{\rm life})\to\R^2.$ The uniform gradient bound gives $\|\dot\theta(t)\| = \|\nabla G(\theta(t))\| \le B.$ Hence, for $t<T_{\rm life}$, $\|\theta(t)\| \le \|\theta(0)\|+Bt.$ Thus the trajectory is bounded on every finite time interval. By the continuation theorem for locally Lipschitz ODEs, a finite maximal time would force $\|\theta(t)\|\to\infty \ \text{as }t\uparrow T_{\rm life},$ which is impossible. Therefore $T_{\rm life}=\infty$.
\end{proof}

\begin{lemma}[Population global well-posedness]
\label{lem:global-well-posedness}
For every initial condition $\theta(0)\in\R^2$, the population gradient flow $\dot\theta(t) = -\nabla\mathcal R_\rho(\theta(t))$ admits a unique solution defined for all $t\ge0$.
\end{lemma}
\begin{proof}
The map $\theta\mapsto \mathcal R_\rho(\theta)$ is smooth. Indeed, on every compact subset of $\R^2$, the integrands defining $F(\theta;\alpha)$ and all their $\theta$-derivatives are uniformly bounded over $(x,\alpha)\in[0,1]\times[\rho,\ell],$ so differentiation under the integral sign is valid.

Moreover, $|F(\theta;\alpha)-\alpha|\le1.$ Since $X,Y\in[0,1]$, $\Var_\theta(X\mid\alpha)\le\dfrac14,$ $|\Cov_\theta(X,Y\mid\alpha)|\le\dfrac14.$
Using Lemma~\ref{lem:main-gradient-identities}, we obtain $|\partial_\lambda \mathcal R_\rho(\theta)|\le\dfrac14$ and $|\partial_\mu \mathcal R_\rho(\theta)|\le\dfrac14.$ Thus $\|\nabla \mathcal R_\rho(\theta)\|\le\dfrac{\sqrt2}{4}$ for every $\theta$. The claim follows from Lemma~\ref{lem:bounded-lipschitz-gradient-flow}.
\end{proof}

The identities above reduce the population training problem to the interaction between the prediction bias $F(\theta;\alpha)-\alpha$ and the moving attention transition $\alpha_c$. Section~\ref{sec:population-dynamics} shows how this interaction first corrects the ratio, then transports the transition across the task interval, and finally produces logarithmic capture near the endpoint.

\section{Proof of Population Capture and Sharpening}
\label{sec:population-dynamics}

This section proves Theorem~\ref{thm:main-population}. The proof has four modules: outer ratio correction, transport through the task interior, logarithmic endpoint capture, and radial persistence. The first three produce finite-time entry into the logarithmic tube; the last prevents the concentration scale from leaving the large-$\mu$ regime before capture.

Fix $C_{\rm band}\ge2$. For $\mu>1$, set $L_{\rm band}(\mu):=\rho-C_{\rm band}\log\mu/\mu$, $U_{\rm band}(\mu):=\ell+C_{\rm band}\log\mu/\mu$, and $b_c(\mu):=\ell+(\log\mu+c)/\mu$. The endpoint drift profiles are
\begin{equation}
\label{eq:ratio-drift-profile}
 d_k(c):=\ell+\frac{\ell(1+\ell^2)}\Delta e^{-c}-\frac{\ell^2(1+\ell^2)}{2\Delta}e^{-2c},
 \qquad
 d_{\rm rad}(c):=1+\frac{\ell^2}\Delta e^{-c}-\frac{\ell^3}{2\Delta}e^{-2c}.
\end{equation}
Lemma~\ref{lem:choose-c0} selects $c_0$ with $d_k(c_0)<0<d_{\rm rad}(c_0)$. Define $\mathcal I_{\log}(\mu):=[b_{c_0}(\mu),U_{\rm band}(\mu)]$ and use the main-text progress region $\mathcal G=\{\mu\ge\mu_*,\ k\in\mathcal I_{\log}(\mu)\}$. Its side barriers are $h_-(\lambda,\mu):=\lambda/\mu-b_{c_0}(\mu)$ and $h_+(\lambda,\mu):=U_{\rm band}(\mu)-\lambda/\mu$.

\textbf{Threshold convention.}
All constants and large-$\mu$ thresholds in this section depend only on the fixed structural parameters (and on $\kappa$ or $\eta$ when displayed). Because only finitely many such thresholds are used in the final assembly, we increase the base radius $\mu_*$ once so that every cited estimate holds whenever $\mu\ge\mu_*$.

\subsection{Barrier tools and regional drift estimates}
\label{subsec:population-regional-tools}

For compactness, write $B_\theta(\alpha):=F(\theta;\alpha)-\alpha$,
$\sigma_\theta^2(\alpha):=\Var_\theta(X\mid\alpha)$, and
$J_\theta(\alpha):=-H_\theta(\alpha)$.
Since $J_\theta=-H_\theta$, Lemma~\ref{lem:ratio-identity} gives
\begin{equation}
\label{eq:K-ratio-identity}
    S(\lambda,\mu)
    =
    -\dfrac1\Delta
    \int_\rho^\ell
    B_\theta(\alpha)J_\theta(\alpha)\,d\alpha,
    \qquad
    \dot k
    =
    \dfrac1{\Delta\mu^2}
    \int_\rho^\ell
    B_\theta(\alpha)J_\theta(\alpha)\,d\alpha.
\end{equation}
For $0<\eta<\Delta/4$, define the lower layer, fixed interior strip, and upper rail tube by
\begin{equation*}
    \mathcal L_\eta(\mu)
    :=
    [L_{\rm band}(\mu),\rho+\eta],
    \qquad
    \mathcal I_\eta
    :=
    [\rho+\eta,\ell-\eta],
    \qquad
    \mathcal T_\eta(\mu)
    :=
    [\ell-\eta,U_{\rm band}(\mu)].
\end{equation*}
\begin{lemma}[Finite-time entry from a one-sided error inequality]
\label{lem:abstract-finite-entry}
Let $T\in(t_0,\infty]$, and let
$e:[t_0,T)\to[0,\infty), \ \mu:[t_0,T)\to(0,\infty)$
be continuously differentiable. Define the entry time
$\tau_e:=\inf\{t\ge t_0:e(t)=0\}.$
Suppose that, whenever $e(t)>0$,
\begin{equation}
\label{eq:abstract-gap-drift}
    \dot e(t)\le-\dfrac{a}{\mu(t)^p}
\end{equation}
for some constants $a>0$ and $p>0$.
\begin{enumerate}[label=\textnormal{(\roman*)}]
\item If $\mu(t)\le\mu(t_0)$ whenever $e(t)>0$, then, for every $t<T\wedge\tau_e$,
    \begin{equation}
    \label{eq:abstract-entry-time-decreasing}
        e(t)
        \le
        e(t_0)-\dfrac{a}{\mu(t_0)^p}(t-t_0).
    \end{equation}
Consequently, if $\tau_e\ge T$, then
    $T-t_0 \le \dfrac{\mu(t_0)^p e(t_0)}{a}.$
In particular, if $T=\infty$, then $\tau_e<\infty$. \item If $p=4$ and
    $\dot\mu(t)\le\dfrac{A}{\mu(t)^3}$
whenever $e(t)>0$, for some $A>0$, then, for every $t<T\wedge\tau_e$,
    \begin{equation}
    \label{eq:abstract-log-gap-bound}
        e(t)
        \le
        e(t_0)
        -
        \dfrac{a}{4A}
        \log\left(
            1+\dfrac{4A(t-t_0)}{\mu(t_0)^4}
        \right).
    \end{equation}
In particular, if $T=\infty$, then $\tau_e<\infty$.
\end{enumerate}
\end{lemma}
\begin{proof}
If $e(t_0)=0$, then $\tau_e=t_0$ and there is nothing to prove. Hence assume $e(t_0)>0$. On $[t_0,T\wedge\tau_e)$, one has $e(t)>0$, so the hypotheses are valid.

For \textnormal{(i)}, on $[t_0, T\wedge\tau_e),$ $\dot e(t)\le-\frac{a}{\mu(t_0)^p},$ and integration gives \eqref{eq:abstract-entry-time-decreasing}. If $\tau_e\ge T$, then $e(t)\ge0$ for all $t<T$, so letting $t\uparrow T$ gives $T-t_0 \le \frac{\mu(t_0)^p e(t_0)}{a}.$ If $T=\infty$, this is impossible; hence $\tau_e<\infty$.

For \textnormal{(ii)}, the radial bound gives $\frac{d}{dt}\mu(t)^4 = 4\mu(t)^3\dot\mu(t) \le 4A.$ Hence, on $[t_0, T\wedge\tau_e),$ $\mu(t)^4 \le \mu(t_0)^4+4A(t-t_0).$ Using \eqref{eq:abstract-gap-drift} with $p=4$ and integrating,
\begin{align*}
    e(t)
    &\le
    e(t_0)
    -
    a\int_{t_0}^t\dfrac{ds}{\mu(s)^4} \\
    &\le
    e(t_0)
    -
    a\int_{t_0}^t
    \dfrac{ds}{\mu(t_0)^4+4A(s-t_0)} \\
    &=
    e(t_0)
    -
    \dfrac{a}{4A}
    \log\left(
        1+\dfrac{4A(t-t_0)}{\mu(t_0)^4}
    \right).
\end{align*}
This proves \eqref{eq:abstract-log-gap-bound}. If $T=\infty$ and $\tau_e=\infty$, the logarithm diverges, forcing the right-hand side eventually to become negative, contradicting $e(t)\ge0$. Hence $\tau_e<\infty$.
\end{proof}
\begin{lemma}[Moving-interval invariance]
\label{lem:moving-interval-invariance}
Let $\mu_{\rm stop}>0$, and let $a,b:(\mu_{\rm stop},\infty)\to\R$ be $C^1$ functions satisfying $a(\mu)<b(\mu)$ for $\mu>\mu_{\rm stop}$. Let $k,\mu$ be continuously differentiable on $[t_0,T)$, with $\mu(t_0)>\mu_{\rm stop}$, and set
$\tau_{\rm stop} := \inf\{t\ge t_0:\mu(t)\le\mu_{\rm stop}\}.$
Define
$h_-(t):=k(t)-a(\mu(t)), \ h_+(t):=b(\mu(t))-k(t).$
Assume that, whenever $t<T\wedge\tau_{\rm stop}$,
\begin{equation*}
    h_-(t)=0 \Longrightarrow \dot h_-(t)>0,
    \qquad
    h_+(t)=0 \Longrightarrow \dot h_+(t)>0.
\end{equation*}
If $h_-(t_0)\ge0$ and $h_+(t_0)\ge0$, then $k(t)\in[a(\mu(t)),b(\mu(t))]$ for every $t\in[t_0,T\wedge\tau_{\rm stop})$.
\end{lemma}
\begin{proof}
Suppose not, and define the first violation time
\begin{equation*}
    \tau
    :=
    \inf\left\{
        t\in[t_0,T\wedge\tau_{\rm stop}):
        h_-(t)<0
        \ \text{or}\
        h_+(t)<0
    \right\}.
\end{equation*}
By continuity, $h_-(t),h_+(t)\ge0$ on $[t_0,\tau]$, and at least one of $h_-(\tau)$ and $h_+(\tau)$ is zero. Every active barrier has strictly positive derivative at $\tau$, while every inactive barrier is positive at $\tau$. Hence, by differentiability and continuity, both barriers remain nonnegative on $(\tau,\tau+\varepsilon)$ for some $\varepsilon>0$. This contradicts the definition of $\tau$. Therefore no violation occurs before $T\wedge\tau_{\rm stop}$.
\end{proof}
\subsubsection{Outer correction and dynamic-band entry}
\label{subsec:outer-correction}
The sign identity from Section~\ref{subsec:basic-dynamics} converts a uniform bias sign outside the support rails into a directional correction of the ratio. The next proposition quantifies this correction and controls the accompanying radial motion.
\begin{proposition}[Quantitative outer correction]
\label{prop:quant-outer}
Assume $C_{\rm band}\ge2$. There exist constants $c_U,C_U>0$ and $\mu_U>0$, depending only on $\rho$ and $C_{\rm band}$, such that, for all $(\lambda,\mu)$ with $\mu\ge\mu_U$,
\begin{equation*}
    k\ge U_{\rm band}(\mu)
    \quad\Longrightarrow\quad
    \dot k
    \le
    -\dfrac{c_U}{\mu^4},
    \qquad
    0<\dot\mu
    \le
    \dfrac{C_U}{\mu^3}.
\end{equation*}
Moreover, for every fixed $\kappa>0$, there exist constants $c_L(\kappa)>0$ and $\mu_L(\kappa)>0$ such that, for all $(\lambda,\mu)$ with $\mu\ge\mu_L(\kappa)$,
\begin{equation*}
    \kappa\le k<L_{\rm band}(\mu)
    \quad\Longrightarrow\quad
    \dot k
    \ge
    \dfrac{c_L(\kappa)}{\mu^3},
    \qquad
    \dot\mu<0.
\end{equation*}
\end{proposition}
\begin{proof}
We use the odds identity
$p_\theta(\alpha)/q_\theta(\alpha)=e^{\mu(k+\alpha-1)}(1-e^{-\mu\alpha})/(1-e^{-\mu(1-\alpha)})$,
where the second factor is uniformly bounded above and below on $\alpha\in[\rho,\ell]$ for all sufficiently large $\mu$.

First assume $k\ge U_{\rm band}(\mu)$. Then $\mu(k+\alpha-1)\ge C_{\rm band}\log\mu$ and hence
\begin{equation}
\label{eq:upper-outer-q-small-short}
    q_\theta(\alpha)\lesssim \mu^{-C_{\rm band}}
\end{equation}
uniformly in $\alpha$. By the bias expansion, $B_\theta(\alpha)=q_\theta(\alpha)(1-\alpha)-\dfrac1\mu+O(e^{-\rho\mu}),$ so, since $C_{\rm band}\ge2$,
\begin{equation}
\label{eq:upper-outer-bias-short}
    B_\theta(\alpha)\le -\dfrac{c}{\mu},
    \qquad
    |B_\theta(\alpha)|\lesssim \dfrac1\mu.
\end{equation}
Also $p_\theta\ge1/2$ for large $\mu$, and the component-variance bounds imply
$\sigma_\theta^2(\alpha)\ge p_\theta(\alpha)s_{1,\mu}^2(\alpha)\gtrsim\mu^{-2}$.
Since $\lambda=k\mu\ge \ell\mu$, we have
\begin{equation}
\label{eq:upper-outer-J-lower-short}
    J_\theta(\alpha)
    =
    \mu p_\theta q_\theta{d_\mu}+\lambda \sigma_\theta^2(\alpha)
    \ge
    \lambda \sigma_\theta^2(\alpha)
    \gtrsim
    \dfrac1\mu.
\end{equation}
Combining \eqref{eq:upper-outer-bias-short} and \eqref{eq:upper-outer-J-lower-short} gives
$\int_\rho^\ell B_\theta(\alpha)J_\theta(\alpha)\,d\alpha\le-c/\mu^2$.
Therefore, by \eqref{eq:K-ratio-identity}, $\dot k\le-\frac{c_U}{\mu^4}$ after naming the constant.

Furthermore, \eqref{eq:upper-outer-bias-short} gives $\partial_\mu \mathcal R_\rho<0$, and hence $\dot\mu>0$. For the upper bound, using \eqref{eq:upper-outer-q-small-short} and the mixture variance identity,
$\sigma_\theta^2(\alpha)\lesssim\mu^{-2}+q_\theta(\alpha)\lesssim\mu^{-2}$.
Thus $0<\dot\mu\lesssim\int_\rho^\ell |B_\theta(\alpha)|\sigma_\theta^2(\alpha)\,d\alpha\lesssim\mu^{-3}$.
This proves the upper outer estimates.

Now fix $\kappa>0$ and assume $\kappa\le k<L_{\rm band}(\mu)$. Then $\mu(k+\alpha-1) \le -C_{\rm band}\log\mu, \ \alpha\in[\rho,\ell],$ and hence
$p_\theta(\alpha)\lesssim \mu^{-C_{\rm band}}$ and
$q_\theta(\alpha)=1+O(\mu^{-C_{\rm band}})$
uniformly in $\alpha$. The bias expansion gives $B_\theta(\alpha) = (1-\alpha)-\frac1\mu+O(\mu^{-C_{\rm band}}),$ so, since $1-\alpha\ge\rho$,
\begin{equation}
\label{eq:lower-outer-bias-short}
    B_\theta(\alpha)\ge c_\rho>0
\end{equation}
for all sufficiently large $\mu$.

Also $q_\theta\ge1/2$ for large $\mu$, and therefore
$\sigma_\theta^2(\alpha)\ge q_\theta(\alpha)s_{2,\mu}^2(\alpha)\gtrsim\mu^{-2}$.
Since $\lambda=k\mu\ge\kappa\mu$,
$J_\theta(\alpha) \ge \lambda \sigma_\theta^2(\alpha) \gtrsim_\kappa \dfrac1\mu.$
Thus $\int_\rho^\ell B_\theta(\alpha)J_\theta(\alpha)\,d\alpha\gtrsim_\kappa\mu^{-1}$.
Using \eqref{eq:K-ratio-identity}, we obtain $\dot k\ge\frac{c_L(\kappa)}{\mu^3}$ after naming the constant. Finally, \eqref{eq:lower-outer-bias-short} gives $\partial_\mu \mathcal R_\rho>0$, and hence $ \dot\mu=-\partial_\mu \mathcal R_\rho<0.$ This proves the lower outer estimates.
\end{proof}
\begin{corollary}[Upper moving-boundary correction]
\label{cor:upper-moving-boundary-correction}
Assume $C_{\rm band}\ge2$, and let $c_U,C_U,\mu_U$ be the constants from Proposition~\ref{prop:quant-outer}. There exist constants $a_U,\mu_{U,\partial}>0$, depending only on $\rho$ and $C_{\rm band}$, such that, whenever
$\mu\ge\mu_{U,\partial}, \ k\ge U_{\rm band}(\mu),$
one has
$\dot k-U_{\rm band}'(\mu)\dot\mu \le -\dfrac{a_U}{\mu^4},$
and
$0<\dot\mu\le\dfrac{C_U}{\mu^3}.$
\end{corollary}
\begin{proof}
By Proposition~\ref{prop:quant-outer}, in the region $k\ge U_{\rm band}(\mu)$, $\dot k\le-\frac{c_U}{\mu^4}, \ 0<\dot\mu\le\frac{C_U}{\mu^3}.$ Since $U_{\rm band}'(\mu) = C_{\rm band}\dfrac{1-\log\mu}{\mu^2},$ we have $|U_{\rm band}'(\mu)| \lesssim \frac{\log\mu}{\mu^2}.$ Therefore
$\dot k-U_{\rm band}'(\mu)\dot\mu
\le-c_U/\mu^4+C(\log\mu)/\mu^5\le-c_U/(2\mu^4)$
for all sufficiently large $\mu$. Taking $a_U:=c_U/2$ and increasing the threshold to $\mu_{U,\partial}$ proves the claim.
\end{proof}
\begin{proposition}[Stopped finite-time entry into the dynamic band]
\label{prop:finite-entry}
Assume $C_{\rm band}\ge2$.
\begin{enumerate}[label=\textnormal{(\roman*)}]
\item There exists $\mu_+>1$, depending only on $\rho$ and $C_{\rm band}$, such that if
    $k(0)>U_{\rm band}(\mu(0)), \ \mu(0)\ge\mu_+,$
then
    \begin{equation*}
        \tau_{\rm band}^{+}
        :=
        \inf\left\{
            t\ge0:
            k(t)\le U_{\rm band}(\mu(t))
        \right\}
        <\infty.
    \end{equation*}
\item For every fixed $\kappa>0$, there exists $\mu_-(\kappa)>1$, depending only on $\rho$, $C_{\rm band}$, and $\kappa$, such that, if
    $\kappa\le k(0)<L_{\rm band}(\mu(0)), \ \mu(0)>\mu_-(\kappa),$
and
    \begin{equation*}
        \tau_{\rm band}^{-}
        :=
        \inf\left\{
            t\ge0:
            k(t)\ge L_{\rm band}(\mu(t))
        \right\},
        \qquad
        \tau_{{\rm band},-}^{\rm exit}
        :=
        \inf\left\{
            t\ge0:
            \mu(t)\le\mu_-(\kappa)
        \right\},
    \end{equation*}
then
    \begin{equation}
    \label{eq:lower-outer-entry-time}
        \tau_{\rm band}^{-}\wedge\tau_{{\rm band},-}^{\rm exit}
        \le
        \dfrac{\mu(0)^3}{c_L(\kappa)}
        \left\{
            L_{\rm band}(\mu(0))-k(0)
        \right\}.
    \end{equation}
In particular, if $\tau_{{\rm band},-}^{\rm exit}>\tau_{\rm band}^{-}$, then $\tau_{\rm band}^{-}<\infty$ and the same bound holds with $\tau_{\rm band}^{-}\wedge\tau_{{\rm band},-}^{\rm exit}$ replaced by $\tau_{\rm band}^{-}$.
\end{enumerate}
\end{proposition}
\begin{proof}
\textbf{Upper outer region.} Choose $\mu_+$ larger than the threshold in Corollary~\ref{cor:upper-moving-boundary-correction}. Set $g_U(t):=k(t)-U_{\rm band}(\mu(t)).$ Before $\tau_{\rm band}^{+}$, one has $g_U(t)>0$, hence $k(t)>U_{\rm band}(\mu(t)).$ Corollary~\ref{cor:upper-moving-boundary-correction} gives
\begin{equation*}
    \dot g_U(t)
    =
    \dot k(t)-U_{\rm band}'(\mu(t))\dot\mu(t)
    \le
    -\dfrac{a_U}{\mu(t)^4},
    \qquad
    0<\dot\mu(t)\le\dfrac{C_U}{\mu(t)^3}.
\end{equation*}
In particular, $\mu(t)\ge\mu(0)\ge\mu_+$ before $\tau_{\rm band}^{+}$, so the estimates remain valid throughout the pre-entry interval. Lemma~\ref{lem:abstract-finite-entry}\textnormal{(ii)} therefore implies that $g_U$ reaches zero in finite time. Hence $\tau_{\rm band}^{+}<\infty$. \textbf{Lower outer region.} Fix $\kappa>0$. Choose $\mu_-(\kappa)>e$ sufficiently large that the lower estimates in Proposition~\ref{prop:quant-outer} hold whenever $\mu\ge\mu_-(\kappa), \ \kappa\le k<L_{\rm band}(\mu).$ Set
$g_L(t):=L_{\rm band}(\mu(t))-k(t), \ \tau_\kappa:=\inf\{t\ge0:k(t)<\kappa\}.$
Before $\tau_{\rm band}^{-}\wedge\tau_{{\rm band},-}^{\rm exit}\wedge\tau_\kappa,$ one has $\kappa\le k(t)<L_{\rm band}(\mu(t)), \ \mu(t)>\mu_-(\kappa).$ Therefore Proposition~\ref{prop:quant-outer} gives $\dot k(t)\ge\frac{c_L(\kappa)}{\mu(t)^3}, \ \dot\mu(t)<0.$ In particular, $k$ is strictly increasing on this stopped interval, so $\tau_\kappa$ cannot be the first stopping time.

Since $L_{\rm band}'(\mu) = C_{\rm band}\dfrac{\log\mu-1}{\mu^2}>0$ for $\mu>\mu_-(\kappa)$, we have
$\dot g_L(t)=L_{\rm band}'(\mu(t))\dot\mu(t)-\dot k(t)\le-c_L(\kappa)/\mu(t)^3$.
Moreover, $\dot\mu(t)<0$, so $\mu(t)\le\mu(0)$ on the stopped interval. Lemma~\ref{lem:abstract-finite-entry}\textnormal{(i)} gives
$\tau_{\rm band}^{-}\wedge\tau_{{\rm band},-}^{\rm exit}\wedge\tau_\kappa
\le\mu(0)^3g_L(0)/c_L(\kappa)$.
Since $\tau_\kappa$ cannot occur before $\tau_{\rm band}^{-}\wedge\tau_{{\rm band},-}^{\rm exit}$, this implies
$\tau_{\rm band}^{-}\wedge\tau_{{\rm band},-}^{\rm exit}
\le\mu(0)^3g_L(0)/c_L(\kappa)$.
Finally, $g_L(0)=L_{\rm band}(\mu(0))-k(0)$, so this is \eqref{eq:lower-outer-entry-time}. The final assertion is immediate.
\end{proof}
\subsubsection{Transport through the dynamic band}
\label{subsec:transport-band}
This subsection proves that a trajectory in the lower portion of the dynamic band is transported successively through the lower boundary layer and the fixed interior strip, after which it is captured by a forward-invariant upper rail tube.
\textbf{Lower boundary-layer escape}
\label{subsubsec:lower-layer}
The finite-time entry result places the trajectory inside the dynamic band. Entry from below gives only $k(t)\ge L_{\rm band}(\mu(t))$, so the trajectory may first lie in the lower layer $\mathcal L_\eta(\mu)$. The goal of this subsection is to prove finite-time entrance into the fixed interior strip $\mathcal I_\eta$.
\begin{proposition}[Lower boundary-layer drift estimates]
\label{prop:lower-boundary-layer-escape}
Assume $0<\eta<\dfrac{\Delta}{4}.$ There exist constants $c_{\rm low,\eta},C_{\rm low,\eta},\mu_{\rm low,\eta}>0$, depending only on $\rho,\eta,C_{\rm band}$, such that, whenever
$\mu\ge\mu_{\rm low,\eta}, \ k\in[L_{\rm band}(\mu),\rho+\eta],$
one has
$\dot k \ge \dfrac{c_{\rm low,\eta}}{\mu^3},$
and
\begin{equation}
\label{eq:lower-layer-mu-control}
    \dot\mu<0,
    \qquad
    |\dot\mu|
    \le
    \dfrac{C_{\rm low,\eta}}{\mu}.
\end{equation}
Moreover,
\begin{equation}
\label{eq:lower-layer-relative-loss}
    \dfrac{(-\dot\mu)_+}{\mu}
    \le
    C_{\rm low,\eta}\dot k.
\end{equation}
\end{proposition}
\begin{proof}
All implicit constants may depend on $\rho,\eta,C_{\rm band}$. Increase $\mu_{\rm low,\eta}$ so that, for $\mu\ge\mu_{\rm low,\eta}$, $L_{\rm band}(\mu)\ge\frac{\rho}{2}, \ \mu_{\rm low,\eta}>e.$ Thus
\begin{equation}
\label{eq:lower-layer-k-range-short}
    \dfrac{\rho}{2}\le k\le\rho+\eta.
\end{equation}
Set $z:=\mu(k+\alpha-1).$ The odds identity gives
\begin{equation}
\label{eq:lower-layer-logistic-comparison-short}
    q_\theta(\alpha)\asymp\dfrac1{1+e^z},
    \qquad
    p_\theta(\alpha)q_\theta(\alpha)
    \lesssim
    \dfrac{e^z}{(1+e^z)^2},
\end{equation}
uniformly for $\alpha\in[\rho,\ell]$. Hence, by the change of variables $z=\mu(k+\alpha-1)$,
\begin{equation}
\label{eq:lower-layer-pq-integral-short}
    \int_\rho^\ell p_\theta(\alpha)q_\theta(\alpha)\,d\alpha
    \lesssim
    \dfrac1\mu.
\end{equation}
The uniform component asymptotics imply
\begin{equation}
\label{eq:lower-layer-basic-bounds-short}
\begin{aligned}
    B_\theta(\alpha)
    &=
    q_\theta(\alpha)(1-\alpha)-\dfrac1\mu+O(e^{-\rho\mu}),\\
    |B_\theta(\alpha)|
    &\lesssim1,\\
    \sigma_\theta^2(\alpha)
    &\lesssim
    \dfrac1{\mu^2}+p_\theta(\alpha)q_\theta(\alpha),\\
    0\le J_\theta(\alpha)
    &\lesssim
    \mu p_\theta(\alpha)q_\theta(\alpha)+\dfrac1\mu.
\end{aligned}
\end{equation}
Here the last bound follows from $J_\theta = \mu p_\theta q_\theta{d_\mu}+k\mu \sigma_\theta^2$ and \eqref{eq:lower-layer-k-range-short}.

Let $J_\eta:=[\rho,\ell-2\eta].$ Since $\eta<\Delta/4$, this interval has positive length. For $\alpha\in J_\eta$, $k+\alpha-1 \le (\rho+\eta)+(\ell-2\eta)-1 = -\eta.$ Thus $p_\theta(\alpha)\lesssim e^{-\eta\mu}$ and $q_\theta(\alpha)=1+O(e^{-\eta\mu})$ on $J_\eta$. Since $1-\alpha\ge\rho+2\eta$, the bias expansion gives $B_\theta(\alpha)\gtrsim1.$ Also $\sigma_\theta^2(\alpha)\gtrsim\mu^{-2}$ on $J_\eta$. Using $\lambda=k\mu\ge\rho\mu/2$ and $\Cov_\theta(X,Y\mid\alpha)\le0$,
$J_\theta(\alpha)=-H_\theta(\alpha)\ge\lambda \sigma_\theta^2(\alpha)\gtrsim\mu^{-1}$.
Consequently,
\begin{equation}
\label{eq:lower-layer-positive-core-short}
    \int_{J_\eta}
    B_\theta(\alpha)J_\theta(\alpha)\,d\alpha
    \gtrsim
    \dfrac1\mu.
\end{equation}
Now define
$E_+:=\{\alpha:B_\theta(\alpha)\ge0\}, \ E_-:=\{\alpha:B_\theta(\alpha)<0\}.$
If $\alpha\in E_-$, then the bias expansion and $1-\alpha\ge\rho$ imply
\begin{equation}
\label{eq:lower-layer-negative-tail-location-short}
    q_\theta(\alpha)\le\dfrac{C}{\mu},
    \qquad
    |B_\theta(\alpha)|\lesssim\dfrac1\mu.
\end{equation}
By \eqref{eq:lower-layer-logistic-comparison-short}, the set $\{q_\theta\le C/\mu\}$ is contained in $\{z\ge\log\mu-O(1)\}$, and therefore
\begin{equation}
\label{eq:lower-layer-tail-q-short}
    \int_{\{q_\theta\le C/\mu\}}q_\theta(\alpha)\,d\alpha
    \lesssim
    \dfrac1{\mu^2}.
\end{equation}
Combining \eqref{eq:lower-layer-basic-bounds-short}, \eqref{eq:lower-layer-negative-tail-location-short}, and \eqref{eq:lower-layer-tail-q-short} gives
\begin{equation}
\label{eq:lower-layer-negative-tail-short}
    \int_{E_-}
    |B_\theta(\alpha)|J_\theta(\alpha)\,d\alpha
    \lesssim
    \dfrac1{\mu^2},
\end{equation}
and
\begin{equation}
\label{eq:lower-layer-negative-radial-tail-short}
    \int_{E_-}
    |B_\theta(\alpha)|\sigma_\theta^2(\alpha)\,d\alpha
    \lesssim
    \dfrac1{\mu^3}.
\end{equation}
Since $J_\theta\ge0$, the nonnegative-bias region outside $J_\eta$ contributes nonnegatively to $\int B_\theta J_\theta$. Hence \eqref{eq:lower-layer-positive-core-short} and \eqref{eq:lower-layer-negative-tail-short} imply
$\int_\rho^\ell B_\theta(\alpha)J_\theta(\alpha)\,d\alpha\ge c/\mu-C/\mu^2\ge c/\mu$
after increasing $\mu_{\rm low,\eta}$. By \eqref{eq:K-ratio-identity}, $\dot k \ge \frac{c_{\rm low,\eta}}{\mu^3}.$ It remains to control $\dot\mu$ and the relative radial loss. Set
\begin{align*}
    A_+:=\int_{E_+}B_\theta(\alpha)\sigma_\theta^2(\alpha)\,d\alpha,
    \qquad
    A_-:=\int_{E_-}|B_\theta(\alpha)|\sigma_\theta^2(\alpha)\,d\alpha,\\
    K_+:=\int_{E_+}B_\theta(\alpha)J_\theta(\alpha)\,d\alpha,
    \qquad
    K_-:=\int_{E_-}|B_\theta(\alpha)|J_\theta(\alpha)\,d\alpha.
\end{align*}
The estimates above give $K_+\gtrsim\mu^{-1}$,
$K_-\lesssim\mu^{-2}$, and $K_+-K_-\ge K_+/2$ for sufficiently
large $\mu$. Moreover, on $E_+$,
$J_\theta\ge k\mu \sigma_\theta^2\ge(\rho\mu/2)\sigma_\theta^2$, so
$K_+\ge(\rho\mu/2)A_+$. Since
$-\dot\mu=(A_+-A_-)/\Delta$ and $\dot k=(K_+-K_-)/(\Delta\mu^2)$, we get
$(-\dot\mu)_+/\mu\le A_+/(\Delta\mu)\le
2K_+/(\Delta\rho\mu^2)\le(4/\rho)\dot k$.
After increasing $C_{\rm low,\eta}$, this proves \eqref{eq:lower-layer-relative-loss}.

Finally,
$\int_{J_\eta}B_\theta(\alpha)\sigma_\theta^2(\alpha)\,d\alpha
\gtrsim\mu^{-2}$,
while the negative radial tail is $O(\mu^{-3})$ by \eqref{eq:lower-layer-negative-radial-tail-short}. Hence $\int_\rho^\ell B_\theta(\alpha)\sigma_\theta^2(\alpha)\,d\alpha > 0$ for sufficiently large $\mu$, and therefore
$\dot\mu=-\Delta^{-1}\int_\rho^\ell
B_\theta(\alpha)\sigma_\theta^2(\alpha)\,d\alpha<0$.
The size bound follows from \eqref{eq:lower-layer-basic-bounds-short} and \eqref{eq:lower-layer-pq-integral-short}:
\begin{equation*}
    |\dot\mu|
    \lesssim
    \int_\rho^\ell
    |B_\theta(\alpha)|
    \left(
        \dfrac1{\mu^2}
        +
        p_\theta(\alpha)q_\theta(\alpha)
    \right)d\alpha
    \lesssim
    \dfrac1\mu.
\end{equation*}
This proves \eqref{eq:lower-layer-mu-control}.
\end{proof}
\begin{corollary}[Finite-time entrance into the interior strip]
\label{cor:lower-layer-to-interior}
Assume $0<\eta<\dfrac{\Delta}{4}.$ Let $t_0$ satisfy
$\mu(t_0)>\mu_{\rm low,\eta}, \ k(t_0)\in[L_{\rm band}(\mu(t_0)),\rho+\eta].$
Define
\begin{equation*}
    \tau_{\rm int}
    :=
    \inf\{t\ge t_0:k(t)\ge\rho+\eta\},
    \qquad
    \tau_{\rm low}
    :=
    \inf\{t\ge t_0:\mu(t)\le\mu_{\rm low,\eta}\}.
\end{equation*}
Then
\begin{equation}
\label{eq:lower-layer-stopped-time}
    \tau_{\rm int}\wedge\tau_{\rm low}-t_0
    \le
    \dfrac{\mu(t_0)^3}{c_{\rm low,\eta}}
    \bigl(\rho+\eta-k(t_0)\bigr).
\end{equation}
In particular, if $\tau_{\rm low}>\tau_{\rm int}$, then $\tau_{\rm int}<\infty$ and
\begin{equation}
\label{eq:lower-layer-to-interior-time-short}
    \tau_{\rm int}-t_0
    \le
    \dfrac{\mu(t_0)^3}{c_{\rm low,\eta}}
    \bigl(\rho+\eta-k(t_0)\bigr).
\end{equation}
\end{corollary}
\begin{proof}
If $k(t_0)=\rho+\eta$, then $\tau_{\rm int}=t_0$ and the claim is immediate. Assume $k(t_0)<\rho+\eta$. First we show that the trajectory cannot exit below $L_{\rm band}(\mu)$ before $\tau_{\rm int}\wedge\tau_{\rm low}$. Set $h(t):=k(t)-L_{\rm band}(\mu(t)).$

At any point with $h(t)=0, \ t<\tau_{\rm int}\wedge\tau_{\rm low},$ one has $k(t)=L_{\rm band}(\mu(t))\le\rho+\eta, \ \mu(t)>\mu_{\rm low,\eta}.$ Thus Proposition~\ref{prop:lower-boundary-layer-escape} gives $\dot k(t)\ge\frac{c_{\rm low,\eta}}{\mu(t)^3}, \ \dot\mu(t)<0.$ Since $L_{\rm band}'(\mu)>0$ for $\mu>\mu_{\rm low,\eta}>e$, it follows that $\dot h(t) = \dot k(t)-L_{\rm band}'(\mu(t))\dot\mu(t) > 0.$ Hence $h$ cannot cross from nonnegative to negative before $\tau_{\rm int}\wedge\tau_{\rm low}$. Therefore $k(t)\in[L_{\rm band}(\mu(t)),\rho+\eta)$ for every $t<\tau_{\rm int}\wedge\tau_{\rm low}$.

Now set $g_{\rm low}(t):=\rho+\eta-k(t).$ On $[t_0,\tau_{\rm int}\wedge\tau_{\rm low})$, Proposition~\ref{prop:lower-boundary-layer-escape} gives
$\dot g_{\rm low}(t)=-\dot k(t)\le-c_{\rm low,\eta}/\mu(t)^3$
and $\dot\mu(t)<0$.
Therefore $\mu(t)\le\mu(t_0)$ on this interval. Lemma~\ref{lem:abstract-finite-entry}\textnormal{(i)} gives
$\tau_{\rm int}\wedge\tau_{\rm low}-t_0
\le\mu(t_0)^3g_{\rm low}(t_0)/c_{\rm low,\eta}$.
Since $g_{\rm low}(t_0)=\rho+\eta-k(t_0)$, this proves \eqref{eq:lower-layer-stopped-time}. If $\tau_{\rm low}>\tau_{\rm int}$, then the stopped time equals $\tau_{\rm int}$, giving \eqref{eq:lower-layer-to-interior-time-short}.
\end{proof}
The large-parameter preservation estimate later rules out $\tau_{\rm low}\le\tau_{\rm int}$ under the initialization assumptions of the global dynamics theorem.
\textbf{Interior Laplace asymptotics and the slingshot mechanism}
\label{subsubsec:interior-slingshot}
Once the trajectory reaches the fixed interior strip $\mathcal I_\eta$, the transition point $\alpha_c=1-k$ remains uniformly separated from the endpoints of the task interval. The gradient then admits a uniform Laplace expansion, which shows that $k$ and $\lambda$ increase while $\mu$ decreases. \textbf{Uniform interior Laplace expansion.}
\begin{lemma}[Uniform interior Laplace expansion]
\label{lem:uniform-interior-laplace}
Fix $0<\eta<\dfrac{\Delta}{4}.$ Uniformly for $k\in[\rho+\eta,\ell-\eta],$ as $\mu\to\infty$,
\begin{equation}
\label{eq:interior-laplace-mu-lambda}
    \dot\mu
    =
    -\dfrac{k^3}{2\Delta\mu}
    +
    O_\eta(\mu^{-2}),
    \qquad
    \dot\lambda
    =
    \dfrac{k^2}{2\Delta\mu}
    +
    O_\eta(\mu^{-2}).
\end{equation}
Consequently,
\begin{equation}
\label{eq:interior-laplace-k}
    \dot k
    =
    \dfrac{k^2+k^4}{2\Delta\mu^2}
    +
    O_\eta(\mu^{-3}).
\end{equation}
\end{lemma}
\begin{proof}
Set $\alpha_c:=1-k$. For $k\in[\rho+\eta,\ell-\eta]$, one has $\alpha_c\in[\rho+\eta,\ell-\eta]$, so $\alpha_c$ remains at least distance $\eta$ from both endpoints. Introduce $z:=\mu(\alpha-\alpha_c), \ \alpha=\alpha_c+\frac{z}{\mu},$ and define $Q(z):=\frac1{1+e^z}, \ \mathfrak l(z):=\frac{e^z}{(1+e^z)^2}.$ The range of $z$ is $I_{\mu,k} := [\mu(\rho-\alpha_c),\mu(\ell-\alpha_c)],$ with $[-\eta\mu,\eta\mu]\subseteq I_{\mu,k}$. Consequently, for every fixed $m\ge0$,
\begin{equation}
\label{eq:interior-logistic-tail}
    \sup_{k\in[\rho+\eta,\ell-\eta]}
    \int_{\R\setminus I_{\mu,k}}
    (1+|z|)^m\mathfrak l(z)\,dz
    =
    O_{\eta,m}(e^{-c_\eta\mu})
\end{equation}
for some $c_\eta>0$.

By \eqref{eq:mixture-odds},
\begin{equation*}
    \dfrac{p_\theta(\alpha)}{q_\theta(\alpha)}
    =
    e^zh_\mu(\alpha),
    \qquad
    h_\mu(\alpha)
    :=
    \dfrac{1-e^{-\mu\alpha}}
    {1-e^{-\mu(1-\alpha)}}.
\end{equation*}
Uniformly for $\alpha\in[\rho,\ell]$,
\begin{align*}
    h_\mu(\alpha)
    &=
    1+O(e^{-\rho\mu})\\
    q_\theta(\alpha)
    &=
    Q(z)+O(e^{-\rho\mu}\mathfrak l(z))\\
    p_\theta(\alpha)q_\theta(\alpha)
    &=
    \mathfrak l(z)+O(e^{-\rho\mu}\mathfrak l(z)).
\end{align*}
By Lemma~\ref{lem:mixture-representation},
${d_\mu}(\alpha) = 1-\alpha+O(e^{-\rho\mu}) = k-\dfrac{z}{\mu}+O(e^{-\rho\mu}).$
Corollary~\ref{cor:bias-expansion} gives
\begin{align*}
    B_\theta(\alpha)
    &=
    q_\theta(\alpha)(1-\alpha)-\dfrac1\mu+O(e^{-\rho\mu}) =
    kQ(z)
    +
    O_\eta\left(\dfrac{1+|z|}{\mu}\right)
    +
    O(e^{-\rho\mu}).
\end{align*}
The error is integrable after multiplication by $\mathfrak l(z)$, since $\int_{\R}(1+|z|)^m\mathfrak l(z)\,dz<\infty$ for every fixed $m\ge0$.

Using $\Cov_\theta(X,Y\mid\alpha)=-p_\theta(\alpha)q_\theta(\alpha){d_\mu}(\alpha)$, we obtain
\begin{align*}
    B_\theta(\alpha)\Cov_\theta(X,Y\mid\alpha)
    &=
    -k^2Q(z)\mathfrak l(z)
    +
    O_\eta\left(
        \dfrac{1+|z|}{\mu}\mathfrak l(z)
    \right)
    +
    O_\eta(e^{-\rho\mu}\mathfrak l(z)).
\end{align*}
Using $d\alpha=dz/\mu$ and \eqref{eq:interior-logistic-tail} to replace $I_{\mu,k}$ by $\R$, we find
\begin{equation}
\label{eq:interior-cov-integral}
    \int_\rho^\ell
    B_\theta(\alpha)\Cov_\theta(X,Y\mid\alpha)\,d\alpha
    =
    -\dfrac{k^2}{\mu}
    \int_{\R}Q(z)\mathfrak l(z)\,dz
    +
    O_\eta(\mu^{-2}).
\end{equation}
For the variance, the component-variance contribution satisfies
\begin{equation*}
    \int_\rho^\ell
    |B_\theta(\alpha)|
    \left(
        p_\theta(\alpha)s_{1,\mu}^2(\alpha)
        +
        q_\theta(\alpha)s_{2,\mu}^2(\alpha)
    \right)d\alpha
    =
    O_\eta(\mu^{-2}),
\end{equation*}
because the bias is uniformly bounded and $s_{1,\mu}^2(\alpha)+s_{2,\mu}^2(\alpha)=O(\mu^{-2})$ uniformly in $\alpha$. The mixture contribution satisfies
\begin{align*}
    B_\theta(\alpha)
    p_\theta(\alpha)q_\theta(\alpha)
    {d_\mu}(\alpha)^2
    &=
    k^3Q(z)\mathfrak l(z)
    +
    O_\eta\left(
        \dfrac{1+|z|}{\mu}\mathfrak l(z)
    \right)
    +
    O_\eta(e^{-\rho\mu}\mathfrak l(z)).
\end{align*}
Consequently,
\begin{equation}
\label{eq:interior-var-integral}
    \int_\rho^\ell
    B_\theta(\alpha)\sigma_\theta^2(\alpha)\,d\alpha
    =
    \dfrac{k^3}{\mu}
    \int_{\R}Q(z)\mathfrak l(z)\,dz
    +
    O_\eta(\mu^{-2}).
\end{equation}
With $r:=\dfrac{e^z}{1+e^z},\;dr=\mathfrak l(z)\,dz,$ one has $Q(z)=1-r$, and therefore $\int_{\R}Q(z)\mathfrak l(z)\,dz = \int_0^1(1-r)\,dr = \frac12.$ Substituting into \eqref{eq:interior-cov-integral} and \eqref{eq:interior-var-integral}, and using Lemma~\ref{lem:main-gradient-identities}, proves \eqref{eq:interior-laplace-mu-lambda}. Finally, $\dot k = \dfrac{\dot\lambda-k\dot\mu}{\mu},$ which gives \eqref{eq:interior-laplace-k}.
\end{proof}
\textbf{Quantitative interior drift.}
\begin{proposition}[Interior slingshot]
\label{prop:interior-slingshot}
Fix $0<\eta<\frac{\Delta}{4}$. There exist constants $c_{\rm int,\eta}>0$ and $\mu_{\rm int,\eta}>0$ such that, whenever $\mu\ge\mu_{\rm int,\eta}, \ k\in[\rho+\eta,\ell-\eta],$ one has $\dot k \ge \dfrac{c_{\rm int,\eta}}{\mu^2} > 0,$ together with $\dot\mu<0,$ and $\dot\lambda>0$. Moreover, for a possibly larger constant $C_{\rm int,\eta}>0$,
\begin{equation}
\label{eq:interior-mu-lambda-size}
    |\dot\mu|
    \le
    \dfrac{C_{\rm int,\eta}}{\mu},
    \qquad
    |\dot\lambda|
    \le
    \dfrac{C_{\rm int,\eta}}{\mu}.
\end{equation}
After increasing $C_{\rm int,\eta}$ if necessary, one also has
\begin{equation}
\label{eq:interior-relative-loss}
    \dfrac{(-\dot\mu)_+}{\mu}
    \le
    C_{\rm int,\eta}\dot k.
\end{equation}
\end{proposition}
\begin{proof}
Since $\rho+\eta\le k\le\ell-\eta$, the coefficients ${k^3}/{2\Delta},\ {k^2}/{2\Delta},\ {(k^2+k^4)}/{2\Delta}$ are uniformly bounded below by positive constants. The uniform remainder estimates in Lemma~\ref{lem:uniform-interior-laplace} can therefore be absorbed into the leading terms by increasing $\mu_{\rm int,\eta}$. This gives
$\dot\mu<0, \ \dot\lambda>0, \ \dot k\ge\dfrac{c_{\rm int,\eta}}{\mu^2}.$
The size bounds \eqref{eq:interior-mu-lambda-size} follow directly from \eqref{eq:interior-laplace-mu-lambda}, since $k$ is uniformly bounded on the strip. Therefore $\frac{(-\dot\mu)_+}{\mu} \le \frac{C}{\mu^2} \le C_{\rm int,\eta}\dot k$ for a possibly larger $C_{\rm int,\eta}$, proving \eqref{eq:interior-relative-loss}.
\end{proof}
\textbf{Finite-time traversal of the interior strip.}
\begin{corollary}[Finite-time traversal of the interior strip]
\label{cor:interior-passage}
Fix $0<\eta<\dfrac{\Delta}{4}.$ Suppose that, at some time $t_0$, $k(t_0)\in[\rho+\eta,\ell-\eta]$ and $\mu(t_0)>\mu_{\rm int,\eta}.$ Define
\begin{equation*}
    \tau_{\rm up}
    :=
    \inf\{t\ge t_0:k(t)\ge\ell-\eta\},
    \qquad
    \tau_{\rm int,exit}
    :=
    \inf\{t\ge t_0:\mu(t)\le\mu_{\rm int,\eta}\}.
\end{equation*}
Then
\begin{equation}
\label{eq:interior-stopped-time}
    \tau_{\rm up}\wedge\tau_{\rm int,exit}-t_0
    \le
    \dfrac{
        \mu(t_0)^2
        \{\ell-\eta-k(t_0)\}
    }{
        c_{\rm int,\eta}
    }.
\end{equation}
In particular, if $\tau_{\rm int,exit}>\tau_{\rm up}$, then $\tau_{\rm up}<\infty$ and
\begin{equation}
\label{eq:interior-time-short}
    \tau_{\rm up}-t_0
    \le
    \dfrac{
        \mu(t_0)^2
        \{\ell-\eta-k(t_0)\}
    }{
        c_{\rm int,\eta}
    }.
\end{equation}
\end{corollary}
\begin{proof}
If $k(t_0)=\ell-\eta$, the claim is immediate. Otherwise, the lower boundary $k=\rho+\eta$ is inward-pointing by Proposition~\ref{prop:interior-slingshot}; hence the trajectory remains in the interior strip until $\tau_{\rm up}\wedge\tau_{\rm int,exit}$.

On this stopped interval, with $g_{\rm int}(t):=\ell-\eta-k(t)$,
$\dot g_{\rm int}(t) \le -\dfrac{c_{\rm int,\eta}}{\mu(t)^2}, \ \dot\mu(t)<0.$
Lemma~\ref{lem:abstract-finite-entry}\textnormal{(i)} gives \eqref{eq:interior-stopped-time}. If $\tau_{\rm int,exit}>\tau_{\rm up}$, the stopped time equals $\tau_{\rm up}$, proving \eqref{eq:interior-time-short}.
\end{proof}
\textbf{Extension toward the upper ratio rail.} The fixed-strip expansion is sufficient for finite-time traversal to $k=\ell-\eta$. The following extension remains valid when the transition point approaches the endpoint, provided the separation is at least logarithmic.
\begin{lemma}[Interior expansion with logarithmic endpoint separation]
\label{lem:interior-log-separation}
Fix $\eta_1\in(0,\Delta)$ and $\beta>0$. Uniformly for
$\ell-\eta_1 \le k \le \ell-\dfrac{\beta\log\mu}{\mu},$
as $\mu\to\infty$,
\begin{equation}
\label{eq:interior-log-separation-mu-lambda}
    \dot\mu
    =
    -\dfrac{k^3}{2\Delta\mu}
    +
    O_{\eta_1,\beta}
    \left(
        \mu^{-2}+\mu^{-\beta-1}
    \right),
\end{equation}
\begin{equation}
\label{eq:interior-log-separation-lambda}
    \dot\lambda
    =
    \dfrac{k^2}{2\Delta\mu}
    +
    O_{\eta_1,\beta}
    \left(
        \mu^{-2}+\mu^{-\beta-1}
    \right),
\end{equation}
and
\begin{equation}
\label{eq:interior-log-separation-k}
    \dot k
    =
    \dfrac{k^2+k^4}{2\Delta\mu^2}
    +
    O_{\eta_1,\beta}
    \left(
        \mu^{-3}+\mu^{-\beta-2}
    \right).
\end{equation}
In particular, there exist constants $c_{\eta_1}>0$ and $\mu_{\rm sep}(\eta_1,\beta)<\infty$ such that $\dot k \ge \dfrac{c_{\eta_1}}{\mu^2}$ and $\dot\mu<0,$ throughout this region whenever $\mu\ge\mu_{\rm sep}(\eta_1,\beta)$.
\end{lemma}
\begin{proof}
For sufficiently large $\mu$, the interval in the statement is nonempty because $\beta\log\mu/\mu<\eta_1$. The proof follows the rescaling argument of Lemma~\ref{lem:uniform-interior-laplace}. The only difference is that the transition point $\alpha_c=1-k$ may now approach the lower task endpoint $\rho$ at logarithmic distance.

Set $\alpha_c:=1-k $ and $z:=\mu(\alpha-\alpha_c).$ The rescaled task
interval is $I_{\mu,k}=[a_{\mu,k},b_{\mu,k}]$, where
$a_{\mu,k}:=\mu(\rho-\alpha_c)=\mu(k-\ell)$ and
$b_{\mu,k}:=\mu(\ell-\alpha_c)=\mu(k-\rho)$.
The upper bound on $k$ gives $a_{\mu,k} \le -\beta\log\mu,$ while the lower bound on $k$ gives $b_{\mu,k} \ge (\Delta-\eta_1)\mu.$ Recall $Q(z):=\frac1{1+e^z}, \ \mathfrak l(z):=\frac{e^z}{(1+e^z)^2}.$ The omitted left tail satisfies
\begin{align*}
    \int_{-\infty}^{a_{\mu,k}}\mathfrak l(z)\,dz
    &\le
    \int_{-\infty}^{-\beta\log\mu}\mathfrak l(z)\,dz
    =
    \dfrac{\mu^{-\beta}}{1+\mu^{-\beta}}
    =
    O(\mu^{-\beta}),
\end{align*}
whereas the omitted right tail satisfies
\begin{align*}
    \int_{b_{\mu,k}}^\infty \mathfrak l(z)\,dz
    &\le
    \int_{(\Delta-\eta_1)\mu}^\infty \mathfrak l(z)\,dz
    =
    \dfrac1{1+e^{(\Delta-\eta_1)\mu}}
    =
    O_{\eta_1}(e^{-(\Delta-\eta_1)\mu}).
\end{align*}
Since $0\le Q\le1$, the same estimates control the omitted tails of $Q(z)\mathfrak l(z)$. After the change of variables $d\alpha=dz/\mu$, the left endpoint truncation contributes $O(\mu^{-\beta-1})$ to the expansions of $\dot\lambda$ and $\dot\mu$, while the right endpoint contribution is exponentially small.

The local expansion errors remain $O(\mu^{-2})$ uniformly in the stated region. Indeed, the argument of Lemma~\ref{lem:uniform-interior-laplace} gives an integrated error bounded by
\begin{align*}
    \dfrac1\mu
    \int_{I_{\mu,k}}
    C_{\eta_1}
    \dfrac{1+|z|}{\mu}
    \mathfrak l(z)\,dz
    &\le
    \dfrac{C_{\eta_1}}{\mu^2}
    \int_{\R}
    (1+|z|)\mathfrak l(z)\,dz
    =
    O_{\eta_1}(\mu^{-2}).
\end{align*}
Combining the local expansion error with the endpoint truncation error gives \eqref{eq:interior-log-separation-mu-lambda} and \eqref{eq:interior-log-separation-lambda}. Finally, $\dot k = \frac{\dot\lambda-k\dot\mu}{\mu}.$ Since $k$ remains uniformly bounded in the stated region, substitution gives \eqref{eq:interior-log-separation-k}. Moreover, $k\ge\ell-\eta_1>\rho>0$, so the leading coefficients in $\dot\mu$ and $\dot k$ are uniformly bounded away from zero. Since
$\mu^{-2}+\mu^{-\beta-1}=o(\mu^{-1}), \ \mu^{-3}+\mu^{-\beta-2}=o(\mu^{-2}),$
the remainders can be absorbed into the leading terms after increasing $\mu_{\rm sep}(\eta_1,\beta)$. Hence $\dot\mu<0, \ \dot k \ge \frac{c_{\eta_1}}{\mu^2}.$
\end{proof}
\textbf{Upper rail-tube capture and invariance}
\label{subsubsec:upper-rail-tube}
Recall $\mathcal T_\eta(\mu) := [\ell-\eta,U_{\rm band}(\mu)].$ Its fixed lower boundary is controlled by the interior slingshot, whereas its moving upper boundary is controlled by the outer ratio correction.
\begin{proposition}[Forward invariance of the upper rail tube]
\label{prop:upper-rail-tube}
Assume $C_{\rm band}\ge2$, and fix $0<\eta<\Delta/4.$ There exists $\mu_{\rm tube,\eta}>0$ such that, if
$k(t_0)\in\mathcal T_\eta(\mu(t_0)), \ \mu(t_0)>\mu_{\rm tube,\eta},$
and
$\tau_{\rm tube} := \inf\{t\ge t_0:\mu(t)\le\mu_{\rm tube,\eta}\},$
then
$k(t)\in\mathcal T_\eta(\mu(t))$
for every $t\in[t_0,\tau_{\rm tube})$.
\end{proposition}
\begin{proof}
Choose $\mu_{\rm tube,\eta}$ larger than $\mu_{\rm int,\eta}$ and the threshold in Corollary~\ref{cor:upper-moving-boundary-correction}. Set $a(\mu):=\ell-\eta,\  b(\mu):=U_{\rm band}(\mu).$ At the lower boundary $k=a(\mu)=\ell-\eta$, Proposition~\ref{prop:interior-slingshot} gives $\dot k-a'(\mu)\dot\mu=\dot k>0.$ At the upper moving boundary $k=b(\mu)=U_{\rm band}(\mu)$, Corollary~\ref{cor:upper-moving-boundary-correction} gives $\dot k-b'(\mu)\dot\mu = \dot k-U_{\rm band}'(\mu)\dot\mu <0.$ Lemma~\ref{lem:moving-interval-invariance} therefore gives $k(t)\in\mathcal T_\eta(\mu(t))$ for every $t\in[t_0,\tau_{\rm tube})$.
\end{proof}
\subsubsection{Logarithmic endpoint dynamics}
\label{subsec:logarithmic-prediction-good-tube}
The upper rail tube provides dynamical trapping. A uniform $O(\mu^{-1})$ prediction bound begins after the ratio crosses a logarithmic rail slightly above $\ell$: $k\ge b_c(\mu) = \ell+\frac{\log\mu+c}{\mu}.$ This subsection proves the endpoint estimates underlying the logarithmic tube introduced in Section~\ref{sec:population-dynamics}. In particular, it proves the logarithmic-rail expansions for the functions $d_k$ and $d_{\rm rad}$ defined in \eqref{eq:ratio-drift-profile}, verifies the choice of $c_0$, proves finite-time approach to the prediction rail, establishes forward invariance of the prediction-good tube, proves radial divergence after capture, and verifies preservation of the large-parameter regime during the transient stages.

\textbf{Prediction and endpoint rescaling}
\begin{lemma}[Uniform prediction bound above a logarithmic rail]
\label{lem:log-prediction-bound}
Fix $c\in\R$. There exist constants $C_c,\mu_c>0$ such that, whenever $\mu\ge\mu_c, k\ge b_c(\mu),$ one has
\begin{equation}
\label{eq:log-pred-bound}
    \sup_{\alpha\in[\rho,\ell]}
    |F(\lambda,\mu;\alpha)-\alpha|
    \le
    \dfrac{C_c}{\mu}.
\end{equation}
\end{lemma}
\begin{proof}
By Lemma~\ref{lem:mixture-representation},
$p_\theta(\alpha)/q_\theta(\alpha)=e^{\mu(k+\alpha-1)}
(1-e^{-\mu\alpha})/(1-e^{-\mu(1-\alpha)})$.
Since $\alpha,1-\alpha\ge\rho$, the second factor is bounded below by $1/2$ for all sufficiently large $\mu$, uniformly in $\alpha\in[\rho,\ell]$. If $k\ge b_c(\mu)$, then $\mu(k+\alpha-1) \ge \mu(k+\rho-1) \ge \log\mu+c.$ Hence
$p_\theta(\alpha)/q_\theta(\alpha)\ge e^c\mu/2$ and
$q_\theta(\alpha)\le C_c/\mu$.
The uniform bias expansion $B_\theta(\alpha) = q_\theta(\alpha)(1-\alpha)-\frac1\mu+O(e^{-\rho\mu})$ therefore gives \eqref{eq:log-pred-bound}.
\end{proof}
\begin{lemma}[Right-mass control near the lower task endpoint]
\label{lem:q-control-log-rail}
Fix $c_{\min}\in\R$. Suppose
$c\ge c_{\min}, \ k=b_c(\mu).$
For
$\alpha=\rho+\dfrac{u}{\mu}, \ 0\le u\le \Delta\mu,$
one has, for every fixed $U<\infty$,
\begin{equation*}
    \sup_{c\ge c_{\min}}
    \sup_{0\le u\le U}
    \left|
        \mu q_\theta\left(\rho+\dfrac{u}{\mu}\right)
        -
        e^{-c-u}
    \right|
    \longrightarrow0
\end{equation*}
as $\mu\to\infty$. Moreover, there exists $C=C(c_{\min})$ such that
\begin{equation}
\label{eq:mu-q-domination}
    \mu q_\theta\left(\rho+\dfrac{u}{\mu}\right)
    \le
    Ce^{-u}
\end{equation}
for all sufficiently large $\mu$, all $u\in[0,\Delta\mu]$, and all $c\ge c_{\min}$.
\end{lemma}
\begin{proof}
Write $h_\mu(\alpha) := \frac{1-e^{-\mu\alpha}} {1-e^{-\mu(1-\alpha)}}.$ Uniformly for $\alpha\in[\rho,\ell]$, $h_\mu(\alpha) = 1+O(e^{-\rho\mu}).$ If $k=b_c(\mu)$ and $\alpha=\rho+u/\mu$, then $\mu(k+\alpha-1) = \log\mu+c+u.$ The odds identity gives
\begin{equation*}
    \mu q_\theta\left(\rho+\dfrac{u}{\mu}\right)
    =
    \dfrac{\mu}
    {1+\mu e^{c+u}h_\mu(\rho+u/\mu)}
    =
    \dfrac{e^{-c-u}}
    {h_\mu(\rho+u/\mu)+\mu^{-1}e^{-c-u}}.
\end{equation*}
The first conclusion follows uniformly for $c\ge c_{\min}$ and bounded $u$. Since $h_\mu\ge1/2$ for all sufficiently large $\mu$,
$\mu q_\theta(\rho+u/\mu)\le2e^{-c-u}\le2e^{-c_{\min}}e^{-u}$,
which proves \eqref{eq:mu-q-domination}.
\end{proof}
\textbf{Limiting dynamics on logarithmic rails}
For the endpoint analysis, define $s_{{\rm comp},\mu}^2(\alpha) := p_\theta(\alpha)s_{1,\mu}^2(\alpha) + q_\theta(\alpha)s_{2,\mu}^2(\alpha).$  Then
\begin{equation*}
    \sigma_\theta^2(\alpha)
    =
    s_{{\rm comp},\mu}^2(\alpha)
    +
    p_\theta(\alpha)q_\theta(\alpha){d_\mu}(\alpha)^2.
\end{equation*}
Define
\begin{align*}
 H_{\rm mix}(\alpha)&:=-\mu p_\theta(\alpha)q_\theta(\alpha)d_\mu(\alpha)
 \{1+kd_\mu(\alpha)\},
 &H_{\rm comp}(\alpha)&:=-k\mu s_{{\rm comp},\mu}^2(\alpha),\\
 S_{\rm mix}&:=\frac1\Delta\int_\rho^\ell B_\theta(\alpha)H_{\rm mix}(\alpha)\,d\alpha,
 &S_{\rm comp}&:=\frac1\Delta\int_\rho^\ell B_\theta(\alpha)H_{\rm comp}(\alpha)\,d\alpha,\\
 \dot\mu_{\rm mix}&:=-\frac1\Delta\int_\rho^\ell B_\theta(\alpha)p_\theta(\alpha)q_\theta(\alpha)d_\mu(\alpha)^2\,d\alpha,
 &\dot\mu_{\rm comp}&:=-\frac1\Delta\int_\rho^\ell B_\theta(\alpha)s_{{\rm comp},\mu}^2(\alpha)\,d\alpha.
\end{align*}
Thus,
$S=S_{\rm mix}+S_{\rm comp}, \ \dot\mu=\dot\mu_{\rm mix}+\dot\mu_{\rm comp}.$
\begin{proposition}[Uniform logarithmic-rail expansions]
\label{prop:log-rail-expansion}
Let $I=[c_-,c_+]\subset\R$ be compact. Uniformly for $c\in I$ and $k=b_c(\mu)$,
\begin{equation}
\label{eq:uniform-compact-S}
    \sup_{c\in I}
    \left|
        \mu^2S(\mu b_c(\mu),\mu)-d_k(c)
    \right|
    \longrightarrow0,
\end{equation}
\begin{equation}
\label{eq:uniform-compact-mudot}
    \sup_{c\in I}
    \left|
        \mu^3\dot\mu-d_{\rm rad}(c)
    \right|
    \longrightarrow0,
\end{equation}
and
\begin{equation}
\label{eq:uniform-compact-kdot}
    \sup_{c\in I}
    \left|
        \mu^4\dot k+d_k(c)
    \right|
    \longrightarrow0.
\end{equation}
In particular, for each fixed $c\in\R$,
\begin{equation*}
    S(\lambda,\mu)
    =
    \dfrac{d_k(c)}{\mu^2}
    +
    o_c(\mu^{-2}),
    \qquad
    \dot\mu
    =
    \dfrac{d_{\rm rad}(c)}{\mu^3}
    +
    o_c(\mu^{-3}),
\end{equation*}
and
$\dot k = -\dfrac{d_k(c)}{\mu^4} + o_c(\mu^{-4}).$
\end{proposition}
\begin{proof}
Let $\alpha=\rho+u/\mu$, $k=b_c(\mu)$, and
$q_{\mu,c}(u):=q_\theta(\rho+u/\mu)$.
Lemma~\ref{lem:q-control-log-rail}, applied with $c_{\min}=c_-$, gives
$\mu q_{\mu,c}(u)\to e^{-c-u}$ locally uniformly in
$(c,u)\in I\times[0,\infty)$ and
$\mu q_{\mu,c}(u)\le C_Ie^{-u}$; squaring the latter bound gives
$\mu^2q_{\mu,c}(u)^2\le C_Ie^{-2u}$. Moreover,
$p_\theta\to1$, $1-\alpha\to\ell$, ${d_\mu}(\alpha)\to\ell$, and
$k\to\ell$
uniformly on bounded $u$-intervals, and the bias expansion gives $\mu B_\theta\left(\rho+\frac{u}{\mu}\right) \to \ell e^{-c-u}-1.$ The rescaled mixture integrands for $S_{\rm mix}$ are dominated by a linear combination of $e^{-u}$ and $e^{-2u}$. Uniform dominated convergence yields
\begin{align*}
    \sup_{c\in I}
    \left|
        \mu^2S_{\rm mix}
        -
        \left[
            \dfrac{\ell(1+\ell^2)}{\Delta}e^{-c}
            -
            \dfrac{\ell^2(1+\ell^2)}{2\Delta}e^{-2c}
        \right]
    \right|
    &\longrightarrow0.
\end{align*}
Similarly,
\begin{align*}
    \sup_{c\in I}
    \left|
        \mu^3\dot\mu_{\rm mix}
        -
        \left[
            \dfrac{\ell^2}{\Delta}e^{-c}
            -
            \dfrac{\ell^3}{2\Delta}e^{-2c}
        \right]
    \right|
    &\longrightarrow0.
\end{align*}
Uniformly for $c\in I$, Lemma~\ref{lem:log-prediction-bound} gives $|B_\theta(\alpha)|\le C_I/\mu$, while the component asymptotics give $s_{{\rm comp},\mu}^2(\alpha)=\mu^{-2}+o(\mu^{-2})$ uniformly in $\alpha$. Moreover, the domination above gives $\int_\rho^\ell q_\theta(\alpha)\,d\alpha=O_I(\mu^{-2})$: indeed, with $\alpha=\rho+u/\mu$, one has $q_\theta(\alpha)\le C_Ie^{-u}/\mu$ and $d\alpha=du/\mu$. Hence Corollary~\ref{cor:bias-expansion} yields $\int_\rho^\ell B_\theta(\alpha)\,d\alpha=-\Delta/\mu+O_I(\mu^{-2})$. Therefore $\sup_{c\in I}|\mu^2S_{\rm comp}-\ell|\to0$ and $\sup_{c\in I}|\mu^3\dot\mu_{\rm comp}-1|\to0$. Adding the mixture and component contributions and using the definitions of $d_k$ and $d_{\rm rad}$ proves \eqref{eq:uniform-compact-S} and \eqref{eq:uniform-compact-mudot}. Finally, $\dot k=-S/\mu^2$, so \eqref{eq:uniform-compact-kdot} follows. The fixed-$c$ asymptotics are the special case $I=\{c\}$.
\end{proof}
\textbf{Selection of the prediction rail}
\begin{lemma}[Choice of a prediction rail]
\label{lem:choose-c0}
There exists $c_0\in\R$ satisfying $d_k(c_0)<0, \ d_{\rm rad}(c_0)>0.$
Moreover, $c_0$ can be chosen so that
\begin{equation*}
    d_k(c)<0
    \quad
    \text{for every }c\le c_0,
    \qquad \rm and \qquad
    d_{\rm rad}(c)>0
    \quad
    \text{for every }c\ge c_0.
\end{equation*}
\end{lemma}
\begin{proof}
Set $x=e^{-c}$. Using \eqref{eq:ratio-drift-profile}, the functions $d_k$ and $d_{\rm rad}$ are downward-opening quadratics in $x$:
\begin{equation*}
    d_k
    =
    \ell
    +
    \dfrac{\ell(1+\ell^2)}{\Delta}x
    -
    \dfrac{\ell^2(1+\ell^2)}{2\Delta}x^2
    \quad\text{and}\quad
    d_{\rm rad}
    =
    1
    +
    \dfrac{\ell^2}{\Delta}x
    -
    \dfrac{\ell^3}{2\Delta}x^2.
\end{equation*}
Their positive roots are
\begin{equation*}
    x_k^\star
    =
    \dfrac1\ell
    \left(
        1+
        \sqrt{
            1+\dfrac{2\ell \Delta}{1+\ell^2}
        }
    \right)
    \quad\text{and}\quad
    x_\mu^\star
    =
    \dfrac1\ell
    \left(
        1+
        \sqrt{
            1+\dfrac{2\Delta}{\ell}
        }
    \right).
\end{equation*}
Since $0<\ell<1$, ${2\ell \Delta}/(1+\ell^2) < {2\Delta}/{\ell},$ so $x_k^\star<x_\mu^\star$. Choose $x_-\in(x_k^\star,x_\mu^\star), \ c_0:=-\log x_-.$ Then $d_k(c_0)<0<d_{\rm rad}(c_0)$. Since $c\le c_0$ is equivalent to $e^{-c}\ge x_-$, the downward-opening structure gives $d_k(c)<0$ for every $c\le c_0$. Similarly, $c\ge c_0$ implies $e^{-c}\le x_-<x_\mu^\star$, and hence $d_{\rm rad}(c)>0$ for every $c\ge c_0$.
\end{proof}
\textbf{Approach to and invariance of the logarithmic tube}
\begin{proposition}[Forward invariance of the prediction rail]
\label{prop:log-lower-rail-invariance}
There exist constants $a>0$ and $\mu_{\rm log}>0$ such that, whenever
$\mu\ge\mu_{\rm log}, \ k=b_{c_0}(\mu),$
one has
\begin{equation}
\label{eq:log-lower-rail-inward}
    \dot k-b_{c_0}'(\mu)\dot\mu
    \ge
    \dfrac{a}{\mu^4}.
\end{equation}
Consequently, the region $k\ge b_{c_0}(\mu)$ is forward invariant from below while $\mu(t)\ge\mu_{\rm log}$.
\end{proposition}
\begin{proof}
By Lemma~\ref{lem:choose-c0}, choose $a>0$ such that $d_k(c_0)\le-4a$. Proposition~\ref{prop:log-rail-expansion} gives, on the rail $k=b_{c_0}(\mu)$,
\begin{equation*}
    \dot k
    =
    -\dfrac{d_k(c_0)}{\mu^4}
    +
    o(\mu^{-4})
    \ge
    \dfrac{2a}{\mu^4}
    \qquad\text{and}\qquad
    |\dot\mu|\le\dfrac{C}{\mu^3}
\end{equation*}
for all sufficiently large $\mu$. Since $b_{c_0}'(\mu)=\dfrac{1-\log\mu-c_0}{\mu^2}$, $|b_{c_0}'(\mu)\dot\mu| \lesssim \frac{\log\mu}{\mu^5} = o(\mu^{-4}).$ After increasing $\mu_{\rm log}$, \eqref{eq:log-lower-rail-inward} follows.
\end{proof}
\begin{lemma}[Endpoint shoulder estimates]
\label{lem:endpoint-shoulder-estimates}
Fix $\beta>0$. There exist constants $c_{\rm sh,L}<c_0$ and $a_{\rm sh},C_{\rm sh},\mu_{\rm sh}>0,$ depending only on $\rho$, $C_{\rm band}$, and $\beta$, such that, whenever $\mu\ge\mu_{\rm sh},\; \ell-\dfrac{\beta\log\mu}{\mu} \le k<b_{c_{\rm sh,L}}(\mu),$ one has
\begin{equation*}
    \dot k-b_{c_0}'(\mu)\dot\mu
    \ge
    \dfrac{a_{\rm sh}}{\mu^4},
    \qquad
    (\dot\mu)_+
    \le
    \dfrac{C_{\rm sh}}{\mu^3},
\end{equation*}
and
$\dot k>0, \ \frac{(-\dot\mu)_+}{\mu} \le C_{\rm sh}\dot k.$
\end{lemma}
\begin{proof}
All implicit constants may depend on $\rho,C_{\rm band},\beta$. Set $s:=\mu(k-\ell), \ \chi:=s-\log\mu.$ The assumptions imply $-(\beta+1)\log\mu\le\chi<c_{\rm sh,L}.$ For $\alpha=\rho+u/\mu$, $0\le u\le \Delta\mu$, write
$q(u):=q_\theta(\rho+u/\mu)$, $p(u):=1-q(u)$, and $r:=q(0)$.
The odds identity gives
\begin{equation}
\label{eq:shoulder-odds}
    \dfrac{p(u)}{q(u)}
    =
    e^{s+u}
    h_\mu\left(\rho+\dfrac{u}{\mu}\right),
    \qquad
    h_\mu(\alpha)
    :=
    \dfrac{1-e^{-\mu\alpha}}
    {1-e^{-\mu(1-\alpha)}}.
\end{equation}
Define $a_\mu(u) := \frac{d}{du}\log\frac{p(u)}{q(u)}.$ Differentiating
\eqref{eq:shoulder-odds} gives
$a_\mu(u)=1+e^{-\mu\rho-u}/(1-e^{-\mu\rho-u})
+e^{-\mu\ell+u}/(1-e^{-\mu\ell+u})$.
Since $0\le u\le \Delta\mu$, both correction terms are $O(e^{-\rho\mu})$ uniformly. Hence, for all sufficiently large $\mu$,
\begin{equation}
\label{eq:shoulder-q-derivative}
    -q'(u)=a_\mu(u)p(u)q(u),
    \qquad
    \dfrac12\le a_\mu(u)\le2.
\end{equation}
In particular, $q$ is decreasing. Also, $q(\Delta\mu)\lesssim\mu^\beta e^{-\Delta\mu}.$ Throughout the endpoint shoulder, $k$ remains bounded above and below by positive constants. The bias expansion gives
\begin{equation}
\label{eq:shoulder-B-bounds}
    B_\theta(\alpha)
    \ge
    \dfrac{\rho}{2}q_\theta(\alpha)-\dfrac{C}{\mu},
    \qquad
    |B_\theta(\alpha)|
    \lesssim
    q_\theta(\alpha)+\dfrac1\mu.
\end{equation}
The mixture identities, the component asymptotics, and ${d_\mu}(\alpha)\ge\rho/2$ give
\begin{equation}
\label{eq:shoulder-JV-bounds}
    J_\theta(\alpha)
    \gtrsim
    \mu p_\theta(\alpha)q_\theta(\alpha),
    \qquad
    J_\theta(\alpha)
    \lesssim
    \dfrac1\mu+\mu p_\theta(\alpha)q_\theta(\alpha),
\end{equation}
and
\begin{equation}
\label{eq:shoulder-V-bound}
    \sigma_\theta^2(\alpha)
    \lesssim
    \dfrac1{\mu^2}+p_\theta(\alpha)q_\theta(\alpha).
\end{equation}
Choose constants $A>0$ and $C_0>0$ such that
\begin{equation}
\label{eq:shoulder-B-sign}
    q_\theta(\alpha)\ge\dfrac{A}{\mu}
    \Longrightarrow
    B_\theta(\alpha)\gtrsim q_\theta(\alpha),
    \qquad
    \{B_\theta<0\}
    \subseteq
    \left\{
        q_\theta\le\dfrac{C_0}{\mu}
    \right\}.
\end{equation}
From \eqref{eq:shoulder-q-derivative},
\begin{equation}
\label{eq:shoulder-pq-tail}
    \int_{\{q\le C_0/\mu\}}p(u)q(u)\,du
    \lesssim
    \dfrac1\mu.
\end{equation}
Using $d\alpha=du/\mu$, \eqref{eq:shoulder-B-bounds}, \eqref{eq:shoulder-JV-bounds}, and \eqref{eq:shoulder-pq-tail},
\begin{equation}
\label{eq:shoulder-negative-K-tail}
    \int_{\{B_\theta<0\}}
    |B_\theta(\alpha)|J_\theta(\alpha)\,d\alpha
    \lesssim
    \dfrac1{\mu^2}.
\end{equation}
Choose $R\ge2A$ and then choose $c_{\rm sh,L}<c_0$ so negative that $4Re^{c_{\rm sh,L}}\le1.$ At $u=0$, $\frac{1-r}{r} = \mu e^\chi h_\mu(\rho).$ Since $h_\mu(\rho)\le2$ for all sufficiently large $\mu$, the condition $\chi<c_{\rm sh,L}$ gives
\begin{equation}
\label{eq:shoulder-r-lower-app}
    r\ge\dfrac{R}{\mu}
\end{equation}
after increasing $\mu_{\rm sh}$.

Because $q(\Delta\mu)<A/\mu$ for all sufficiently large $\mu$, $q$ is decreasing, and $r\ge R/\mu\ge2A/\mu$, equation \eqref{eq:shoulder-q-derivative} yields
\begin{align*}
    \int_{\{q\ge A/\mu\}}p(u)q(u)^2\,du
    &\ge
    \dfrac12
    \int_{\{q\ge A/\mu\}}q(u)(-q'(u))\,du=
    \dfrac14
    \left(
        r^2-\dfrac{A^2}{\mu^2}
    \right).
\end{align*}
Hence \eqref{eq:shoulder-B-sign}, \eqref{eq:shoulder-JV-bounds}, and $d\alpha=du/\mu$ give
\begin{equation*}
    \int_{\{q\ge A/\mu\}}
    B_\theta(\alpha)J_\theta(\alpha)\,d\alpha
    \gtrsim
    r^2-\dfrac{A^2}{\mu^2}.
\end{equation*}
Combining this with \eqref{eq:shoulder-negative-K-tail} and increasing $R$ if necessary gives $\int_\rho^\ell B_\theta(\alpha)J_\theta(\alpha)\,d\alpha\gtrsim r^2.$ Therefore, by \eqref{eq:K-ratio-identity},
\begin{equation}
\label{eq:shoulder-kdot-app}
    \dot k
    \ge
    \dfrac{c_K r^2}{\mu^2}
\end{equation}
for some constant $c_K>0$.

From \eqref{eq:shoulder-q-derivative},
\begin{equation}
\label{eq:shoulder-pq-integrals}
    \int_0^{\Delta\mu}p(u)q(u)\,du\lesssim r,
    \qquad
    \int_0^{\Delta\mu}p(u)q(u)^2\,du\lesssim r^2.
\end{equation}
Since $p'=-q'=a_\mu pq$ and $a_\mu\ge1/2$, $\int_0^{\Delta\mu}q(u)\,du \lesssim \log\frac1{1-r}.$ If $r\le1/2$, then $\log(1/(1-r))\lesssim r\le\mu^2r^2$ by \eqref{eq:shoulder-r-lower-app}. If $r>1/2$, then $\frac{1-r}{r} = e^sh_\mu(\rho) \gtrsim \mu^{-\beta},$ so $\log(1/(1-r))\lesssim\log\mu\lesssim\mu^2r^2$. Thus
\begin{equation}
\label{eq:shoulder-q-integral-scale}
    \int_0^{\Delta\mu}q(u)\,du
    \lesssim
    \mu^2r^2.
\end{equation}
Using $d\alpha=du/\mu$, \eqref{eq:shoulder-B-bounds}, \eqref{eq:shoulder-V-bound}, \eqref{eq:shoulder-pq-integrals}, and \eqref{eq:shoulder-q-integral-scale},
\begin{align}
\label{eq:shoulder-mudot-app}
    |\dot\mu|
    &\lesssim
    \dfrac1\mu
    \int_0^{\Delta\mu}
    \left(
        q(u)+\dfrac1\mu
    \right)
    \left(
        \dfrac1{\mu^2}+p(u)q(u)
    \right)du
    \lesssim
    \dfrac{r^2}{\mu}.
\end{align}
For the positive radial part, only the set $\{B_\theta<0\}$ can contribute positively to $\dot\mu$. On this set, $q_\theta\le C_0/\mu$ and $|B_\theta|\lesssim\mu^{-1}$. Hence \eqref{eq:shoulder-V-bound}, \eqref{eq:shoulder-pq-tail}, and $d\alpha=du/\mu$ give
\begin{equation}
\label{eq:shoulder-positive-mudot-app}
    (\dot\mu)_+
    \lesssim
    \int_{\{B_\theta<0\}}
    |B_\theta(\alpha)|\sigma_\theta^2(\alpha)\,d\alpha
    \lesssim
    \dfrac1{\mu^3}.
\end{equation}
Finally, $b_{c_0}'(\mu) = \frac{1-\log\mu-c_0}{\mu^2}.$ Equations \eqref{eq:shoulder-kdot-app} and \eqref{eq:shoulder-mudot-app} imply
\begin{align*}
    \dot k-b_{c_0}'(\mu)\dot\mu
    &\ge
    \dfrac{c_Kr^2}{\mu^2}
    -
    C\dfrac{r^2\log\mu}{\mu^3}\\
    &\ge
    \dfrac{c_Kr^2}{2\mu^2}
    \ge
    \dfrac{c_KR^2}{2\mu^4}
\end{align*}
after increasing $\mu_{\rm sh}$. Set $a_{\rm sh}:=\dfrac{c_KR^2}{2}$ and choose $C_{\rm sh}$ large enough to dominate the constants in \eqref{eq:shoulder-positive-mudot-app} and in the comparison obtained from \eqref{eq:shoulder-kdot-app} and \eqref{eq:shoulder-mudot-app}. Then
\begin{equation*}
    \dot k-b_{c_0}'(\mu)\dot\mu
    \ge
    \dfrac{a_{\rm sh}}{\mu^4},
    \qquad
    (\dot\mu)_+
    \le
    \dfrac{C_{\rm sh}}{\mu^3},
\end{equation*}
and $\dot k>0, \ \frac{(-\dot\mu)_+}{\mu} \le C_{\rm sh}\dot k.$
\end{proof}
\begin{lemma}[Uniform barrier estimate in the logarithmic endpoint window]
\label{lem:uniform-log-window-approach}
Fix $\beta>0$. There exist constants $a_{\rm win},C_{\rm win},\mu_{\rm win}>0$ such that, whenever
$\mu\ge\mu_{\rm win}, \ \ell-\dfrac{\beta\log\mu}{\mu} \le k<b_{c_0}(\mu),$
one has
\begin{equation}
\label{eq:uniform-log-window-barrier}
    \dot k-b_{c_0}'(\mu)\dot\mu
    \ge
    \dfrac{a_{\rm win}}{\mu^4}.
\end{equation}
Moreover, whenever $\dot\mu>0$ in the same region,
\begin{equation}
\label{eq:uniform-log-window-mu-control}
    \dot\mu
    \le
    \dfrac{C_{\rm win}}{\mu^3}.
\end{equation}
Finally,
\begin{equation}
\label{eq:uniform-log-window-relative-loss}
    \dot k>0,
    \qquad
    \dfrac{(-\dot\mu)_+}{\mu}
    \le C_{\rm win}\dot k.
\end{equation}
\end{lemma}
\begin{proof}
Let $c(k,\mu):=\mu(k-\ell)-\log\mu$. By Lemma~\ref{lem:endpoint-shoulder-estimates}, there are constants $c_{\rm sh,L}<c_0$, $a_{\rm sh}>0$, $C_{\rm sh}>0$, and $\mu_{\rm sh}>0$ such that the desired barrier, positive radial bound, and relative-loss estimate hold on the lower shoulder $c(k,\mu)<c_{\rm sh,L}.$ It remains to treat the compact logarithmic subwindow $c(k,\mu)\in[c_{\rm sh,L},c_0].$ On this subwindow, write $k=b_c(\mu)$ with $c\in[c_{\rm sh,L},c_0]$. By Proposition~\ref{prop:log-rail-expansion},
\begin{equation*}
    \dot k
    =
    -\dfrac{d_k(c)}{\mu^4}
    +
    o(\mu^{-4}),
    \qquad
    \dot\mu
    =
    \dfrac{d_{\rm rad}(c)}{\mu^3}
    +
    o(\mu^{-3})
\end{equation*}
uniformly over $c\in[c_{\rm sh,L},c_0]$. Since $d_k(c)<0$ for every $c\le c_0$, compactness gives $c_{\rm cp}>0$ such that $-d_k(c)\ge4c_{\rm cp}, \ c\in[c_{\rm sh,L},c_0].$ Therefore, for all sufficiently large $\mu$,
\begin{equation}
\label{eq:compact-c-drift-bounds}
    \dot k
    \ge
    \dfrac{2c_{\rm cp}}{\mu^4},
    \qquad
    |\dot\mu|
    \le
    \dfrac{C_{\rm cp}}{\mu^3}
\end{equation}
for some $C_{\rm cp}>0$. Since $|b_{c_0}'(\mu)\dot\mu| \lesssim \frac{\log\mu}{\mu^5} = o(\mu^{-4}),$ we get
\begin{equation}
\label{eq:compact-c-barrier-drift}
    \dot k-b_{c_0}'(\mu)\dot\mu
    \ge
    \dfrac{c_{\rm cp}}{\mu^4}
\end{equation}
after increasing the threshold.

Combining the lower-shoulder estimates with \eqref{eq:compact-c-drift-bounds} and \eqref{eq:compact-c-barrier-drift}, define $a_{\rm win} := \min\{a_{\rm sh},c_{\rm cp}\}.$ Choose
\begin{equation*}
    C_{\rm win}
    \ge
    \max\left\{
        C_{\rm sh},
        C_{\rm cp},
        \dfrac{C_{\rm cp}}{2c_{\rm cp}}
    \right\}.
\end{equation*}
Taking $\mu_{\rm win}$ larger than all thresholds used above proves \eqref{eq:uniform-log-window-barrier}, \eqref{eq:uniform-log-window-mu-control}, and \eqref{eq:uniform-log-window-relative-loss}.
\end{proof}
\begin{proposition}[One-sided approach to the logarithmic rail]
\label{prop:approach-log-rail}
Fix $0<\eta<\Delta/4$. There exist constants $a_{\rm ent},C_{\rm ent},\mu_{\rm ent}>0$ such that, whenever
$\mu\ge\mu_{\rm ent}, \ \ell-\eta \le k < b_{c_0}(\mu),$
one has
\begin{equation}
\label{eq:approach-log-gap-drift}
    \dot k-b_{c_0}'(\mu)\dot\mu
    \ge
    \dfrac{a_{\rm ent}}{\mu^4}.
\end{equation}
Moreover, whenever $\dot\mu>0$ in the same region,
\begin{equation}
\label{eq:approach-log-mu-control}
    \dot\mu
    \le
    \dfrac{C_{\rm ent}}{\mu^3}.
\end{equation}
Finally,
\begin{equation}
\label{eq:approach-log-relative-loss}
    \dot k>0,
    \qquad
    \dfrac{(-\dot\mu)_+}{\mu}
    \le
    C_{\rm ent}\dot k.
\end{equation}
\end{proposition}
\begin{proof}
Choose $0<\eta_1<\eta$ and $\beta_{\rm sep}>2$. Split the region into
\begin{equation*}
    [\ell-\eta,\ell-\eta_1],
    \qquad
    \left[
        \ell-\eta_1,
        \ell-\dfrac{\beta_{\rm sep}\log\mu}{\mu}
    \right],
    \qquad
    \left[
        \ell-\dfrac{\beta_{\rm sep}\log\mu}{\mu},
        b_{c_0}(\mu)
    \right].
\end{equation*}
On the first interval, Proposition~\ref{prop:interior-slingshot}, applied with $\eta_1$, gives $\dot k\gtrsim\dfrac1{\mu^2}, \ \dot\mu<0, \ |\dot\mu|\lesssim\dfrac1\mu.$
On the second interval, Lemma~\ref{lem:interior-log-separation} gives the same bounds. Since $|b_{c_0}'(\mu)|\lesssim\log\mu/\mu^2$, both regions satisfy
\begin{equation*}
    \dot k-b_{c_0}'(\mu)\dot\mu
    \gtrsim
    \dfrac1{\mu^2},
    \qquad
    \dfrac{(-\dot\mu)_+}{\mu}
    \lesssim
    \dfrac1{\mu^2}
    \lesssim
    \dot k.
\end{equation*}
On the endpoint window, Lemma~\ref{lem:uniform-log-window-approach} gives $\dot k-b_{c_0}'(\mu)\dot\mu \ge \frac{a_{\rm win}}{\mu^4},$ as well as the positive radial bound and the relative-loss estimate. Combining the three regions proves \eqref{eq:approach-log-gap-drift}, \eqref{eq:approach-log-mu-control}, and \eqref{eq:approach-log-relative-loss} after increasing $\mu_{\rm ent}$ and enlarging $C_{\rm ent}$.
\end{proof}
\begin{proposition}[Stopped finite-time approach to the logarithmic tube]
\label{prop:finite-entry-progress-region}
Assume $C_{\rm band}\ge2$, and fix $0<\eta<\Delta/4$. There exists $\mu_{\rm G,ent}>0$ such that the following holds. Suppose that, at some time $t_0$,
\begin{equation*}
    k(t_0)\in\mathcal T_\eta(\mu(t_0)),
    \qquad
    k(t_0)<b_{c_0}(\mu(t_0)),
    \qquad
    \mu(t_0)>\mu_{\rm G,ent}.
\end{equation*}
Define $\tau_{\log} := \inf\{t\ge t_0:k(t)\ge b_{c_0}(\mu(t))\}$ and $\tau_{\rm G,exit} := \inf\{t\ge t_0:\mu(t)\le\mu_{\rm G,ent}\}.$ Then
$\tau_{\log}\wedge\tau_{\rm G,exit}<\infty.$
In particular, if $\tau_{\rm G,exit}>\tau_{\log}$, then the trajectory satisfies $k(t)\in\mathcal I_{\log}(\mu(t))$ in finite time.
\end{proposition}
\begin{proof}
Increase $\mu_{\rm G,ent}$ so that $\mu_{\rm G,ent} \ge \max\{\mu_{\rm tube,\eta},\mu_{\rm ent}\}.$ Before $\tau_{\log}\wedge\tau_{\rm G,exit}$, Proposition~\ref{prop:upper-rail-tube} gives $\ell-\eta\le k(t)<b_{c_0}(\mu(t)).$ With $g_{\log}(t):=b_{c_0}(\mu(t))-k(t)$, Proposition~\ref{prop:approach-log-rail} gives $\dot g_{\log}(t) \le -\frac{a_{\rm ent}}{\mu(t)^4}.$ The same proposition gives $\dot\mu(t)\le C_{\rm ent}/\mu(t)^3$ whenever $\dot\mu(t)>0$, and this upper bound is trivial when $\dot\mu(t)\le0$. Lemma~\ref{lem:abstract-finite-entry}\textnormal{(ii)} applied on the stopped interval proves that $\tau_{\log}\wedge\tau_{\rm G,exit}<\infty$.
\end{proof}
\begin{corollary}[Forward-invariant logarithmic prediction-good tube]
\label{cor:forward-invariant-progress-region}
Assume $C_{\rm band}\ge2$, and fix $0<\eta<\Delta/4$. There exist constants $C_{\rm bias},\mu_{\log,*}>0$ such that, for every $\mu\ge\mu_{\log,*}$, the interval $\mathcal I_{\log}(\mu)$ is nonempty and contained in $\mathcal T_\eta(\mu)$. Moreover, if $k(t_0)\in\mathcal I_{\log}(\mu(t_0)),\ \mu(t_0)\ge\mu_{\log,*},$ and $\tau_{\log,*} := \inf\{t\ge t_0:\mu(t)\le\mu_{\log,*}\},$ then
$k(t)\in\mathcal I_{\log}(\mu(t))$
for every $t\in[t_0,\tau_{\log,*})$. Throughout the tube,
\begin{equation*}
    \sup_{\alpha\in[\rho,\ell]}
    |F(\lambda,\mu;\alpha)-\alpha|
    \le
    \dfrac{C_{\rm bias}}{\mu}.
\end{equation*}
\end{corollary}
\begin{proof}
Since $C_{\rm band}\ge2$, $b_{c_0}(\mu)\le U_{\rm band}(\mu)$ for all sufficiently large $\mu$, while $b_{c_0}(\mu)>\ell-\eta$. Hence $\mathcal I_{\log}(\mu)$ is nonempty and contained in $\mathcal T_\eta(\mu)$ after increasing $\mu_{\log,*}$.

Set $a(\mu):=b_{c_0}(\mu), \ b(\mu):=U_{\rm band}(\mu).$ At the lower boundary, Proposition~\ref{prop:log-lower-rail-invariance} gives $\dot k-a'(\mu)\dot\mu = \dot k-b_{c_0}'(\mu)\dot\mu >0.$ At the upper boundary $k=b(\mu)=U_{\rm band}(\mu)$, Corollary~\ref{cor:upper-moving-boundary-correction} gives
$b'(\mu)\dot\mu-\dot k = U_{\rm band}'(\mu)\dot\mu-\dot k \ge \dfrac{a_U}{\mu^4}.$
Equivalently, $\dot k-b'(\mu)\dot\mu<0$ on the upper boundary. Lemma~\ref{lem:moving-interval-invariance} therefore shows that the ratio cannot leave the moving interval $\mathcal I_{\log}(\mu(t))$ before $\tau_{\log,*}$. The prediction bound follows from Lemma~\ref{lem:log-prediction-bound} with $c=c_0$.
\end{proof}
\textbf{Uniform radial growth after logarithmic capture}
\begin{lemma}[Uniform radial expansion in the logarithmic tube]
\label{lem:uniform-radial-expansion-progress-region}
Let $c(k,\mu):=\mu(k-\ell)-\log\mu.$ Then
\begin{equation}
\label{eq:uniform-radial-expansion-progress-region}
    \sup_{k\in[b_{c_0}(\mu),U_{\rm band}(\mu)]}
    \left|
        \mu^3\dot\mu-d_{\rm rad}(c(k,\mu))
    \right|
    \longrightarrow0
\end{equation}
as $\mu\to\infty$. Moreover,
\begin{equation*}
    m_{\rm rad}
    :=
    \inf_{c\ge c_0}d_{\rm rad}(c)
    >
    0.
\end{equation*}
\end{lemma}
\begin{proof}
Inside the tube, $c_0 \le c(k,\mu) \le (C_{\rm band}-1)\log\mu.$ Fix $c_1>c_0$. On the compact interval $[c_0,c_1]$, Proposition~\ref{prop:log-rail-expansion} gives
\begin{equation}
\label{eq:uniform-radial-compact-c}
    \sup_{c\in[c_0,c_1]}
    \left|
        \mu^3\dot\mu-d_{\rm rad}(c)
    \right|
    \longrightarrow0.
\end{equation}
Now suppose $c\ge c_1$. The odds identity yields $q_\theta\left(\rho+\frac{u}{\mu}\right) \le \frac{Ce^{-c-u}}{\mu}.$ Therefore
\begin{equation*}
    \int_\rho^\ell q_\theta(\alpha)\,d\alpha
    \le
    \dfrac{Ce^{-c}}{\mu^2},
    \qquad
    \int_\rho^\ell q_\theta(\alpha)^2\,d\alpha
    \le
    \dfrac{Ce^{-2c}}{\mu^3}.
\end{equation*}
Using the bias expansion and the uniform component asymptotics, the component-variance contribution satisfies
\begin{equation*}
    \dot\mu_{\rm comp}
    =
    \dfrac1{\mu^3}
    +
    o(\mu^{-3})
    +
    O\left(\dfrac{e^{-c}}{\mu^4}\right)
\end{equation*}
uniformly for $c\ge c_1$. The mixture contribution satisfies
\begin{equation*}
    |\dot\mu_{\rm mix}|
    \le
    C\int_\rho^\ell
    \left(
        \dfrac1\mu+q_\theta(\alpha)
    \right)
    q_\theta(\alpha)\,d\alpha
    \le
    \dfrac{C}{\mu^3}
    \left(
        e^{-c}+e^{-2c}
    \right).
\end{equation*}
Hence $\left| \mu^3\dot\mu-1 \right| \le C(e^{-c}+e^{-2c})+o(1)$ uniformly for $c\ge c_1$. Since $|d_{\rm rad}(c)-1|\le C(e^{-c}+e^{-2c})$, we obtain
\begin{equation*}
    \sup_{c_1\le c\le(C_{\rm band}-1)\log\mu}
    \left|
        \mu^3\dot\mu-d_{\rm rad}(c)
    \right|
    \le
    Ce^{-c_1}+o(1).
\end{equation*}
Letting first $c_1\to\infty$ and then $\mu\to\infty$, together with \eqref{eq:uniform-radial-compact-c}, proves \eqref{eq:uniform-radial-expansion-progress-region}. By Lemma~\ref{lem:choose-c0}, $d_{\rm rad}(c)>0$ for every $c\ge c_0$, and $d_{\rm rad}(c)\to1$ as $c\to\infty$. Continuity therefore gives $m_{\rm rad}>0$.
\end{proof}
\begin{proposition}[Two-sided radial drift in the logarithmic tube]
\label{prop:radial-growth-progress-region}
There exist constants $0<c_\mu<C_\mu<\infty$ and $\mu_{\rm rad}$ such that, whenever $\mu\ge\mu_{\rm rad}$ and $k\in\mathcal I_{\log}(\mu)$,
\[
    \frac{c_\mu}{\mu^3}\le \dot\mu\le\frac{C_\mu}{\mu^3}.
\]
Consequently, if the trajectory remains in the logarithmic tube for $t\ge t_0$, then
\[
    \mu(t_0)^4+4c_\mu(t-t_0)
    \le \mu(t)^4
    \le \mu(t_0)^4+4C_\mu(t-t_0).
\]
\end{proposition}
\begin{proof}
Lemma~\ref{lem:uniform-radial-expansion-progress-region} gives $\mu^3\dot\mu=d_{\rm rad}(c(k,\mu))+o(1)$ uniformly over $k\in\mathcal I_{\log}(\mu)$. By Lemma~\ref{lem:choose-c0}, $d_{\rm rad}(c)>0$ for $c\ge c_0$, while $d_{\rm rad}(c)\to1$ as $c\to\infty$. Hence $0<\inf_{c\ge c_0}d_{\rm rad}(c)\le\sup_{c\ge c_0}d_{\rm rad}(c)<\infty$, and the two-sided drift bound follows for sufficiently large $\mu$. Multiplying by $4\mu^3$ and integrating proves the fourth-power bounds.
\end{proof}

\textbf{Pre-capture preservation of the large-parameter regime}
\label{subsubsec:large-mu-preservation}
The preceding regional transport arguments require lower bounds on $\mu$. We now show that these lower bounds are preserved automatically when the initial radius is sufficiently large and the initial ratio is bounded away from zero.
\begin{proposition}[Relative radial-loss estimate before logarithmic capture]
\label{prop:relative-radial-loss}
Assume $C_{\rm band}\ge2$, fix $0<\eta<\Delta/4$, and fix $\kappa>0$. There exist constants $C_{\rm pres},\mu_{\rm rel}>0$ such that, whenever
\begin{equation}
\label{eq:relative-loss-region}
    \mu\ge\mu_{\rm rel},
    \qquad
    \kappa\le k\le U_{\rm band}(\mu),
    \qquad
    k<b_{c_0}(\mu),
\end{equation}
one has
\begin{equation}
\label{eq:relative-loss-k-positive}
    \dot k>0
\end{equation}
and
\begin{equation}
\label{eq:relative-radial-loss}
    \dfrac{(-\dot\mu)_+}{\mu}
    \le C_{\rm pres}\dot k,
    \qquad
    (-\dot\mu)_+:=\max\{-\dot\mu,0\}.
\end{equation}
\end{proposition}
\begin{proof}
The mixture identities and $\lambda=k\mu$ give
\begin{equation}
\label{eq:J-decomposition-preservation}
    J_\theta(\alpha)
    =
    \mu p_\theta(\alpha)q_\theta(\alpha){d_\mu}(\alpha)
    +k\mu \sigma_\theta^2(\alpha)
    \ge k\mu \sigma_\theta^2(\alpha).
\end{equation}
Set $A:=\int_\rho^\ell B_\theta(\alpha)\sigma_\theta^2(\alpha)\,d\alpha, \ K:=\int_\rho^\ell B_\theta(\alpha)J_\theta(\alpha)\,d\alpha.$ Then $-\dot\mu=\frac{A}{\Delta}, \ \dot k=\frac{K}{\Delta\mu^2}.$ We cover the region in \eqref{eq:relative-loss-region} by four subregions. \textbf{1. Lower outer region.} If $\kappa\le k<L_{\rm band}(\mu)$, the odds identity and Corollary~\ref{cor:bias-expansion} give $B_\theta(\alpha)\ge c_\rho>0, \ \alpha\in[\rho,\ell],$ for all sufficiently large $\mu$. Hence $A>0$, and \eqref{eq:J-decomposition-preservation} implies
\begin{equation*}
    K
    =
    \int_\rho^\ell B_\theta(\alpha)J_\theta(\alpha)\,d\alpha
    \ge
    k\mu
    \int_\rho^\ell B_\theta(\alpha)\sigma_\theta^2(\alpha)\,d\alpha
    =
    k\mu A
    \ge
    \kappa\mu A.
\end{equation*}
Therefore $K>0$ and
\begin{equation*}
    \dfrac{(-\dot\mu)_+}{\mu}
    =
    \dfrac{A}{\Delta\mu}
    \le
    \dfrac{K}{\Delta\kappa\mu^2}
    =
    \dfrac1\kappa\dot k.
\end{equation*}
\textbf{2. Lower boundary layer.} If $L_{\rm band}(\mu)\le k\le\rho+\eta$, Proposition~\ref{prop:lower-boundary-layer-escape} gives $\dot k>0, \ \frac{(-\dot\mu)_+}{\mu} \le C_{\rm low,\eta}\dot k.$ \textbf{3. Fixed interior strip.} If $\rho+\eta\le k\le\ell-\eta$, Proposition~\ref{prop:interior-slingshot} gives $\dot k>0, \ \frac{(-\dot\mu)_+}{\mu} \le C_{\rm int,\eta}\dot k.$ \textbf{4. Upper approach region.} If $\ell-\eta\le k<b_{c_0}(\mu)$, Proposition~\ref{prop:approach-log-rail} gives the desired estimates directly.

Taking $C_{\rm pres}$ to be the maximum of the constants in the four regions and increasing $\mu_{\rm rel}$ so that all corresponding thresholds apply proves \eqref{eq:relative-loss-k-positive} and \eqref{eq:relative-radial-loss}.
\end{proof}
\begin{lemma}[Integrated radial-loss comparison]
\label{lem:integrated-radial-loss}
Let $[t_0,t_1]$ be an interval on which
\begin{equation*}
    \mu(t)>0,
    \qquad
    \dot k(t)\ge0,
    \qquad
    \dfrac{(-\dot\mu(t))_+}{\mu(t)}
    \le C\dot k(t).
\end{equation*}
Then, for every $t\in[t_0,t_1]$,
$\mu(t) \ge \mu(t_0)\exp\{-C(k(t)-k(t_0))\}.$
\end{lemma}
\begin{proof}
Since $\frac{d}{dt}\log\mu(t)=\dot\mu(t)/\mu(t)
\ge-(-\dot\mu(t))_+/\mu(t)\ge-C\dot k(t)$,
integration gives $\log\frac{\mu(t)}{\mu(t_0)} \ge -C\{k(t)-k(t_0)\}.$ Exponentiating proves the claim.
\end{proof}
\subsection{Completion of the proof}
\label{subsec:completion}

\begin{proof}[Proof of Theorem~\ref{thm:main-population}]
Fix $\eta:=\Delta/8$ and take $\mu_*$ according to the threshold convention
above. Increase $\mu_*$ once more, if necessary, so that
$b_{c_0}(\mu)<2$ for every $\mu\ge\mu_*$. Let $C_{\rm pres}$ be the
constant in Proposition~\ref{prop:relative-radial-loss} and set
\[
    c_{\rm rad}:=e^{-2C_{\rm pres}}.
\]
Increase the initialization threshold so that
$\mu(0)\ge2c_{\rm rad}^{-1}\mu_*$. Write
$\tau_{\mathcal G}:=\inf\{t\ge0:\theta(t)\in\mathcal G\}$.

If $\theta(0)\in\mathcal G$, then $\tau_{\mathcal G}=0$. If
$k(0)>U_{\rm band}(\mu(0))$, Proposition~\ref{prop:finite-entry} gives
finite-time entry through $k=U_{\rm band}(\mu)$, while
Proposition~\ref{prop:quant-outer} shows that $\mu$ increases beforehand.
Since $b_{c_0}(\mu)\le U_{\rm band}(\mu)$ for $\mu\ge\mu_*$, this entry
point already belongs to $\mathcal G$.

It therefore remains only to consider
$k(0)\le U_{\rm band}(\mu(0))$ with
$k(0)<b_{c_0}(\mu(0))$. In this region,
Proposition~\ref{prop:relative-radial-loss} gives $\dot k>0$ and
$(-\dot\mu)_+/\mu\le C_{\rm pres}\dot k$ as long as
$\mu\ge\mu_*$. Before capture,
$k(t)<b_{c_0}(\mu(t))<2$, while $k(0)>0$. Hence
$k(t)-k(0)<2$, and Lemma~\ref{lem:integrated-radial-loss} gives
\[
    \mu(t)
    \ge
    \mu(0)e^{-C_{\rm pres}\{k(t)-k(0)\}}
    \ge
    c_{\rm rad}\mu(0)
    \ge
    2\mu_*.
\]
Thus none of the stopped large-radius exits can occur before capture.

The regional transport results may therefore be applied successively, skipping any stage already passed: Proposition~\ref{prop:finite-entry} enters the dynamic band, Corollary~\ref{cor:lower-layer-to-interior} enters the fixed interior strip, Corollary~\ref{cor:interior-passage} reaches the upper endpoint region, Proposition~\ref{prop:upper-rail-tube} keeps the trajectory inside the upper tube, and Proposition~\ref{prop:finite-entry-progress-region} gives finite-time crossing of the lower logarithmic rail. Hence $\tau_{\mathcal G}<\infty$, and the same persistence estimate gives $\inf_{0\le t\le\tau_{\mathcal G}}\mu(t)\ge c_{\rm rad}\mu(0)$.

At capture, Corollary~\ref{cor:forward-invariant-progress-region} prevents side-rail exit and Proposition~\ref{prop:radial-growth-progress-region} gives positive radial drift, so radial exit is impossible as well. Thus $\theta(t)\in\mathcal G$ for every $t\ge\tau_{\mathcal G}$, and Proposition~\ref{prop:radial-growth-progress-region} gives the two-sided bound $c\mu(t)^{-3}\le\dot\mu(t)\le C\mu(t)^{-3}$ and hence \eqref{eq:main-population-growth} after renaming constants.

Finally, Lemma~\ref{lem:log-prediction-bound} yields $\sup_{\alpha\in[\rho,\ell]}|F(\theta(t);\alpha)-\alpha|\le C/\mu(t)$. Therefore $\mathcal R_\rho(\theta(t))\le C/\mu(t)^2$. Membership in $\mathcal G$ gives $k(t)=\ell+O(\{\log\mu(t)\}/\mu(t))$, while the fourth-power bounds imply $\mu(t)\asymp t^{1/4}$ after capture. The prediction and risk rates in \eqref{eq:main-population-time-rates} follow, and since $\ell>0$, both $\mu(t)$ and $\lambda(t)=k(t)\mu(t)$ diverge.
\end{proof}

\section{Population Gradient Descent}
\label{sec:finite-step-gd}

Define $v:=-\nabla\mathcal R_\rho$ and the population GD iterates
\begin{equation}
\label{eq:population-gd}
    \theta_{s+1}^{\gamma}=\theta_s^{\gamma}+\gamma v(\theta_s^{\gamma}),
    \qquad \theta_0^\gamma=\theta(0).
\end{equation}
Choose a deterministic $T_0>\tau_{\mathcal G}$ for which the population flow is strictly inside $\mathcal G$, and choose $\tau_0>0$ so that positive radial and side-barrier margins persist on $[T_0,T_0+\tau_0]$. Let $s_0(\gamma):=\lceil T_0/\gamma\rceil$.

For $R\ge\mu_*$, write
\[
    \mathcal S_R
    :=
    \mathcal G\cap\{R\le\mu\le2R\}
\]
for the core logarithmic shell.

\subsection{Regularity and compact Euler tracking}
\begin{lemma}[Global Lipschitz regularity of the population vector field]
\label{lem:population-vector-field-global-regularity}
There exists a universal constant $L_v<\infty$ such that
\[
    \|v(\theta)-v(\vartheta)\|
    \le
    L_v\|\theta-\vartheta\|,
    \qquad
    \theta,\vartheta\in\R^2.
\]
\end{lemma}

\begin{proof}
For fixed $\alpha$, write
\[
    g(\theta;\alpha)
    :=
    \{F(\theta;\alpha)-\alpha\}\nabla F(\theta;\alpha),
    \qquad
    v(\theta)=-\E_\alpha g(\theta;\alpha).
\]
Let $T_1:=Y$ and $T_2:=X$, with $a,b\in\{1,2\}$ indexing
$(\lambda,\mu)$. Exponential-family differentiation gives
\[
    \partial_a F
    =
    \Cov_\theta(X,T_a\mid\alpha),
    \qquad
    \partial_{ab}^2F
    =
    \E_\theta\!\left[
        (X-\E_\theta X)
        (T_a-\E_\theta T_a)
        (T_b-\E_\theta T_b)
        \mid\alpha
    \right].
\]
Since $X,Y\in[0,1]$, both $\|\nabla F\|$ and
$\|\nabla^2F\|_{\rm op}$ are uniformly bounded in $\theta$ and $\alpha$.
Also $|F-\alpha|\le1$. Hence
\[
    Dg
    =
    \nabla F\nabla F^\top
    +
    (F-\alpha)\nabla^2F
\]
has uniformly bounded operator norm. Therefore $g(\cdot;\alpha)$ is
globally Lipschitz with a constant independent of $\alpha$.
Averaging over $\alpha$ gives the same property for
$v=-\E_\alpha g$.
\end{proof}

\begin{lemma}[Population speed in the logarithmic tube]
\label{lem:population-speed-progress-region}
There is a constant $C_{\rm sp}<\infty$ such that
\begin{equation*}
    \|v(\theta)\|
    \le
    \frac{C_{\rm sp}}{\mu^3},
    \qquad
    \theta=(\lambda,\mu)\in\mathcal G.
\end{equation*}
\end{lemma}

\begin{proof}
Inside $\mathcal G$, Lemma~\ref{lem:log-prediction-bound} gives $|F(\theta;\alpha)-\alpha| \le \frac{C}{\mu}$ uniformly in $\alpha$. Moreover, the odds formula and the lower rail imply $q_\theta(\alpha) \le \frac{C}{\mu} e^{-\mu(\alpha-\rho)}.$ Using the mixture identities and the component asymptotics, $|\Cov_\theta(X,Y\mid\alpha)| \le Cq_\theta(\alpha),$ and $\Var_\theta(X\mid\alpha) \le \frac{C}{\mu^2} + Cq_\theta(\alpha).$ Hence $\int_\rho^\ell |\Cov_\theta(X,Y\mid\alpha)|\,d\alpha \le \frac{C}{\mu^2},$ and $\int_\rho^\ell \Var_\theta(X\mid\alpha)\,d\alpha \le \frac{C}{\mu^2}.$ Substituting into the risk-gradient identity proves
\begin{equation*}
    |\partial_\lambda\mathcal R_\rho(\theta)|
    +
    |\partial_\mu\mathcal R_\rho(\theta)|
    \le
    \frac{C}{\mu^3}.
\end{equation*}
Since $v=-\nabla\mathcal R_\rho$, the claim follows.
\end{proof}

\begin{lemma}[Barrier regularity]
\label{lem:gd-barrier-regularity}
There are constants $C_h,C_{hh}<\infty$ such that, for every $\theta=(\lambda,\mu)\in\mathcal G$,
$\|\nabla h_-(\theta)\| + \|\nabla h_+(\theta)\| \le \frac{C_h}{\mu},$
and
\begin{equation*}
    \|\nabla^2h_-(\theta)\|_{\rm op}
    +
    \|\nabla^2h_+(\theta)\|_{\rm op}
    \le
    \frac{C_{hh}}{\mu^2}.
\end{equation*}
\end{lemma}

\begin{proof}
Since $h_-(\lambda,\mu) = \frac{\lambda}{\mu}-b_{c_0}(\mu),$ and $h_+(\lambda,\mu) = U_{\rm band}(\mu)-\frac{\lambda}{\mu},$ the first derivatives of $\lambda/\mu$ have size $O(\mu^{-1})$ on the tube because $k=\lambda/\mu=O(1)$. Its second derivatives have size $O(\mu^{-2})$. Finally, $b_{c_0}'(\mu), \ U_{\rm band}'(\mu) = O\left(\frac{\log\mu}{\mu^2}\right),$ and $b_{c_0}''(\mu), \ U_{\rm band}''(\mu) = O\left(\frac{\log\mu}{\mu^3}\right).$ After increasing the fixed base radius $\mu_*$, these terms are bounded by the claimed scales.
\end{proof}

For a vector field $w$ on the logarithmic tube, define $\Psi_{\pm,w}(\theta) := \mu^4 \langle\nabla h_\pm(\theta),w(\theta)\rangle.$

\begin{lemma}[Normalized directional regularity]
\label{lem:normalized-directional-regularity}
Let $R\ge\mu_*$ and suppose a $C^1$ vector field $w$ on a fixed-width enlargement of $\mathcal S_R$ satisfies
$\|w(\theta)\| \le C_w, \ \|Dw(\theta)\|_{\rm op} \le L_w$
there. Then
\begin{equation*}
    |\Psi_{\pm,w}(\theta)-\Psi_{\pm,w}(\vartheta)|
    \le
    C(C_w+L_w)R^3
    \|\theta-\vartheta\|
\end{equation*}
whenever $\theta,\vartheta$ lie in that enlargement.
\end{lemma}

\begin{proof}
On a fixed-width enlargement of $\mathcal S_R$, one has $\mu\asymp R$ and $k=O(1)$. Direct differentiation of the explicit barriers therefore gives
\begin{equation*}
    \|\mu^4\nabla h_\pm(\theta)\|
    \le
    CR^3,
    \qquad
    \left\|
        D\{\mu^4\nabla h_\pm(\theta)\}
    \right\|_{\rm op}
    \le
    CR^2.
\end{equation*}
Using $\Psi_{\pm,w}(\theta) = \langle \mu^4\nabla h_\pm(\theta),w(\theta)\rangle,$ the product rule gives $\|D\Psi_{\pm,w}(\theta)\| \le CR^2C_w + CR^3L_w \le C(C_w+L_w)R^3,$ where $R\ge\mu_*\ge1$. The mean-value inequality proves the claim.
\end{proof}

\subsection{Finite-horizon Euler shadowing}
\begin{lemma}[Euler shadowing on a compact time interval]
\label{lem:population-gd-shadowing}
Let $\theta(t)$ solve $\dot\theta=v(\theta)$ on $[0,T]$, and suppose a compact set $K$ contains a positive neighborhood of $\theta([0,T])$. Assume
$\|v(\theta)\|\le V, \ \|v(\theta)-v(\vartheta)\| \le L\|\theta-\vartheta\|$
on $K$. Let
\begin{equation*}
    \theta_{s+1}^{\gamma}
    =
    \theta_s^\gamma+\gamma v(\theta_s^\gamma),
    \qquad
    \theta_0^\gamma=\theta(0).
\end{equation*}
Then, for all sufficiently small $\gamma$ and all integers $s$ satisfying $s\gamma\le T$,
$\|\theta_s^\gamma-\theta(s\gamma)\| \le C_{T,L,V}\gamma.$
\end{lemma}

\begin{proof}
For $t_s=s\gamma$, $\theta(t_{s+1}) = \theta(t_s)+\gamma v(\theta(t_s))+\tau_s,$ where
\begin{align*}
    \|\tau_s\|
    &\le
    \int_{t_s}^{t_{s+1}}
    \|v(\theta(u))-v(\theta(t_s))\|\,du\\
    &\le
    L\int_{t_s}^{t_{s+1}}
    \|\theta(u)-\theta(t_s)\|\,du\\
    &\le
    \frac12LV\gamma^2.
\end{align*}
Set $e_s=\|\theta_s^\gamma-\theta(t_s)\|$. Before either trajectory leaves $K$, $e_{s+1} \le (1+\gamma L)e_s + \frac12LV\gamma^2.$ Since $e_0=0$,
$e_s\le(V\gamma/2)\{(1+\gamma L)^s-1\}
\le(V\gamma/2)(e^{LT}-1)$.
For sufficiently small $\gamma$ this is smaller than the clearance of the reference path from $\partial K$, so a stopped argument rules out exit from $K$. This proves the bound.
\end{proof}

Applying Lemma~\ref{lem:population-gd-shadowing} on $[0,T_0+\tau_0]$ gives, for all sufficiently small $\gamma$,
\begin{equation}
\label{eq:population-gd-shadowing-bound}
    \max_{0\le s\le s_0(\gamma)}\|\theta_s^\gamma-\theta(s\gamma)\|\le C_{\rm gd}\gamma.
\end{equation}
Since $s_0(\gamma)\gamma\in[T_0,T_0+\gamma]$ and the reference flow has strict tube margin there, reducing the learning-rate threshold if necessary gives
\begin{equation}
\label{eq:population-gd-entry}
    \theta_{s_0(\gamma)}^\gamma\in\operatorname{int}(\mathcal G).
\end{equation}

\subsection{One-step barrier stability}
\begin{lemma}[No-overshoot barrier lemma]
\label{lem:discrete-no-overshoot}
Fix $R\ge\mu_*$. Choose a fixed $K_{\rm tube}>0$ larger than $\sup\{|\lambda/\mu|:(\lambda,\mu)\in\mathcal G\}+1$, and define the two-sided side-rail enlargement
\begin{equation*}
    \mathcal U_R
    :=
    \left\{
        (\lambda,\mu):
        R\le\mu\le4R,
        \quad
        |\lambda/\mu|\le K_{\rm tube}
    \right\}.
\end{equation*}
Let $w$ be a $C^1$ vector field on an open neighborhood of $\mathcal U_R$. Suppose
\begin{equation}
\label{eq:generic-discrete-radial-speed}
    \frac{c_r}{\mu^3}
    \le
    w_\mu(\theta)
    \le
    \frac{C_r}{\mu^3},
    \qquad
    \|w(\theta)\|
    \le
    \frac{C_w}{\mu^3},
\end{equation}
for $\theta\in\mathcal S_R$, and suppose further that
\begin{equation}
\label{eq:generic-discrete-barrier-margin}
    \langle\nabla h_-(\theta),w(\theta)\rangle
    \ge
    \frac{a_-'}{\mu^4}
    \quad\text{for }\theta\in\mathcal U_R\text{ with }h_-(\theta)=0,
\end{equation}
and
\begin{equation*}
    \langle\nabla h_+(\theta),w(\theta)\rangle
    \ge
    \frac{a_+'}{\mu^4}
    \quad\text{for }\theta\in\mathcal U_R\text{ with }h_+(\theta)=0.
\end{equation*}
Assume that the normalized directional functions $\Psi_{\pm,w}$ are $CR^3$-Lipschitz on $\mathcal U_R$. Then there exists a constant $\gamma_{\rm bar}>0$, depending only on the constants in these bounds and on $\mu_*$, such that for every $0<\gamma\le\gamma_{\rm bar}$ the update
$\theta^+ = \theta+\gamma w(\theta)$
cannot cross either side boundary of $\mathcal G$ from a starting point $\theta\in\mathcal S_R$.
\end{lemma}

\begin{proof}
We treat the lower rail; the upper rail is identical. Suppose for contradiction that $h_-(\theta)\ge0, \ h_-(\theta^+)<0.$ By continuity, the segment $[\theta,\theta^+]$ contains a point $\zeta$ satisfying $h_-(\zeta)=0$. By the speed bound, $\|\theta-\zeta\| \le \|\theta^+-\theta\| \le \frac{C\gamma}{R^3}.$ Because $w_\mu(\theta)>0$ and the step length is $O(\gamma R^{-3})$, for sufficiently small fixed $\gamma$ the segment stays in $\mathcal U_R$. At $\zeta$, $\Psi_{-,w}(\zeta) \ge c\,a_-'$ after absorbing the comparability of $\mu(\zeta)$ and $R$. The Lipschitz bound gives $\Psi_{-,w}(\theta) \ge c\,a_-'-C\gamma.$ Choose $\gamma$ so that the right-hand side is at least $c\,a_-'/2$. Then
\begin{equation}
\label{eq:discrete-starting-inward}
    \langle\nabla h_-(\theta),w(\theta)\rangle
    \ge
    \frac{c\,a_-'}{2R^4}.
\end{equation}
Taylor's theorem gives
\begin{equation*}
    h_-(\theta^+)
    =
    h_-(\theta)
    +
    \gamma
    \langle\nabla h_-(\theta),w(\theta)\rangle
    +
    \frac{\gamma^2}{2}
    w(\theta)^\top
    \nabla^2h_-(\xi)
    w(\theta)
\end{equation*}
for a point $\xi$ on the segment. Since the whole segment lies in $\mathcal U_R$, where $\mu\asymp R$ and $k=O(1)$, direct differentiation of $h_-$ gives $\|\nabla^2h_-(\xi)\|_{\rm op} \le \frac{C}{R^2}.$ Together with the speed bound,
\begin{equation*}
    \left|
        \frac{\gamma^2}{2}
        w(\theta)^\top
        \nabla^2h_-(\xi)
        w(\theta)
    \right|
    \le
    C\frac{\gamma^2}{R^8}.
\end{equation*}
The first-order term in \eqref{eq:discrete-starting-inward} is at least $c\gamma/R^4$. Since $R\ge\mu_*$, a further fixed reduction of $\gamma_{\rm bar}$ makes the remainder at most half of this term. Therefore $h_-(\theta^+) > h_-(\theta) \ge0,$ a contradiction.
\end{proof}

\begin{proposition}[Discrete logarithmic-tube invariance]
\label{prop:population-gd-tube-invariance}
There exists a constant $\gamma_{\rm tube}>0$, depending only on the fixed population and tube constants, such that for every $0<\gamma\le\gamma_{\rm tube}$, if $\theta_s^\gamma\in\mathcal G$, then $\theta_{s+1}^\gamma\in\mathcal G$ and
\begin{equation}
\label{eq:population-gd-mu-increment}
    \mu_{s+1}^\gamma-\mu_s^\gamma
    \ge
    \frac{c_\mu\gamma}{(\mu_s^\gamma)^3}.
\end{equation}
Moreover, after increasing a fixed constant $C_{\mu,\rm gd}$,
\begin{equation}
\label{eq:population-gd-mu-upper-increment}
    (\mu_{s+1}^\gamma)^4-(\mu_s^\gamma)^4
    \le
    C_{\mu,\rm gd}\gamma.
\end{equation}
\end{proposition}
\begin{proof}
Apply Lemma~\ref{lem:discrete-no-overshoot} with $w=v$. The radial bounds follow from Proposition~\ref{prop:radial-growth-progress-region}; the lower- and upper-rail margins follow from Proposition~\ref{prop:log-lower-rail-invariance} and Corollary~\ref{cor:upper-moving-boundary-correction}; the speed bound is Lemma~\ref{lem:population-speed-progress-region}; and normalized directional regularity follows from Lemmas~\ref{lem:normalized-directional-regularity} and \ref{lem:population-vector-field-global-regularity}.

The radial update is exact: $\mu_{s+1}^\gamma-\mu_s^\gamma = \gamma v_\mu(\theta_s^\gamma),$ which proves \eqref{eq:population-gd-mu-increment}. For the upper fourth-power increment, write
\begin{equation*}
    \delta_s
    :=
    \gamma v_\mu(\theta_s^\gamma),
    \qquad
    0\le\delta_s\le\frac{C_\mu\gamma}{(\mu_s^\gamma)^3}.
\end{equation*}
Then
\begin{align*}
    (\mu_s^\gamma+\delta_s)^4-(\mu_s^\gamma)^4
    &=
    4(\mu_s^\gamma)^3\delta_s
    +
    6(\mu_s^\gamma)^2\delta_s^2
    +
    4\mu_s^\gamma\delta_s^3
    +
    \delta_s^4\\
    &\le
    4C_\mu\gamma+C\gamma^2,
\end{align*}
where the last constant is uniform because $\mu_s^\gamma\ge\mu_*$. For $\gamma\le\gamma_0$, absorb $C\gamma^2$ into $C_{\mu,\rm gd}\gamma$.
\end{proof}

\subsection{Population-GD reference trajectory}
\begin{proposition}[Population-GD reference trajectory]
\label{prop:population-gd-reference}
Under Assumption~\ref{ass:initialization}, there exists a deterministic $\gamma_0>0$, depending only on the fixed population and initialization constants, such that for every $0<\gamma\le\gamma_0$ the population GD iterates \eqref{eq:population-gd} enter $\mathcal G$ after finitely many iterations. More precisely, $\theta_s^\gamma\in\mathcal G$ for every $s\ge s_0(\gamma)$, and
\begin{equation}
\label{eq:population-gd-growth-theorem}
    (\mu_{s_0}^\gamma)^4+4c_\mu\gamma(s-s_0)
    \le (\mu_s^\gamma)^4
    \le (\mu_{s_0}^\gamma)^4+C_{\mu,\rm gd}\gamma(s-s_0).
\end{equation}
Consequently,
\begin{equation}
\label{eq:population-gd-prediction}
    \sup_{\alpha\in[\rho,\ell]}
    |F(\theta_s^\gamma;\alpha)-\alpha|
    \le
    \frac{C_{\rm bias}}
    {\left\{
        (\mu_{s_0}^\gamma)^4
        +
        4c_\mu\gamma(s-s_0)
    \right\}^{1/4}},
\end{equation}
Thus the threshold-estimation error is $O((\gamma s)^{-1/4})$ and the risk is $O((\gamma s)^{-1/2})$.
\end{proposition}
\begin{proof}
Finite-horizon shadowing gives $\theta_{s_0}^\gamma\in\operatorname{int}\mathcal G$ for every sufficiently small $\gamma$. Proposition~\ref{prop:population-gd-tube-invariance} then gives tube invariance for all subsequent iterations. By convexity of $x^4$ and \eqref{eq:population-gd-mu-increment},
\begin{equation*}
    (\mu_{s+1}^\gamma)^4
    \ge
    (\mu_s^\gamma)^4
    +
    4(\mu_s^\gamma)^3
    (\mu_{s+1}^\gamma-\mu_s^\gamma)
    \ge
    (\mu_s^\gamma)^4+4c_\mu\gamma.
\end{equation*}
The upper fourth-power increment in Proposition~\ref{prop:population-gd-tube-invariance} iterates in the same way, so together the two estimates prove \eqref{eq:population-gd-growth-theorem}. The prediction bound follows from Lemma~\ref{lem:log-prediction-bound}, and the risk bound follows by squaring the uniform prediction bound and averaging over $\alpha$.
\end{proof}

\section{Empirical-Process Estimates}
\label{app:empirical-technical-estimates-auxiliary}

This appendix proves the two empirical-process inputs used in
Section~\ref{sec:empirical-pretraining}: compact approximation before
capture and scale-adapted approximation inside $\mathcal G$.
All constants $C,c>0$ may change from line to line. Compact constants
are universal; logarithmic-tube constants may additionally depend on
$\rho$, $c_0$, $C_{\rm band}$, and the fixed population thresholds.

For an i.i.d.\ sample $Z_1,\ldots,Z_N$ with law $P$, write
$P_Nf:=N^{-1}\sum_{i=1}^N f(Z_i)$ and $Pf:=\E f(Z)$. For $R\ge1$, write $K_R:=[0,R]^2.$

For one task block $Z=(\alpha,\mathcal D)$, define
\begin{equation}
\label{eq:one-task-gradient-notation}
    g(\theta;\alpha)
    :=
    \{F(\theta;\alpha)-\alpha\}\nabla F(\theta;\alpha),
    \qquad
    g_n(\theta;Z)
    :=
    \{F_n(\theta;\mathcal D)-\alpha\}\nabla F_n(\theta;\mathcal D).
\end{equation}
Retain $v=-\nabla\mathcal R_\rho$ and write
\begin{equation}
\label{eq:three-vector-fields}
    v_n(\theta;Z):=-g_n(\theta;Z),
    \qquad
    \bar v_n(\theta)
    :=
    \E_Zv_n(\theta;Z)
    =
    -\nabla\mathcal R_{\rho,n}(\theta),
\end{equation}
\[
    \widehat v_{m,n}(\theta)
    :=
    \frac1m\sum_{j=1}^m v_n(\theta;Z_j)
    =
    -\nabla\widehat{\mathcal R}_{m,n}(\theta).
\]

For $R\ge\mu_*$, recall the core shell and define its statistical
enlargement by
\begin{equation}
\label{eq:empirical-core-shell}
    \mathcal S_R
    :=
    \mathcal G\cap\{R\le\mu\le2R\},
\end{equation}
\begin{equation}
\label{eq:empirical-enlarged-shell}
    \widetilde{\mathcal S}_R
    :=
    \mathcal G\cap
    \left\{\frac R2\le\mu\le4R\right\}.
\end{equation}
Also write
$\partial_\pm\widetilde{\mathcal S}_R
:=\{\theta\in\widetilde{\mathcal S}_R:h_\pm(\theta)=0\}$.

Define the normalized empirical shell errors
\begin{align}
\label{eq:empirical-gd-shell-errors}
    \mathcal E_\mu(R)
    &:={}
    \sup_{\theta\in\widetilde{\mathcal S}_R}
    \mu^3
    |[\widehat v_{m,n}(\theta)]_\mu-v_\mu(\theta)|,\notag
    \\
    \mathcal E_-(R)
    &:={}
    \sup_{\theta\in\partial_-\widetilde{\mathcal S}_R}
    \mu^4
    |\langle\nabla h_-(\theta),
    \widehat v_{m,n}(\theta)-v(\theta)\rangle|,\notag
    \\
    \mathcal E_+(R)
    &:={}
    \sup_{\theta\in\partial_+\widetilde{\mathcal S}_R}
    \mu^4
    |\langle\nabla h_+(\theta),
    \widehat v_{m,n}(\theta)-v(\theta)\rangle|,
    \\
    \mathcal E_{\rm sp}(R)
    &:={}
    \sup_{\theta\in\widetilde{\mathcal S}_R}
    \mu^3
    \|\widehat v_{m,n}(\theta)-v(\theta)\|,\notag
    \\
    \mathcal E(R)
    &:={}
    \max\{
        \mathcal E_\mu(R),
        \mathcal E_-(R),
        \mathcal E_+(R),
        \mathcal E_{\rm sp}(R)
    \}.\notag
\end{align}

\subsection{Basic empirical-process tools}
\label{appsubsec:basic-empirical-tools}
\begin{lemma}[One-task gradient regularity]
\label{lem:task-gradient-regularity}
There exist universal constants $B_G,L_G<\infty$ such that, for every
$n$, every task block $Z=(\alpha,\mathcal D)$, and every
$\theta,\vartheta\in\R^2$,
\[
    \|g_n(\theta;Z)\|\le B_G,
    \qquad
    \|g_n(\theta;Z)-g_n(\vartheta;Z)\|
    \le L_G\|\theta-\vartheta\|.
\]
\end{lemma}
\begin{proof}
This is the empirical analogue of
Lemma~\ref{lem:population-vector-field-global-regularity}.
Under the empirical softmax law, boundedness of $X,Y\in[0,1]$
gives uniform bounds on $\nabla F_n$ and $\nabla^2F_n$.
Since also $F_n,\alpha\in[0,1]$, both $\|g_n\|$ and
$\|Dg_n\|_{\rm op}$ are uniformly bounded. The mean-value inequality
gives the claim.
\end{proof}

\begin{lemma}[Dudley bound for bounded Lipschitz parametric classes]
\label{lem:dudley-bounded-lipschitz-class}
Let $\Theta\subset\R^d$ have Euclidean diameter at most $D_\Theta$, and suppose
$|f_\theta(z)|\le B, \ |f_\theta(z)-f_\vartheta(z)| \le L\|\theta-\vartheta\|$
for every $z$ and every $\theta,\vartheta\in\Theta$. Then
\begin{equation}
\label{eq:dudley-parametric-expectation}
    \E\sup_{\theta\in\Theta}|(P_N-P)f_\theta|
    \le
    \frac{CB}{\sqrt N}
    \sqrt{
        d\log\!\left(
            C\left(1+\frac{LD_\Theta}{B}\right)
        \right)
    }.
\end{equation}
Moreover, for every $\delta\in(0,1)$, with probability at least $1-\delta$,
\begin{equation}
\label{eq:dudley-parametric-high-probability}
    \sup_{\theta\in\Theta}|(P_N-P)f_\theta|
    \le
    CB\sqrt{
        \frac{
            d\log\!\left(
                C\left(1+\frac{LD_\Theta}{B}\right)
            \right)
            +\log(C/\delta)
        }{N}
    }.
\end{equation}
\end{lemma}

\begin{proof}
By symmetrization,
\[
    \E\sup_{\theta\in\Theta}|(P_N-P)f_\theta|
    \le
    2\E\E_\varepsilon
    \sup_{\theta\in\Theta}
    \left|
        \frac1N\sum_{i=1}^N
        \varepsilon_i f_\theta(Z_i)
    \right|.
\]
Conditional on the sample, the Rademacher process is sub-Gaussian with
respect to
\[
    d_N(\theta,\vartheta)
    :=
    \left\{
        \frac1N\sum_{i=1}^N
        \bigl(f_\theta(Z_i)-f_\vartheta(Z_i)\bigr)^2
    \right\}^{1/2}.
\]
By the Lipschitz assumption,
$d_N(\theta,\vartheta)\le L\|\theta-\vartheta\|$, while
$\operatorname{diam}(\Theta,d_N)\le2B$. Hence
\[
    N(\epsilon,\Theta,d_N)
    \le
    \left(
        1+\frac{CLD_\Theta}{\epsilon}
    \right)^d .
\]
Dudley's entropy integral therefore gives
\begin{align*}
    \E_\varepsilon
    \sup_{\theta\in\Theta}
    \left|
        \frac1N\sum_{i=1}^N
        \varepsilon_i f_\theta(Z_i)
    \right|
    &\le
    \frac{C}{\sqrt N}
    \left[
        B+
        \int_0^{2B}
        \sqrt{
            d\log\!\left(
                1+\frac{CLD_\Theta}{\epsilon}
            \right)
        }\,d\epsilon
    \right] \\
    &\le
    \frac{CB}{\sqrt N}
    \sqrt{
        d\log\!\left(
            C\left(1+\frac{LD_\Theta}{B}\right)
        \right)
    }.
\end{align*}
Together with symmetrization, this proves
\eqref{eq:dudley-parametric-expectation}.

For the high-probability bound, define $\Phi(Z_1,\ldots,Z_N):=\sup_{\theta\in\Theta}|(P_N-P)f_\theta|.$ If only the $i$th observation is replaced by $Z_i'$, then
\begin{align*}
    \bigl|
        \Phi(Z_1,\ldots,Z_i,\ldots,Z_N)
        -
        \Phi(Z_1,\ldots,Z_i',\ldots,Z_N)
    \bigr|
    &\le
    \sup_{\theta\in\Theta}
    \frac1N
    |f_\theta(Z_i)-f_\theta(Z_i')|
    \le
    \frac{2B}{N}.
\end{align*}
Thus McDiarmid's inequality, with
$c_i=2B/N$, gives for every $t>0$,
\[
    \Prob\!\left(
        \Phi-\E\Phi\ge t
    \right)
    \le
    \exp\!\left(
        -\frac{2t^2}
        {\sum_{i=1}^N c_i^2}
    \right)
    =
    \exp\!\left(
        -\frac{Nt^2}{2B^2}
    \right).
\]
Taking $t=B\sqrt{{2\log(1/\delta)}/{N}}$ shows that, with probability at least $1-\delta$,
\[
    \Phi
    \le
    \E\Phi
    +
    B\sqrt{\frac{2\log(1/\delta)}{N}}.
\]
Combining this with
\eqref{eq:dudley-parametric-expectation} and enlarging the universal
constant $C$ yields \eqref{eq:dudley-parametric-high-probability}.
\end{proof}

\begin{lemma}[Moments of bounded empirical averages]
\label{lem:bounded-average-moments}
Let $Z_1,\ldots,Z_N$ be i.i.d. with $\E Z_1=0$, $|Z_1|\le b$, and $\E Z_1^2\le v$. For $\bar Z_N:=N^{-1}\sum_{i=1}^N Z_i$, one has $\E \bar Z_N^2\le v/N$, $\E \bar Z_N^4\le C\{v^2/N^2+b^2v/N^3\}$, and $\E \bar Z_N^8\le C\{v^4/N^4+b^6v/N^7\}$.
\end{lemma}
\begin{proof}
The second-moment bound is immediate. For $r\in\{4,8\}$, Rosenthal's inequality gives $\E|\sum_{i=1}^N Z_i|^r\le C_r\{(Nv)^{r/2}+N\E|Z_1|^r\}$. Since $\E|Z_1|^r\le b^{r-2}v$, division by $N^r$ gives the stated fourth- and eighth-moment bounds.
\end{proof}

\begin{lemma}[Self-normalized moment and bias bounds]
\label{lem:self-normalized-ratio}
Let $A>0$ almost surely and let $\phi$ satisfy
$PA=1, \ P\phi=0, \ \|A\|_\infty\le B, \ P(A-1)^2\le\sigma_A^2,$ and $\|\phi\|_\infty\le b, \ P\phi^2\le s_2, \ |\phi|\le dA \quad\text{a.s.}$
Set $e_N:=\frac{P_N\phi}{P_NA}.$
There are universal constants $C,c>0$ such that, whenever $N\ge C(B+\sigma_A^2)$,
\begin{align}
\label{eq:self-normalized-second}
    \E e_N^2
    &\le
    C\frac{s_2}{N}
    +
    Cd^2e^{-cN/(B+\sigma_A^2)},
    \\
    \E e_N^4
    &\le
    C\left(
        \frac{s_2^2}{N^2}
        +
        \frac{b^2s_2}{N^3}
    \right)
    +
    Cd^4e^{-cN/(B+\sigma_A^2)},
    \notag
    \\
\label{eq:self-normalized-eighth}
    \E e_N^8
    &\le
    C\left(
        \frac{s_2^4}{N^4}
        +
        \frac{b^6s_2}{N^7}
    \right)
    +
    Cd^8e^{-cN/(B+\sigma_A^2)},
    \\
\label{eq:self-normalized-bias}
    |\E e_N|
    &\le
    C\frac{\sqrt{s_2\sigma_A^2}}{N}
    +
    C(b+d)e^{-cN/(B+\sigma_A^2)}.
\end{align}
\end{lemma}

\begin{proof}
Write $U:=P_N\phi$, $T:=P_NA-1$, and $\mathcal H:=\{|T|\le1/2\}$. Bernstein's inequality gives $\Prob(\mathcal H^c)\le2e^{-cN/(B+\sigma_A^2)}$. On $\mathcal H$, $|e_N|\le2|U|$, while globally $|e_N|\le d$. Lemma~\ref{lem:bounded-average-moments}, applied to $Z_i=\phi_i$ with $v=s_2$, gives the required second, fourth, and eighth moments of $U$. Splitting over $\mathcal H$ and $\mathcal H^c$ proves \eqref{eq:self-normalized-second}--\eqref{eq:self-normalized-eighth}.

For the bias, on $\mathcal H$, $\frac{U}{1+T} = U-\frac{UT}{1+T}.$ Since $\E U=0$, $|\E[U\mathbf 1_{\mathcal H}]| \le b\Prob(\mathcal H^c),$ while
\begin{equation*}
    \E\left[
        \frac{|UT|}{1+T}\mathbf 1_{\mathcal H}
    \right]
    \le
    2\sqrt{\E U^2\E T^2}
    \le
    2\frac{\sqrt{s_2\sigma_A^2}}{N}.
\end{equation*}
The contribution of $\mathcal H^c$ is at most $d\Prob(\mathcal H^c)$, proving \eqref{eq:self-normalized-bias}.
\end{proof}

\begin{lemma}[Self-normalized high-probability deviation]
\label{lem:self-normalized-ratio-high-probability}
Let $A>0$ almost surely and let $\phi$ satisfy
\[
    PA=1,
    \qquad
    P\phi=0,
    \qquad
    \|A\|_\infty\le B,
    \qquad
    P(A-1)^2\le\sigma_A^2,
\]
and
\[
    \|\phi\|_\infty\le b,
    \qquad
    P\phi^2\le s_2.
\]
For an i.i.d.\ sample of size $N$, define
\[
    e_N:=\frac{P_N\phi}{P_NA}.
\]
There exist universal constants $C,c>0$ such that, for every
\[
    1\le u
    \le
    c\,\frac{N}{B+\sigma_A^2+1},
\]
with probability at least $1-4e^{-u}$,
\begin{equation}
\label{eq:self-normalized-high-probability-bound}
    |e_N|
    \le
    C\left\{
        \sqrt{\frac{s_2u}{N}}
        +
        \frac{bu}{N}
    \right\}.
\end{equation}
\end{lemma}

\begin{proof}
Write
\[
    U:=P_N\phi,
    \qquad
    T:=P_NA-1.
\]
Since $P(A-1)=0$, $P(A-1)^2\le\sigma_A^2$, and
$|A-1|\le B+1$, Bernstein's inequality gives
\[
    \Prob\left(|T|>\frac12\right)
    \le
    2\exp\left\{
        -c\frac{N}{B+\sigma_A^2+1}
    \right\}.
\]
Hence, for
$u\le cN/(B+\sigma_A^2+1)$, after adjusting the universal constant
$c$ if necessary,
\begin{equation}
\label{eq:self-normalized-denominator-event}
    \Prob\left(|T|>\frac12\right)
    \le
    2e^{-u}.
\end{equation}

Since $P\phi=0$, $P\phi^2\le s_2$, and $|\phi|\le b$, Bernstein's
inequality also gives, for every $u\ge1$,
\begin{equation}
\label{eq:self-normalized-numerator-event}
    \Prob\left(
        |U|
        >
        C\left\{
            \sqrt{\frac{s_2u}{N}}
            +
            \frac{bu}{N}
        \right\}
    \right)
    \le
    2e^{-u}.
\end{equation}
On the event $|T|\le1/2$,
$P_NA=1+T\ge1/2$, and therefore
$|e_N|\le2|U|$.
Combining
\eqref{eq:self-normalized-denominator-event} and
\eqref{eq:self-normalized-numerator-event} by a union bound proves
\eqref{eq:self-normalized-high-probability-bound}.
\end{proof}

\subsection{Compact vector-field approximation}
\label{appsubsec:compact-vector-field-approximation}

The compact-entry argument separates the finite-context bias from the finite-task fluctuation.

\begin{proposition}[Uniform finite-context bias on compact rectangles]
\label{prop:finite-context-bias}
There is a universal constant $C>0$ such that, for every $R\ge1$ and $n\ge2$,
\begin{equation}
\label{eq:compact-finite-context-bias}
    \sup_{\theta\in K_R}
    \|\nabla\mathcal R_{\rho,n}(\theta)-\nabla\mathcal R_\rho(\theta)\|
    \le
    Ce^{2R}
    \sqrt{
        \frac{\log\!\bigl(C(1+R)\bigr)}{n}
    }.
\end{equation}
\end{proposition}

\begin{proof}
Let $\mathcal A_{\rm mom}:=\{(0,0),(1,0),(0,1),(2,0),(1,1)\}$ and
$\psi_{a,b}(x,y):=x^ay^b$. Define
$Q_{a,b}^{(n)}(\theta):=n^{-1}\sum_{i=1}^n
\psi_{a,b}(X_i,Y_i)e^{\lambda Y_i+\mu X_i}$ and
$Q_{a,b}(\theta;\alpha):=
\E[\psi_{a,b}(X,Y)e^{\lambda Y+\mu X}\mid\alpha]$.
For fixed $\alpha$ and $(a,b)\in\mathcal A_{\rm mom}$, the class $f_\theta(x,y) := \psi_{a,b}(x,y)e^{\lambda y+\mu x}, \ \theta\in K_R,$ has envelope $e^{2R}$, parameter-Lipschitz constant at most $\sqrt2e^{2R}$, and parameter diameter $\sqrt2R$. Lemma~\ref{lem:dudley-bounded-lipschitz-class} therefore gives
\begin{equation}
\label{eq:dudley-uniform-weighted-moments}
    \E\left[
        \sup_{\theta\in K_R}
        \max_{(a,b)\in\mathcal A_{\rm mom}}
        |Q_{a,b}^{(n)}(\theta)-Q_{a,b}(\theta;\alpha)|
        \,\middle|\,
        \alpha
    \right]
    \le
    Ce^{2R}
    \sqrt{
        \frac{\log\!\bigl(C(1+R)\bigr)}{n}
    }.
\end{equation}

For $q=(q_{0,0},q_{1,0},q_{0,1},q_{2,0},q_{1,1}), \ q_{0,0}>0,$ define
\begin{equation*}
    \Gamma_\alpha(q)
    :=
    \left(\frac{q_{1,0}}{q_{0,0}}-\alpha\right)
    \begin{pmatrix}
        \displaystyle
        \frac{q_{1,1}}{q_{0,0}}
        -
        \frac{q_{1,0}q_{0,1}}{q_{0,0}^2}
        \\[1.2ex]
        \displaystyle
        \frac{q_{2,0}}{q_{0,0}}
        -
        \frac{q_{1,0}^2}{q_{0,0}^2}
    \end{pmatrix}.
\end{equation*}
Then
$g_n(\theta;Z)=\Gamma_\alpha((Q_{a,b}^{(n)}(\theta))_{(a,b)\in\mathcal A_{\rm mom}})$
and
$g(\theta;\alpha)=\Gamma_\alpha((Q_{a,b}(\theta;\alpha))_{(a,b)\in\mathcal A_{\rm mom}})$.

Because the raw weight is at least one on $K_R$, the relevant moment vectors satisfy $q_{0,0}\ge1, 0\le q_{a,b}\le q_{0,0},$ and $\Gamma_\alpha$ is uniformly Lipschitz on this set. Hence
\begin{equation}
\label{eq:compact-expected-gradient-deviation}
    \E_{\mathcal D\mid\alpha}
    \left[
        \sup_{\theta\in K_R}
        \|g_n(\theta;Z)-g(\theta;\alpha)\|
    \right]
    \le
    Ce^{2R}
    \sqrt{
        \frac{\log\!\bigl(C(1+R)\bigr)}{n}
    }.
\end{equation}
Since $\nabla\mathcal R_{\rho,n}(\theta)=
\E_{\alpha,\mathcal D}g_n(\theta;Z)$ and
$\nabla\mathcal R_\rho(\theta)=\E_\alpha g(\theta;\alpha)$,
we obtain
\begin{align*}
    \sup_{\theta\in K_R}
    \|\nabla\mathcal R_{\rho,n}(\theta)-\nabla\mathcal R_\rho(\theta)\|
    &\le
    \E_\alpha\E_{\mathcal D\mid\alpha}
    \left[
        \sup_{\theta\in K_R}
        \|g_n(\theta;Z)-g(\theta;\alpha)\|
    \right],
\end{align*}
and \eqref{eq:compact-finite-context-bias} follows from \eqref{eq:compact-expected-gradient-deviation}.
\end{proof}

\begin{proposition}[Uniform finite-task concentration]
\label{prop:finite-task-concentration}
There is a universal constant $C>0$ such that, for every $R\ge1$, $m\ge2$, and $\delta\in(0,1)$, with probability at least $1-\delta$,
\begin{equation}
\label{eq:dudley-finite-task-concentration}
    \sup_{\theta\in K_R}
    \|\nabla\widehat{\mathcal R}_{m,n}(\theta)-\nabla\mathcal R_{\rho,n}(\theta)\|
    \le
    C\sqrt{
        \frac{\log\!\bigl(C(1+R)/\delta\bigr)}{m}
    }.
\end{equation}
\end{proposition}

\begin{proof}
Write $g_j(\theta):=g_n(\theta;Z_j)$ and
$\bar g_n(\theta):=\E g_j(\theta)=\nabla\mathcal R_{\rho,n}(\theta)$.
By Lemma~\ref{lem:task-gradient-regularity}, each coordinate class $\{Z\mapsto[g_n(\theta;Z)]_r:\theta\in K_R\}, \ r\in\{1,2\},$ is uniformly bounded and Lipschitz in $\theta$. Lemma~\ref{lem:dudley-bounded-lipschitz-class}, followed by a union bound over the two coordinates, proves \eqref{eq:dudley-finite-task-concentration}.
\end{proof}

\begin{proposition}[Dudley compact vector-field approximation]
\label{prop:compact-vector-field-approximation}
Let $K\subset[0,R]^2$ be compact with $R\ge1$. For $m,n\ge2$ and $\delta\in(0,1)$, define
\begin{equation}
\label{eq:dudley-compact-vector-field-error}
    \varepsilon_{m,n}(R,\delta)
    :=
    C\sqrt{
        \frac{\log(C(1+R)/\delta)}{m}
    }
    +
    Ce^{2R}
    \sqrt{
        \frac{\log(C(1+R))}{n}
    }.
\end{equation}
Then, with probability at least $1-\delta$, $\sup_{\theta\in K}\|\widehat v_{m,n}(\theta)-v(\theta)\|\le\varepsilon_{m,n}(R,\delta)$. For fixed $R$, the error is $O_R\{\sqrt{\log(C/\delta)/m}+n^{-1/2}\}$.
\end{proposition}

\begin{proof}[Proof of Proposition~\ref{prop:compact-vector-field-approximation}]
Use the decomposition $\widehat v_{m,n}-v=
-(\nabla\widehat{\mathcal R}_{m,n}-\nabla\mathcal R_{\rho,n})
-(\nabla\mathcal R_{\rho,n}-\nabla\mathcal R_\rho)$.
Since $K\subset K_R$, Propositions~\ref{prop:finite-task-concentration} and \ref{prop:finite-context-bias} give the result.
\end{proof}

\subsection{One-task estimates in the logarithmic tube}
\label{appsubsec:one-task-logarithmic-tube}

For fixed $\theta=(\lambda,\mu)$ and $\alpha\in[\rho,\ell]$, define
$W_\theta(X,Y):=e^{\lambda Y+\mu X}$,
$\mathcal Z_{\theta,\alpha}:=\E[W_\theta(X,Y)\mid\alpha]$, and
$W_{\theta,\alpha}:=W_\theta/\mathcal Z_{\theta,\alpha}$.
Let $P_\alpha$ denote conditional expectation given $\alpha$. In particular, $P_\alpha W_{\theta,\alpha}=1.$
\begin{lemma}[Endpoint weight geometry]
\label{lem:normalized-weight-geometry}
After increasing $\mu_*$ if necessary, the following bounds hold uniformly
over $\theta=(\lambda,\mu)\in\mathcal G$ and
$\alpha\in[\rho,\ell]$. Let
$s:=\mu(\alpha-\rho)$,
$p:=\E_\theta[Y\mid\alpha]$,
$q:=1-p$,
$F:=F(\theta;\alpha)$,
$\sigma_X^2:=\Var_\theta(X\mid\alpha)$, and
$\mathsf C:=\Cov_\theta(X,Y\mid\alpha)$.
Then
\[
    q\lesssim \frac{e^{-s}}{\mu},
    \qquad
    |F-\alpha|\lesssim \frac1\mu,
    \qquad
    \sigma_X^2
    \lesssim
    \frac1{\mu^2}+\frac{e^{-s}}{\mu},
    \qquad
    |\mathsf C|
    \lesssim
    \frac{e^{-s}}{\mu}.
\]
Moreover, the normalized population weight satisfies
\[
    W_{\theta,\alpha}(x,1)
    \lesssim
    \mu e^{-\mu(\alpha-x)},
    \quad 0\le x\le\alpha,
    \qquad
    W_{\theta,\alpha}(x,0)
    \lesssim
    e^{-s}e^{-\mu(1-x)},
    \quad \alpha<x\le1.
\]
In particular,
$P_\alpha W_{\theta,\alpha}^2\lesssim\mu$.
\end{lemma}

\begin{proof}
Write
$I_L:=\int_0^\alpha e^{\mu x}\,dx$ and
$I_R:=\int_\alpha^1e^{\mu x}\,dx$.
Since $\alpha,1-\alpha\ge\rho$,
$I_L\asymp e^{\mu\alpha}/\mu$ and
$I_R\asymp e^\mu/\mu$.
Hence, with left and right tilted masses
$L=e^\lambda I_L$ and $R=I_R$,
\[
    \frac LR
    \asymp
    e^{\mu(k+\alpha-1)}
    \ge
    \mu e^{c_0+s},
\]
because $k\ge b_{c_0}(\mu)$.
Thus $q=R/(L+R)\lesssim e^{-s}/\mu$.

Let $X_L$ and $X_R$ denote draws from the left and right tilted
component laws, respectively. By the explicit truncated-exponential
formulas in Lemma~\ref{lem:mixture-representation},
$\E(\alpha-X_L)^r+\E(1-X_R)^r\lesssim\mu^{-r}$ for every fixed
integer $r\ge1$. Combining these component-moment bounds with the
two-component formulas gives
\[
    |F-\alpha|\lesssim\mu^{-1},
    \qquad
    \sigma_X^2
    \lesssim
    \mu^{-2}+\frac{e^{-s}}{\mu},
    \qquad
    |\mathsf C|
    \lesssim
    \frac{e^{-s}}{\mu}.
\]

Finally, the normalizing constant satisfies
$\mathcal Z_{\theta,\alpha}\gtrsim
e^{\lambda+\mu\alpha}/\mu$.
The two stated weight envelopes follow directly. Squaring them and
integrating over the two intervals gives
$P_\alpha W_{\theta,\alpha}^2\lesssim\mu$.
\end{proof}
\begin{corollary}[Influence-function envelopes]
\label{cor:normalized-influence-geometry}
Under the conditions of Lemma~\ref{lem:normalized-weight-geometry}, define
\[
\begin{aligned}
    \phi_F
    &:=
    W_{\theta,\alpha}(X-F),
    &
    \phi_p
    &:=
    W_{\theta,\alpha}(Y-p),\\
    \phi_{\sigma^2}
    &:=
    W_{\theta,\alpha}\bigl((X-F)^2-\sigma_X^2\bigr),
    &
    \phi_{\mathsf C}
    &:=
    W_{\theta,\alpha}
    \bigl((X-F)(Y-p)-\mathsf C\bigr).
\end{aligned}
\]
Each influence function is $P_\alpha$-centered, and uniformly over
$\theta\in\mathcal G$ and $\alpha\in[\rho,\ell]$,
\[
\begin{array}{c|cc}
 & \|\phi\|_\infty & P_\alpha\phi^2\\
\hline
\phi_F
    & \lesssim 1
    & \lesssim \mu^{-1}\\
\phi_p
    & \lesssim e^{-s}
    & \lesssim e^{-2s}\mu^{-1}\\
\phi_{\sigma^2}
    & \lesssim \mu^{-1}+e^{-s}
    & \lesssim \mu^{-3}+e^{-2s}\mu^{-1}\\
\phi_{\mathsf C}
    & \lesssim e^{-s}
    & \lesssim e^{-2s}\mu^{-1}.
\end{array}
\]
Moreover, $|\phi_\diamond|\le2W_{\theta,\alpha}$ for
$\diamond\in\{F,p,\sigma^2,\mathsf C\}$.
\end{corollary}

\begin{proof}
Centering follows directly from the definitions.
On the left interval set $u=\alpha-X$, and on the right set $v=1-X$.
Lemma~\ref{lem:normalized-weight-geometry} reduces all bounds to
\[
    \sup_{u\ge0}
    \mu e^{-\mu u}
    \left(u+\frac1\mu\right)^r
    \lesssim
    \mu^{-(r-1)}
\]
and
\[
    \mu^2
    \int_0^\infty
    e^{-2\mu u}
    \left(u+\frac1\mu\right)^{2r}\,du
    \lesssim
    \mu^{-(2r-1)}.
\]
Applying these with $r=1$ for $\phi_F$ and $\phi_{\mathsf C}$ and
$r=2$ for $\phi_{\sigma^2}$, together with
$q\lesssim e^{-s}/\mu$ and the population bounds of
Lemma~\ref{lem:normalized-weight-geometry}, gives the table.
The domination $|\phi_\diamond|\le2W_{\theta,\alpha}$ follows from $X,Y,F,p\in[0,1]$.
\end{proof}

For a context $\mathcal D=\{(X_i,Y_i)\}_{i=1}^n$, define
\begin{equation*}
    F_n:=F_n(\theta;\mathcal D),
    \qquad
    p_n:=\E_{\theta,n}Y,
    \qquad
    \sigma_{X,n}^2:=\Var_{\theta,n}(X),
    \qquad
    \mathsf C_n:=\Cov_{\theta,n}(X,Y),
\end{equation*}
and
\begin{equation*}
    e_F:=\frac{P_n\phi_F}{P_nW_{\theta,\alpha}},
    \qquad
    e_p:=\frac{P_n\phi_p}{P_nW_{\theta,\alpha}},
    \qquad
    e_{\sigma^2}:=\frac{P_n\phi_{\sigma^2}}{P_nW_{\theta,\alpha}},
    \qquad
    e_{\mathsf C}:=\frac{P_n\phi_{\mathsf C}}{P_nW_{\theta,\alpha}}.
\end{equation*}
Then
\begin{equation}
\label{eq:dictionary-ratio-identities}
    F_n-F=e_F,
    \qquad
    p_n-p=e_p,
\end{equation}
and
\begin{equation}
\label{eq:dictionary-variance-identities}
    \sigma_{X,n}^2-\sigma_X^2
    =
    e_{\sigma^2}-e_F^2,
    \qquad
    \mathsf C_n-\mathsf C
    =
    e_{\mathsf C}-e_Fe_p.
\end{equation}

\begin{lemma}[One-task gradient moment and bias bounds]
\label{lem:one-task-dictionary}
There exist constants $C,c>0$ such that, uniformly over $\theta=(\lambda,\mu)\in\mathcal G$ whenever $n\ge C\mu$,
\begin{align}
\label{eq:task-gradient-population-L2}
    \E_\alpha\|g(\theta;\alpha)\|^2
    &\le
    \frac{C}{\mu^5},
    \\
\label{eq:task-gradient-context-MSE}
    \E_{\alpha,\mathcal D}
    \|g_n(\theta;Z)-g(\theta;\alpha)\|^2
    &\le
    \frac{C}{\mu^4n}
    +
    Ce^{-cn/\mu},
    \\
\label{eq:task-gradient-context-bias}
    \left\|
        \E_{\alpha,\mathcal D}g_n(\theta;Z)
        -
        \E_\alpha g(\theta;\alpha)
    \right\|
    &\le
    \frac{C}{\mu^2n}
    +
    Ce^{-cn/\mu}.
\end{align}
If $\mathcal H\subseteq\{P_nW_{\theta,\alpha}\ge1/2\}$, then
\begin{equation}
\label{eq:stable-denominator-gradient-moment}
    \E_{\alpha,\mathcal D}
    \left[
        \|g_n(\theta;Z)-g(\theta;\alpha)\|^2
        \mathbf 1_{\mathcal H}
    \right]
    \le
    \frac{C}{\mu^4n}.
\end{equation}
\end{lemma}

\begin{proof}
Apply Lemma~\ref{lem:self-normalized-ratio} with $A=W_{\theta,\alpha}$ and $\phi\in \{\phi_F,\phi_p,\phi_{\sigma^2},\phi_{\mathsf C}\}.$ By Lemma~\ref{lem:normalized-weight-geometry},
$\|A\|_\infty+P_\alpha(A-1)^2\lesssim\mu$, while
Corollary~\ref{cor:normalized-influence-geometry} gives the required
envelope, second-moment, and domination bounds, with $d=2$. Hence
\begin{equation}
    \E[e_F^2\mid\alpha]
    \le
    \frac{C}{\mu n}+Ce^{-cn/\mu},
    \qquad
    \E[e_p^2\mid\alpha]
    \le
    \frac{Ce^{-2s}}{\mu n}+Ce^{-cn/\mu},
\end{equation}
and
\begin{equation*}
    \E[e_{\sigma^2}^2\mid\alpha]
    \le
    \frac{C}{n}
    \left(
        \frac1{\mu^3}
        +
        \frac{e^{-2s}}{\mu}
    \right)
    +
    Ce^{-cn/\mu},
\end{equation*}
$\E[e_{\mathsf C}^2\mid\alpha] \le \frac{Ce^{-2s}}{\mu n} + Ce^{-cn/\mu}.$

The fourth- and eighth-moment bounds in
Lemma~\ref{lem:self-normalized-ratio}, together with
Corollary~\ref{cor:normalized-influence-geometry} and $n\ge C\mu$,
give by Cauchy--Schwarz
\begin{align*}
    \E[(e_Fe_p)^2+(e_Fe_{\mathsf C})^2\mid\alpha]
    &\le
    \frac{Ce^{-2s}}{\mu^2n^2}+Ce^{-cn/\mu},\\
    \E[e_F^2e_{\sigma^2}^2\mid\alpha]
    &\le
    C\left\{
        \frac1{\mu^4n^2}
        +
        \frac{e^{-2s}}{\mu^2n^2}
    \right\}
    +Ce^{-cn/\mu},\\
    \E[e_F^6+(e_F^2e_p)^2\mid\alpha]
    &\le
    C\left\{
        \frac1{\mu^3n^3}
        +
        \frac{e^{-2s}}{\mu^3n^3}
    \right\}
    +Ce^{-cn/\mu}.
\end{align*}
These are the only higher-order products needed below.

Write $B:=F-\alpha.$ Then $g(\theta;\alpha) = B \begin{pmatrix} \mathsf C\\ \sigma_X^2 \end{pmatrix},$ while
\begin{align}
\label{eq:dictionary-gradient-expansion}
    (F_n-\alpha)\mathsf C_n-B\mathsf C
    &=
    \mathsf C\,e_F
    +
    B(\mathsf C_n-\mathsf C)
    +
    e_F(\mathsf C_n-\mathsf C),
    \\
    (F_n-\alpha)\sigma_{X,n}^2-B\sigma_X^2
    &=
    \sigma_X^2e_F
    +
    B(\sigma_{X,n}^2-\sigma_X^2)
    +
    e_F(\sigma_{X,n}^2-\sigma_X^2).
\end{align}
Using Lemma~\ref{lem:normalized-weight-geometry},
the ratio identities, and Cauchy--Schwarz gives
\begin{equation}
    \E_{\mathcal D\mid\alpha}
    \|g_n-g\|^2
    \le
    C\left\{
        \frac1{\mu^5n}
        +
        \frac{e^{-s}}{\mu^4n}
        +
        \frac{e^{-2s}}{\mu^3n}
    \right\}
    +
    Ce^{-cn/\mu}.
\end{equation}
Also, $\|g(\theta;\alpha)\|^2\le
C\{\mu^{-6}+e^{-s}\mu^{-5}+e^{-2s}\mu^{-4}\}$.
Since
\begin{equation}
\label{eq:endpoint-alpha-integration}
    s=\mu(\alpha-\rho),
    \qquad
    \int_\rho^\ell e^{-qs}\,d\alpha
    \le
    \frac{C_q}{\mu},
    \qquad
    q>0,
\end{equation}
integrating over $\alpha$ proves
\eqref{eq:task-gradient-population-L2} and
\eqref{eq:task-gradient-context-MSE}.

For the bias, \eqref{eq:self-normalized-bias}, the identities \eqref{eq:dictionary-ratio-identities}--\eqref{eq:dictionary-variance-identities}, and the expansion \eqref{eq:dictionary-gradient-expansion} give
\begin{equation*}
    \left\|
        \E_{\mathcal D\mid\alpha}[g_n-g]
    \right\|
    \le
    C\left\{
        \frac1{\mu^2n}
        +
        \frac{e^{-s}}{\mu n}
    \right\}
    +
    Ce^{-cn/\mu}.
\end{equation*}
Since \eqref{eq:endpoint-alpha-integration} gives $\int_\rho^\ell e^{-s}\,d\alpha \le \frac C\mu,$ we obtain
$\|\E_{\alpha,\mathcal D}g_n(\theta;Z)-\E_\alpha g(\theta;\alpha)\|
\le C/(\mu^2n)+Ce^{-cn/\mu}$,
which proves \eqref{eq:task-gradient-context-bias}.

Finally, on any
$\mathcal H\subseteq\{P_nW_{\theta,\alpha}\ge1/2\}$,
$|e_\diamond|\le2|P_n\phi_\diamond|$. Repeating the same moment
calculation on $\mathcal H$ removes the denominator-failure term and
gives \eqref{eq:stable-denominator-gradient-moment}.
\end{proof}

\begin{lemma}[Fixed-point control at shell resolution]
\label{lem:fixed-parameter-deviations}
There are constants $C,c>0$ such that, for every $R\ge\mu_*$, $\theta\in\widetilde{\mathcal S}_R$, $\alpha\in[\rho,\ell]$, and $n\ge CR$, conditional on $\alpha$, with probability at least
$1-Ce^{-cn/R},$
one has
\begin{equation}
    P_nW_{\theta,\alpha}\ge\frac34
\end{equation}
and
\begin{equation*}
    |F_n-F|
    +
    |p_n-p|
    +
    |\sigma_{X,n}^2-\sigma_X^2|
    +
    |\mathsf C_n-\mathsf C|
    \le
    \frac{C}{R}.
\end{equation*}
\end{lemma}

\begin{proof}
Since $\mu\asymp R$ on $\widetilde{\mathcal S}_R$,
Lemma~\ref{lem:normalized-weight-geometry} gives
$\|W_{\theta,\alpha}\|_\infty+
P_\alpha(W_{\theta,\alpha}-1)^2\le CR$.
Bernstein's inequality therefore yields
$P_nW_{\theta,\alpha}\ge3/4$ with conditional failure probability
at most $Ce^{-cn/R}$.

Apply
Lemma~\ref{lem:self-normalized-ratio-high-probability} to each influence
function in Corollary~\ref{cor:normalized-influence-geometry}, with
$u=c_1n/R$. Since
$B+\sigma_A^2+1\lesssim R$, choosing $c_1>0$ sufficiently small places
$u$ in the admissible range of the lemma, while $n\ge CR$ ensures
$u\ge1$. Thus, with conditional failure probability at most
$Ce^{-cn/R}$,
\[
    |e_F|+|e_p|+|e_{\sigma^2}|+|e_{\mathsf C}|
    \le
    \frac{C}{R}.
\]
Intersecting this event with the denominator event above and using a
union bound preserves the same failure order. The identities
\eqref{eq:dictionary-ratio-identities}--\eqref{eq:dictionary-variance-identities}
then prove the claim.
\end{proof}

\subsection{Uniform shellwise concentration}
\label{appsubsec:shellwise-concentration}

\begin{lemma}[Uniform good-denominator event on one shell]
\label{lem:uniform-good-event}
There are constants $C,c>0$ such that, for every $R\ge\mu_*$ and $n\ge CR$, there is an event $E_R(Z)$ satisfying
\begin{equation}
\label{eq:good-event-failure}
    \Prob\{E_R(Z)^c\}
    \le
    C\exp\left(
        -\frac{cn}{R}
        +
        C\log(CR)
    \right).
\end{equation}
On $E_R(Z)$,
\begin{equation}
\label{eq:uniform-denominator}
    \inf_{\theta\in\widetilde{\mathcal S}_R}
    P_nW_{\theta,\alpha}
    \ge
    \frac12,
\end{equation}
and
\begin{equation}
\label{eq:uniform-task-gradient-shell}
    \sup_{\theta\in\widetilde{\mathcal S}_R}
    \mu^2\|g_n(\theta;Z)\|
    \le
    C.
\end{equation}
\end{lemma}

\begin{proof}
Fix a sufficiently small universal $c_{\rm net}>0$ and choose an $\eta_R$-net $\mathcal N_R$ of $\widetilde{\mathcal S}_R$ with $\eta_R:=\frac{c_{\rm net}}{R}.$ Since $\widetilde{\mathcal S}_R$ has diameter $O(R)$ in two dimensions, $|\mathcal N_R| \le \left( 1+\frac{CR}{\eta_R} \right)^2 \le CR^4.$ Lemma~\ref{lem:fixed-parameter-deviations} and a union bound over $\mathcal N_R$ give \eqref{eq:good-event-failure} and, at every net point, $P_nW_{\theta,\alpha}\ge\frac34,$ together with $O(R^{-1})$ control of
\begin{equation*}
    F_n-F,
    \qquad
    p_n-p,
    \qquad
    \sigma_{X,n}^2-\sigma_X^2,
    \qquad
    \mathsf C_n-\mathsf C.
\end{equation*}

The maps $F_n,F,p_n,p,\sigma_{X,n}^2,\sigma_X^2,\mathsf C_n,\mathsf C$ have universal parameter-Lipschitz constants. Hence, if $\vartheta\in\mathcal N_R$ is nearest to $\theta\in\widetilde{\mathcal S}_R$, $|H(\theta)-H(\vartheta)| \le C\eta_R = \frac{Cc_{\rm net}}{R}$ for each such map $H$. Thus the mesh $\eta_R=c_{\rm net}/R$ produces exactly the $O(R^{-1})$ extension error required here.

Moreover,
$\nabla_\theta\log P_nW_{\theta,\alpha} = \E_{\theta,n}(Y,X) - \E_\theta(Y,X),$
so $\theta\mapsto\log P_nW_{\theta,\alpha}$ is universally Lipschitz. Choosing $c_{\rm net}$ sufficiently small extends the denominator lower bound from $3/4$ on the net to $1/2$ on the full shell, proving \eqref{eq:uniform-denominator}.

Finally, the uniform $O(R^{-1})$ deviations established above,
together with Lemma~\ref{lem:normalized-weight-geometry} and
$\mu\asymp R$, give
$|F_n-\alpha|\le C/\mu$ and
$\sigma_{X,n}^2+|\mathsf C_n|\le C/\mu$ uniformly over
$\widetilde{\mathcal S}_R$. Hence
$\|g_n(\theta;Z)\|\le C/\mu^2$, proving
\eqref{eq:uniform-task-gradient-shell}.
\end{proof}

For $\diamond\in\{\mu,-,+\}$, define
\begin{align*}
    f_{\mu,\theta}(Z)
    &:=
    \mu^3v_{n,\mu}(\theta;Z),
    &&\theta\in\widetilde{\mathcal S}_R,
    \\
    f_{-,\theta}(Z)
    &:=
    \mu^4
    \langle\nabla h_-(\theta),v_n(\theta;Z)\rangle,
    &&\theta\in\partial_-\widetilde{\mathcal S}_R,
    \\
    f_{+,\theta}(Z)
    &:=
    \mu^4
    \langle\nabla h_+(\theta),v_n(\theta;Z)\rangle,
    &&\theta\in\partial_+\widetilde{\mathcal S}_R.
\end{align*}

\begin{lemma}[Directional envelopes, variance, and regularity]
\label{lem:directional-regularity}
There is a constant $C>0$ such that, for every $R\ge\mu_*$ and $n\ge CR$, if
$\widetilde f_{\diamond,\theta}(Z) := f_{\diamond,\theta}(Z)\mathbf 1_{E_R(Z)},$
then, on the relevant parameter domain for each $\diamond\in\{\mu,-,+\}$,
\begin{align}
\label{eq:directional-global-envelope}
    \sup_{\theta,Z}|f_{\diamond,\theta}(Z)|
    &\le
    CR^3,
    \\
\label{eq:directional-truncated-envelope}
    \sup_{\theta,Z}|\widetilde f_{\diamond,\theta}(Z)|
    &\le
    CR,
    \\
\label{eq:directional-second-moment}
    \sup_\theta
    \E\widetilde f_{\diamond,\theta}(Z)^2
    &\le
    C\left(
        R+\frac{R^2}{n}
    \right),
    \\
\label{eq:directional-Lipschitz}
    |f_{\diamond,\theta}(Z)-f_{\diamond,\vartheta}(Z)|
    &\le
    CR^3\|\theta-\vartheta\|.
\end{align}
\end{lemma}

\begin{proof}
Lemma~\ref{lem:task-gradient-regularity} gives $\|g_n\|\le C$ globally. Since $\mu\asymp R$ on $\widetilde{\mathcal S}_R$ and $\mu^4\|\nabla h_\pm(\theta)\| \le C\mu^3,$ the global envelope is $CR^3$. On $E_R$, Lemma~\ref{lem:uniform-good-event} improves $\|g_n\|$ to $C\mu^{-2}$, giving the truncated envelope $CR$.

For fixed $\theta$, $g_n = g+(g_n-g).$ By \eqref{eq:task-gradient-population-L2} and \eqref{eq:stable-denominator-gradient-moment},
$\E_\alpha\|g\|^2\le C\mu^{-5}$ and
$\E_{\alpha,\mathcal D}[\|g_n-g\|^2\mathbf 1_{E_R}]
\le C/(\mu^4n)$.
Each directional scaling has magnitude at most $C\mu^3$, so
$\E\widetilde f_{\diamond,\theta}^2
\le C\mu^6\{\mu^{-5}+(\mu^4n)^{-1}\}
\le C(R+R^2/n)$.

Finally, $\|\nabla(\mu^3)\|\le CR^2,$ and Lemma~\ref{lem:gd-barrier-regularity} gives
$\|D(\mu^4\nabla h_\pm(\theta))\|_{\rm op}\le CR^2$ and
$\mu^4\|\nabla h_\pm(\theta)\|\le CR^3$.
Together with the global bounds on $g_n$ and $Dg_n$, these imply \eqref{eq:directional-Lipschitz}. Since $E_R(Z)$ does not depend on the particular $\theta$ inside the enlarged shell, the same bound holds for the truncated classes.
\end{proof}

\begin{lemma}[Finite-dimensional Lipschitz Bernstein bound]
\label{lem:Lipschitz-Bernstein}
Let $K\subset\R^d$ have diameter at most $D_K$, and suppose
\begin{equation*}
    |f_\theta(Z)|\le b,
    \qquad
    \E f_\theta(Z)^2\le s_2,
    \qquad
    |f_\theta(Z)-f_\vartheta(Z)|
    \le
    L_K\|\theta-\vartheta\|.
\end{equation*}
Then, for every $\eta>0$ and $\delta\in(0,1)$, with probability at least $1-\delta$,
\begin{equation*}
    \sup_{\theta\in K}|(P_m-P)f_\theta|
    \le
    C\left\{
        \sqrt{\frac{s_2L_\eta(\delta)}{m}}
        +
        \frac{bL_\eta(\delta)}{m}
        +
        L_K\eta
    \right\},
\end{equation*}
where
$L_\eta(\delta) := d\log\left(1+\frac{CD_K}{\eta}\right) + \log\frac{C}{\delta}.$
\end{lemma}

\begin{proof}
Take an $\eta$-net of $K$ with cardinality at most $\left(1+\frac{CD_K}{\eta}\right)^d.$ Bernstein's inequality and a union bound give the first two terms on the net. If $\vartheta$ is nearest to $\theta$, then
$|(P_m-P)(f_\theta-f_\vartheta)|
\le P_m|f_\theta-f_\vartheta|+P|f_\theta-f_\vartheta|
\le2L_K\eta$.
Absorb the factor two into $C$.
\end{proof}

\begin{theorem}[Shellwise directional approximation]
\label{thm:shellwise-directional-approximation}
There is a fixed constant $C_{\rm res}\ge1$ such that the following holds. For $R\ge\mu_*$ and $\delta\in(0,1)$, set
$L_R(\delta) := \log\left( \frac{C_{\rm res}(1+R)mn}{\delta} \right).$
If $n\ge C_{\rm res}RL_R(\delta),$ then, with probability at least $1-\delta$,
\begin{equation}
\label{eq:shellwise-directional-bound}
    \max\{
        \mathcal E_\mu(R),
        \mathcal E_-(R),
        \mathcal E_+(R)
    \}
    \le
    C_{\rm res}\left\{
        \sqrt{\frac{RL_R(\delta)}{m}}
        +
        \frac{RL_R(\delta)}{m}
        +
        \frac{R}{n}
    \right\}.
\end{equation}
\end{theorem}

\begin{proof}[Proof of Theorem~\ref{thm:shellwise-directional-approximation}]
Fix $\diamond\in\{\mu,-,+\}$. By Lemma~\ref{lem:directional-regularity}, the truncated class satisfies
\begin{equation*}
    d=2,
    \qquad
    D_K\le CR,
    \qquad
    b=CR,
    \qquad
    s_2=C\left(
        R+\frac{R^2}{n}
    \right),
    \qquad
    L_K=CR^3.
\end{equation*}
The $O(R^{-1})$ net in Lemma~\ref{lem:uniform-good-event} is sufficient there because the quantities being extended have $O(1)$ Lipschitz constants. Here the scaled directional class has $L_K=O(R^3)$, so we choose the finer sample-dependent mesh
\begin{equation}
\label{eq:app-directional-net-mesh}
    \eta
    :=
    \frac1{CmR^3}.
\end{equation}
Then $L_K\eta \le CR^3\frac1{CmR^3} \le \frac C m.$ Moreover,
\begin{align*}
    L_\eta(\delta)
    &=
    2\log\left(
        1+\frac{CD_K}{\eta}
    \right)
    +
    \log\frac C\delta
    \\
    &\le
    C\log(CmR^4)
    +
    \log\frac C\delta
    \\
    &\le
    CL_R(\delta).
\end{align*}
Thus the net-extension contribution is $O(m^{-1})$, while the finer
covering changes only the logarithmic factor.

Under the context condition in Theorem~\ref{thm:shellwise-directional-approximation}, $n\ge CRL_R(\delta),$ so $s_2 \le CR.$ Lemma~\ref{lem:Lipschitz-Bernstein} therefore yields
\begin{align*}
    \sup_\theta
    |(P_m-P)\widetilde f_{\diamond,\theta}|
    &\le
    C\left\{
        \sqrt{\frac{RL_R(\delta)}{m}}
        +
        \frac{RL_R(\delta)}{m}
        +
        \frac1m
    \right\}
    \\
    &\le
    C\left\{
        \sqrt{\frac{RL_R(\delta)}{m}}
        +
        \frac{RL_R(\delta)}{m}
    \right\}.
\end{align*}
Applying this bound to the three classes with failure probability $\delta/9$ each and taking a union bound gives, with probability at least $1-\delta/3$,
\begin{equation}
\label{eq:truncated-centered-shell}
    \max_{\diamond\in\{\mu,-,+\}}
    \sup_\theta
    |(P_m-P)\widetilde f_{\diamond,\theta}|
    \le
    C\left\{
        \sqrt{\frac{RL_R(\delta)}{m}}
        +
        \frac{RL_R(\delta)}{m}
    \right\}.
\end{equation}

Set $p_R^{\rm bad}(n) := C\exp\left( -\frac{cn}{R} + C\log(CR) \right).$ By Lemma~\ref{lem:uniform-good-event}, $\Prob\left( \bigcup_{j=1}^m E_R(Z_j)^c \right) \le mp_R^{\rm bad}(n).$ The context condition ensures, after increasing its numerical constant if necessary,
$mp_R^{\rm bad}(n) \le \delta/3, \ R^3p_R^{\rm bad}(n) \le C R/n.$
On the event that all observed task blocks belong to $E_R$, $P_mf_{\diamond,\theta} = P_m\widetilde f_{\diamond,\theta},$ while \eqref{eq:directional-global-envelope} gives
$|P(f_{\diamond,\theta}-\widetilde f_{\diamond,\theta})|
\le CR^3p_R^{\rm bad}(n)\le CR/n$.
Together with \eqref{eq:truncated-centered-shell}, this controls the finite-task fluctuation $\widehat v_{m,n}-\bar v_n$ in all three directions.

It remains to add the finite-context bias. By \eqref{eq:task-gradient-context-bias},
$\|\bar v_n(\theta)-v(\theta)\| \le C\left( \frac1{\mu^2n} + e^{-cn/\mu} \right).$
Multiplying by $\mu^3$ in the radial direction and by $\mu^4\|\nabla h_\pm(\theta)\| \le C\mu^3$ in the barrier directions gives $C\left( \frac Rn + R^3e^{-cn/R} \right) \le C\frac Rn$ under the same context condition. The exceptional events have total probability at most $2\delta/3<\delta$, proving \eqref{eq:shellwise-directional-bound}.
\end{proof}

\begin{proposition}[Shellwise speed approximation]
\label{prop:shellwise-speed-approximation}
After increasing $C_{\rm res}$ if necessary, under the same assumptions as Theorem~\ref{thm:shellwise-directional-approximation}, with probability at least $1-\delta$,
\begin{equation}
\label{eq:shellwise-speed-bound}
    \mathcal E_{\rm sp}(R)
    \le
    C_{\rm res}\left\{
        \sqrt{\frac{RL_R(\delta)}{m}}
        +
        \frac{RL_R(\delta)}{m}
        +
        \frac{R}{n}
    \right\}.
\end{equation}
\end{proposition}
\begin{proof}
For $r\in\{1,2\}$, define $f_{{\rm sp},r,\theta}(Z):=\mu^3[v_n(\theta;Z)]_r$ and $\widetilde f_{{\rm sp},r,\theta}:= f_{{\rm sp},r,\theta}\mathbf1_{E_R(Z)}$. The argument of Lemma~\ref{lem:directional-regularity} applies verbatim to these coordinate classes: on $\widetilde{\mathcal S}_R$,
\[
    |\widetilde f_{{\rm sp},r,\theta}|
    \le CR,
    \qquad
    \E\widetilde f_{{\rm sp},r,\theta}^2
    \le
    C\left(R+\frac{R^2}{n}\right),
\]
while the untruncated class is $CR^3$-Lipschitz in $\theta$. Applying Lemma~\ref{lem:Lipschitz-Bernstein} with the same mesh $\eta=(CmR^3)^{-1}$ as in
Theorem~\ref{thm:shellwise-directional-approximation} yields
\[
    \sup_{\theta\in\widetilde{\mathcal S}_R}
    |(P_m-P)\widetilde f_{{\rm sp},r,\theta}|
    \le
    C\left\{
        \sqrt{\frac{RL_R(\delta)}m}
        +
        \frac{RL_R(\delta)}m
    \right\}.
\]
The bad-task truncation and finite-context bias are exactly those in the proof of Theorem~\ref{thm:shellwise-directional-approximation} and contribute $CR/n$. A union bound over $r=1,2$ and conversion to the Euclidean norm prove \eqref{eq:shellwise-speed-bound}.
\end{proof}

\section{Proof of the Finite-Pretraining Theorem}
\label{sec:empirical-pretraining}
\label{app:empirical-gd-proofs}

The empirical-process section already defined the one-task gradients,
the vector fields $v$, $\bar v_n$, and $\widehat v_{m,n}$, and the core
and enlarged shells. We now analyze empirical GD,
\begin{equation}
\label{eq:empirical-gd}
    \widehat\theta_{s+1}^{\gamma}
    =
    \widehat\theta_s^{\gamma}
    +
    \gamma\widehat v_{m,n}(\widehat\theta_s^{\gamma}),
    \qquad
    \widehat\theta_0^\gamma=\theta_0^\gamma.
\end{equation}

\subsection{Two-stage perturbation principle}
\label{subsec:two-stage-perturbation}

The proof has two stages, with population GD as the reference dynamics.
Before capture, empirical GD tracks population GD on a fixed compact set
for a fixed physical time using compact Dudley approximation and discrete
Gr\"onwall control. After capture, trajectory tracking is replaced by
scale-adapted shell estimates that preserve the radial drift, the two
barrier margins, and the speed bound required for discrete no-overshoot.

\subsection{Regime I: compact tracking of population GD}
\label{subsec:empirical-gd-compact}
\label{appsubsec:empirical-gd-compact-proof}

Let $s_0=s_0(\gamma)$ be the deterministic population-GD entry index from Section~\ref{sec:finite-step-gd}. Reduce the population learning-rate threshold if necessary and fix $\gamma_0>0$ so that \eqref{eq:population-gd-shadowing-bound} and \eqref{eq:population-gd-entry} hold for every $0<\gamma\le\gamma_0$. Choose one deterministic compact set $K_0$ containing, with positive clearance, every $\theta_s^\gamma$ with $0\le s\le s_0(\gamma)$ and $0<\gamma\le\gamma_0$. Such a common compact set exists by uniform finite-horizon shadowing on $[0,T_0+\tau_0]$. Let $L_0$ be a Lipschitz constant of $v$ on $K_0$ and $L_{h,0}$ a common Lipschitz constant of $h_-$ and $h_+$ there.

Define the uniform compact and entry clearances
\begin{align*}
    d_{K,0}^{\rm gd}
    &:={}
    \inf_{\substack{0<\gamma\le\gamma_0\\0\le s\le s_0(\gamma)}}
    \operatorname{dist}(\theta_s^\gamma,\partial K_0),\\
    d_{\mu,0}^{\rm gd}
    &:={}
    \inf_{0<\gamma\le\gamma_0}
    \{\mu_{s_0(\gamma)}^\gamma-\mu_*\},\\
    d_{-,0}^{\rm gd}
    &:={}
    \inf_{0<\gamma\le\gamma_0}
    h_-(\theta_{s_0(\gamma)}^\gamma),\\
    d_{+,0}^{\rm gd}
    &:={}
    \inf_{0<\gamma\le\gamma_0}
    h_+(\theta_{s_0(\gamma)}^\gamma).
\end{align*}
By the strict post-capture population-flow margin and the uniform Euler shadowing, all four quantities are strictly positive after reducing $\gamma_0$ if necessary. Set
\begin{equation}
\label{eq:empirical-gd-entry-margin}
    d_{\rm ent}^{\rm gd}
    :=
    \min\left\{
        d_{K,0}^{\rm gd},
        d_{\mu,0}^{\rm gd},
        \frac{d_{-,0}^{\rm gd}}{L_{h,0}},
        \frac{d_{+,0}^{\rm gd}}{L_{h,0}}
    \right\}.
\end{equation}

\begin{lemma}[Discrete compact tracking]
\label{lem:empirical-gd-compact-tracking}
Let $K\subset\mathbb R^2$ be compact and suppose the population and empirical GD sequences start from the same point. Assume $v$ is $L$-Lipschitz on $K$ and
\begin{equation*}
    \sup_{\theta\in K}
    \|\widehat v_{m,n}(\theta)-v(\theta)\|
    \le
    \varepsilon.
\end{equation*}
Then, before either sequence leaves $K$,
\begin{equation*}
    \|\widehat\theta_s^\gamma-\theta_s^\gamma\|
    \le
    \varepsilon
    \begin{cases}
    \dfrac{(1+\gamma L)^s-1}{L},&L>0,\\[2mm]
    s\gamma,&L=0.
    \end{cases}
\end{equation*}
If this upper bound is strictly smaller than the clearance of the reference sequence from $\partial K$, then neither sequence exits $K$ over the stated index range.
\end{lemma}

\begin{proof}
Set $e_s:=\|\widehat\theta_s^\gamma-\theta_s^\gamma\|.$ Then, before exit,
\begin{align*}
    e_{s+1}
    &\le
    e_s
    +
    \gamma
    \|\widehat v_{m,n}(\widehat\theta_s^\gamma)-v(\theta_s^\gamma)\|\\
    &\le
    (1+\gamma L)e_s+\gamma\varepsilon.
\end{align*}
Iterating from $e_0=0$ proves the displayed bound. The final statement follows by the standard stopped argument: at a first exit time the empirical iterate would have to be at least the reference clearance away from the corresponding population iterate, contradicting the tracking bound.
\end{proof}

Applying Lemma~\ref{lem:empirical-gd-compact-tracking} on $K_0$, define
\[
    \Delta_{0,\gamma}
    :=
    \begin{cases}
    \dfrac{(1+\gamma L_0)^{s_0}-1}{L_0},&L_0>0,\\[2mm]
    s_0\gamma,&L_0=0.
    \end{cases}
\]
Since $s_0\gamma\le T_0+\gamma$, the deterministic constant
$\Delta_0^{\rm gd}:=\sup_{0<\gamma\le\gamma_0}\Delta_{0,\gamma}$ is finite.
Set
\[
    \varepsilon_{\rm ent}^{\rm gd}
    :=
    \frac{d_{\rm ent}^{\rm gd}}{2\Delta_0^{\rm gd}}.
\]
If
$\sup_{\vartheta\in K_0}\|\widehat v_{m,n}(\vartheta)-v(\vartheta)\|
\le\varepsilon_{\rm ent}^{\rm gd}$,
the stopped argument prevents exit from $K_0$ and yields
$\widehat\theta_{s_0}^\gamma\in\operatorname{int}(\mathcal G)$.

For later uniform radial bounds, define
\begin{equation}
\label{eq:population-gd-entry-radius-bounds}
    \underline\mu_0^{\rm gd}
    :=
    \inf_{0<\gamma\le\gamma_0}
    \mu_{s_0(\gamma)}^\gamma,
    \qquad
    \overline\mu_0^{\rm gd}
    :=
    \sup_{0<\gamma\le\gamma_0}
    \mu_{s_0(\gamma)}^\gamma.
\end{equation}
Then
$\mu_*<\underline\mu_0^{\rm gd} \le \overline\mu_0^{\rm gd}<\infty.$

\subsection{Regime II: discrete continuation with overlapping statistical shells}
\label{subsec:empirical-gd-shells}
\label{appsubsec:empirical-gd-one-shell}
Recall the core shell $\mathcal S_R$, its enlarged statistical shell
$\widetilde{\mathcal S}_R$, and the normalized errors
$\mathcal E_\mu(R)$, $\mathcal E_-(R)$, $\mathcal E_+(R)$, and
$\mathcal E_{\rm sp}(R)$ from Appendix~\ref{app:empirical-technical-estimates-auxiliary}.
The core shell governs the dyadic continuation, whereas concentration is
proved on the enlargement so that the complete landing step and any
potential side-rail crossing point are covered by the same statistical
event.

Applying Theorem~\ref{thm:shellwise-directional-approximation} and
Proposition~\ref{prop:shellwise-speed-approximation} with failure probability
$\delta/2$ each, taking a union bound, and absorbing the resulting constant
inside the logarithm into $C_{\rm res}$ gives, with probability at least
$1-\delta$,
\begin{equation}
\label{eq:speed-shell-rate}
    \max\{
        \mathcal E_\mu(R),
        \mathcal E_-(R),
        \mathcal E_+(R),
        \mathcal E_{\rm sp}(R)
    \}
    \le
    C_{\rm res}\left\{
        \sqrt{\frac{RL_R(\delta)}{m}}
        +
        \frac{RL_R(\delta)}{m}
        +
        \frac{R}{n}
    \right\},
\end{equation}
where $L_R(\delta) := \log\left( \frac{C_{\rm res}(1+R)mn}{\delta} \right).$ Thus the enlarged shell needed for discrete handoff does not change the statistical order of the shell resolution.

Let $a_->0$ and $a_+>0$ denote the population inward-margin constants
from Proposition~\ref{prop:log-lower-rail-invariance} and
Corollary~\ref{cor:upper-moving-boundary-correction}, respectively. Thus,
after increasing $\mu_*$ if necessary,
\[
    \langle\nabla h_\pm(\theta),v(\theta)\rangle
    \ge
    \frac{a_\pm}{\mu^4}
    \qquad\text{on }h_\pm(\theta)=0.
\]
Set
\begin{equation*}
    c_{\rm dir}
    :=
    \frac14\min\{c_\mu,a_-,a_+\},
    \qquad
    \overline c_\mu
    :=
    C_\mu+c_{\rm dir},
\end{equation*}
and fix $C_{\rm sp,0}>0$. On a shell event satisfying
\begin{equation}
\label{eq:empirical-gd-shell-smallness}
    \mathcal E_\mu(R),
    \mathcal E_-(R),
    \mathcal E_+(R)
    \le
    c_{\rm dir},
    \qquad
    \mathcal E_{\rm sp}(R)
    \le
    C_{\rm sp,0},
\end{equation}
the empirical field inherits, throughout the enlarged shell where relevant,
\begin{equation}
\label{eq:empirical-gd-radial-margin}
    \frac{c_\mu}{2\mu^3}
    \le
    [\widehat v_{m,n}(\theta)]_\mu
    \le
    \frac{\overline c_\mu}{\mu^3},
\end{equation}
the side-rail margins
\begin{equation}
\label{eq:empirical-gd-barrier-margin}
    \langle\nabla h_\pm(\theta),
    \widehat v_{m,n}(\theta)\rangle
    \ge
    \frac{a_\pm}{2\mu^4}
    \quad
    \text{on }\partial_\pm\widetilde{\mathcal S}_R,
\end{equation}
and, with
$C_{\rm step}:=C_{\rm sp}+C_{\rm sp,0},$
the speed bound
\begin{equation}
\label{eq:empirical-gd-speed-bound}
    \|\widehat v_{m,n}(\theta)\|
    \le
    \frac{C_{\rm step}}{\mu^3}.
\end{equation}

\begin{proposition}[Discrete empirical continuation through one shell]
\label{prop:empirical-gd-one-shell}
There exists a fixed $\gamma_{\rm shell}>0$ such that the following holds. Suppose the enlarged-shell event \eqref{eq:empirical-gd-shell-smallness} holds at scale $R\ge\mu_*$. If
$0<\gamma\le\gamma_{\rm shell}, \ \widehat\theta_{s_*}^\gamma\in\mathcal S_R,$
and $\tau_R := \inf\{s\ge s_*:\widehat\mu_s^\gamma\ge2R\},$ then every iterate with $s_*\le s\le\tau_R$ belongs to $\mathcal G$. Moreover,
\begin{equation}
\label{eq:empirical-gd-one-shell-growth}
    (\widehat\mu_s^\gamma)^4
    \ge
    (\widehat\mu_{s_*}^\gamma)^4
    +
    2c_\mu\gamma(s-s_*),
    \qquad
    s_*\le s\le\tau_R,
\end{equation}
and for a fixed constant $c_+^{\rm gd}>0$,
\begin{equation}
\label{eq:empirical-gd-one-shell-upper}
    (\widehat\mu_s^\gamma)^4
    \le
    (\widehat\mu_{s_*}^\gamma)^4
    +
    c_+^{\rm gd}\gamma(s-s_*),
    \qquad
    s_*\le s\le\tau_R.
\end{equation}
Moreover, $\tau_R<\infty$; more precisely,
\begin{equation*}
    \tau_R
    \le
    s_*
    +
    \left\lceil
        \frac{(2R)^4-(\widehat\mu_{s_*}^\gamma)^4}
        {2c_\mu\gamma}
    \right\rceil
    +1.
\end{equation*}
The landing iterate satisfies
\begin{equation}
\label{eq:empirical-gd-shell-handoff}
    \widehat\theta_{\tau_R}^\gamma
    \in
    \mathcal S_{2R};
\end{equation}
in particular,
$2R \le \widehat\mu_{\tau_R}^\gamma < 4R.$
\end{proposition}

\begin{proof}
On \eqref{eq:empirical-gd-shell-smallness}, equations \eqref{eq:empirical-gd-radial-margin}--\eqref{eq:empirical-gd-speed-bound} hold throughout the enlarged shell. For fixed pretraining data, Lemma~\ref{lem:task-gradient-regularity} makes $\widehat v_{m,n}$ globally bounded and globally Lipschitz with universal constants. Hence Lemma~\ref{lem:normalized-directional-regularity} applies on the two-sided enlargement $\mathcal U_R$ from Lemma~\ref{lem:discrete-no-overshoot}. Every side-rail point in $\mathcal U_R$ has $R\le\mu\le4R$ and lies in $\partial_\pm\widetilde{\mathcal S}_R$, so the empirical inward margins are available on the entire rail portion required by the no-overshoot argument.

Consider an iterate before $\tau_R$. It lies radially in the core range $R\le\widehat\mu_s^\gamma<2R$. The speed bound gives a step length $O(\gamma R^{-3})$. After reducing $\gamma_{\rm shell}$, the entire segment joining $\widehat\theta_s^\gamma$ to $\widehat\theta_{s+1}^\gamma$ has radial coordinate in $[R/2,4R]$. If the segment attempted to cross a side rail, the crossing point would therefore lie in $\partial_\pm\widetilde{\mathcal S}_R$, where the empirical inward margin \eqref{eq:empirical-gd-barrier-margin} is available. The no-overshoot Lemma~\ref{lem:discrete-no-overshoot} then rules out such a crossing. Thus tube membership is preserved, including on the landing step.

The radial update satisfies
\begin{equation*}
    \widehat\mu_{s+1}^\gamma-\widehat\mu_s^\gamma
    =
    \gamma[\widehat v_{m,n}(\widehat\theta_s^\gamma)]_\mu
    \ge
    \frac{c_\mu\gamma}{2(\widehat\mu_s^\gamma)^3}.
\end{equation*}
Convexity of $x^4$ gives $(\widehat\mu_{s+1}^\gamma)^4 -(\widehat\mu_s^\gamma)^4 \ge 2c_\mu\gamma,$ which proves \eqref{eq:empirical-gd-one-shell-growth}. Similarly, from
\begin{equation*}
    0\le
    \widehat\mu_{s+1}^\gamma-\widehat\mu_s^\gamma
    \le
    \frac{\overline c_\mu\gamma}{(\widehat\mu_s^\gamma)^3},
\end{equation*}
expanding the fourth power and using $\widehat\mu_s^\gamma\ge\mu_*$ yields \eqref{eq:empirical-gd-one-shell-upper} for a fixed $c_+^{\rm gd}$.

If $\tau_R$ were infinite, \eqref{eq:empirical-gd-one-shell-growth} would hold for every $s\ge s_*$ and would eventually force $(\widehat\mu_s^\gamma)^4>(2R)^4$, contradicting the definition of $\tau_R$. Solving \eqref{eq:empirical-gd-one-shell-growth} for the first index at which $\widehat\mu_s^\gamma\ge2R$ gives the stated upper bound on $\tau_R$.

If $\tau_R=s_*$, then $\widehat\mu_{s_*}^\gamma=2R$ because the starting point lies in $\mathcal S_R$, and the handoff conclusion is immediate. Assume therefore that $\tau_R>s_*$. At the landing time, $\widehat\mu_{\tau_R-1}^\gamma<2R, \ \widehat\mu_{\tau_R}^\gamma\ge2R.$ The upper speed bound gives $\widehat\mu_{\tau_R}^\gamma < 2R+\frac{C\gamma}{R^3} < 4R$ for sufficiently small fixed $\gamma$. The preceding no-overshoot argument also gives $\widehat\theta_{\tau_R}^\gamma\in\mathcal G$. Therefore
\begin{equation*}
    \widehat\theta_{\tau_R}^\gamma
    \in
    \mathcal G\cap\{2R\le\mu<4R\}
    \subseteq
    \mathcal S_{2R},
\end{equation*}
which proves \eqref{eq:empirical-gd-shell-handoff}.
\end{proof}

\subsection{Statistical radius and a sufficient pretraining regime}
\label{subsec:one-boundary-empirical-scales}
\label{appsubsec:empirical-gd-resolution}

Choose $R_{\rm cmp}\ge1$ with $K_0\subset[0,R_{\rm cmp}]^2$. Fix $C_{\rm cmp}\ge1$ sufficiently large and $c_{\rm cmp}>0$ sufficiently small, and define
\begin{equation}
\label{eq:main-entry-failure-probability}
    p_{\rm ent}(m,n)
    :=
    C_{\rm cmp}
    \exp\{-c_{\rm cmp}m(\varepsilon_{\rm ent}^{\rm gd})^2\}.
\end{equation}
Set
\begin{equation*}
    \varepsilon_{\rm cmp}(m,n)
    :=
    \varepsilon_{m,n}(R_{\rm cmp},p_{\rm ent}(m,n)),
    \qquad
    e_{\rm cmp}^{\rm gd}(m,n)
    :=
    \Delta_0^{\rm gd}\varepsilon_{\rm cmp}(m,n),
\end{equation*}
and
\begin{equation}
\label{eq:empirical-gd-entry-radii}
    r_-^{\rm gd}
    :=
    \underline\mu_0^{\rm gd}-e_{\rm cmp}^{\rm gd}(m,n),
    \qquad
    r_+^{\rm gd}
    :=
    \overline\mu_0^{\rm gd}+e_{\rm cmp}^{\rm gd}(m,n).
\end{equation}
Let $n_{\rm eff}:=m\wedge n$ and fix $C_{\rm stat}>1$.
Choose the shell constants $C_{\rm shell}\ge1$ and $c_{\rm shell}>0$
from the concentration bounds. Then choose $c_{{\rm stat},0}>0$
sufficiently small, depending only on the fixed population and
concentration constants and on the enlargement factors in
\eqref{eq:empirical-enlarged-shell}, and fix
\begin{equation}
\label{eq:main-cstat-choice-gd}
    0<c_{\rm stat}\le c_{{\rm stat},0}.
\end{equation}
Define $\overline M_{\rm stat}(m,n) := c_{\rm stat} n_{\rm eff}/{\log(C_{\rm stat}mn)}.$ Whenever $\overline M_{\rm stat}\ge2\mu_*$, let $J_{\rm stat}$ be the largest integer satisfying
$2R_{J_{\rm stat}} \le \overline M_{\rm stat}, \ R_j:=2^j\mu_*,$
and set
$M_{\rm stat}:=2R_{J_{\rm stat}}.$
Then
\begin{equation}
\label{eq:stat-radius-equivalence}
    \frac12\overline M_{\rm stat}
    <
    M_{\rm stat}
    \le
    \overline M_{\rm stat}.
\end{equation}

For $0\le j\le J_{\rm stat}$ define
\begin{equation}
\label{eq:main-shell-failure-probability}
    p_j(m,n)
    :=
    C_{\rm shell}(1+R_j)mn
    \exp\left\{
        -c_{\rm shell}\frac{n_{\rm eff}}{R_j}
    \right\},
\end{equation}
and
\begin{equation}
\label{eq:main-training-failure-probability}
    p_{\rm train}^{\rm gd}(m,n)
    :=
    p_{\rm ent}(m,n)
    +
    \sum_{j=0}^{J_{\rm stat}}p_j(m,n).
\end{equation}
For the fixed sufficiently small choice of $c_{\rm stat}$, there exist fixed constants $C_{\rm tr},c_{\rm tr},c_{\rm prob},q_{\rm prob}>0$ and deterministic thresholds $m_{\rm prob},n_{\rm prob}$ such that
\begin{equation}
\label{eq:main-training-failure-simplified}
    p_{\rm train}^{\rm gd}(m,n)
    \le
    C_{\rm tr}
    e^{-c_{\rm tr}m(\varepsilon_{\rm ent}^{\rm gd})^2}
    +
    C_{\rm tr}(mn)^{-c_{\rm prob}}
\end{equation}
whenever
\begin{equation*}
    m\ge m_{\rm prob},
    \qquad
    n\ge n_{\rm prob},
    \qquad
    \frac{m\wedge n}{\log(C_{\rm stat}mn)}\ge q_{\rm prob}.
\end{equation*}

\begin{proposition}[Simple sufficient empirical-GD pretraining regime]
\label{prop:main-simple-pretraining-regime-gd}
Fix the population initialization and the constants above. There exist deterministic constants $m_0,n_0,q_0>0$ such that, whenever
\begin{equation}
\label{eq:main-simple-pretraining-regime-gd}
    m\ge m_0,
    \qquad
    n\ge n_0,
    \qquad
    A_{m,n}
    :=
    \frac{m\wedge n}{\log(C_{\rm stat}mn)}
    \ge q_0,
\end{equation}
one has
\begin{equation*}
    p_{\rm ent}(m,n)\in(0,1),
    \qquad
    e_{\rm cmp}^{\rm gd}(m,n)
    \le
    \frac{d_{\rm ent}^{\rm gd}}2,
    \qquad
    r_-^{\rm gd}>\mu_*,
\end{equation*}
and
$M_{\rm stat}>2r_+^{\rm gd}, \ p_{\rm train}^{\rm gd}(m,n)\le\frac12.$
Moreover, for every certified shell $0\le j\le J_{\rm stat}$,
$p_j(m,n)\in(0,1), \ n\ge C_{\rm res}R_jL_{R_j}(p_j/2),$
and
\begin{equation*}
    C_{\rm res}\left\{
        \sqrt{\frac{R_jL_{R_j}(p_j/2)}{m}}
        +
        \frac{R_jL_{R_j}(p_j/2)}{m}
        +
        \frac{R_j}{n}
    \right\}
    \le
    \min\{c_{\rm dir},C_{\rm sp,0}\}.
\end{equation*}
\end{proposition}

\begin{proof}[Proof of Proposition~\ref{prop:main-simple-pretraining-regime-gd}]
All population-GD entry quantities are fixed once the initialization and the learning-rate ceiling $\gamma_0$ are fixed. Because $\varepsilon_{\rm ent}^{\rm gd}>0$ is fixed, the choice
$p_{\rm ent} = C_{\rm cmp}e^{-c_{\rm cmp}m(\varepsilon_{\rm ent}^{\rm gd})^2}$
gives
\begin{equation*}
    \frac1m
    \log\left(
        \frac{C(1+R_{\rm cmp})}{p_{\rm ent}}
    \right)
    \le
    c_{\rm cmp}(\varepsilon_{\rm ent}^{\rm gd})^2+\frac Cm.
\end{equation*}
Choose $c_{\rm cmp}>0$ sufficiently small and then choose $m_0,n_0$ sufficiently large that $p_{\rm ent}\in(0,1), \ e_{\rm cmp}^{\rm gd} \le \frac{d_{\rm ent}^{\rm gd}}2.$
In particular $r_-^{\rm gd}>\mu_*$.

For the shell events,
\begin{equation*}
    L_{R_j}(p_j/2)
    =
    \log\left(
        \frac{2C_{\rm res}(1+R_j)mn}{p_j}
    \right)
    \le
    C+c_{\rm shell}\frac{n_{\rm eff}}{R_j}.
\end{equation*}
Hence $R_jL_{R_j}(p_j/2)\le CR_j+c_{\rm shell}n_{\rm eff}$. Choose $c_{\rm shell}$ sufficiently small, then $c_{\rm stat}\le c_{{\rm stat},0}$ sufficiently small, and finally increase the fixed thresholds $m_0,n_0,q_0$. Uniformly over all certified shells, this gives $n\ge C_{\rm res}R_jL_{R_j}(p_j/2)$ and
\begin{equation*}
    C_{\rm res}\left\{
        \sqrt{\frac{R_jL_{R_j}(p_j/2)}m}
        +
        \frac{R_jL_{R_j}(p_j/2)}m
        +
        \frac{R_j}{n}
    \right\}
    \le
    \min\{c_{\rm dir},C_{\rm sp,0}\}
\end{equation*}
after increasing the deterministic sample thresholds. The uniform upper entry radius $\overline\mu_0^{\rm gd}$ is fixed. Increasing $q_0$ if necessary makes
\begin{equation*}
    \overline M_{\rm stat}
    \ge
    4\left(\overline\mu_0^{\rm gd}+\frac{d_{\rm ent}^{\rm gd}}2\right).
\end{equation*}
Since $M_{\rm stat}>\overline M_{\rm stat}/2$ and $r_+^{\rm gd}\le\overline\mu_0^{\rm gd}+d_{\rm ent}^{\rm gd}/2$, this implies $M_{\rm stat}>2r_+^{\rm gd}$. Finally, the shell failure sum satisfies
\begin{equation*}
    \sum_{j=0}^{J_{\rm stat}}p_j
    \le
    Cmn(1+M_{\rm stat})\log\left(\frac{2M_{\rm stat}}{\mu_*}\right)
    \exp\left\{-c\frac{n_{\rm eff}}{M_{\rm stat}}\right\}
    \le
    C(mn)^{-c_{\rm prob}}
\end{equation*}
for sufficiently small $c_{\rm stat}$. Increasing $m_0,n_0,q_0$ once more gives $p_{\rm train}^{\rm gd}\le1/2$ and $p_j\in(0,1)$ for every certified shell. This proves all claims.
\end{proof}

\subsection{Deterministic training-time window}
\label{subsec:empirical-gd-window}

Let
$c_-^{\rm gd} := 2c_\mu,$
and let $c_+^{\rm gd}>0$ be the fixed upper fourth-power growth constant in Proposition~\ref{prop:empirical-gd-one-shell}. For the actual GD iterates, retain the integer entry index $s_0=s_0(\gamma)$. Introduce the post-capture training clock $T_s := T_0 + \gamma(s-s_0), \qquad s\ge s_0,$ so that $T_{s_0}=T_0$. Define $\mathbb T_\gamma := \left\{ T_0+\gamma j: j=0,1,2,\ldots \right\},$ and write
\begin{equation}
\label{eq:empirical-gd-time-parametrization}
    \widehat\theta^\gamma(T_s)
    :=
    \widehat\theta_s^\gamma,
    \qquad
    \widehat\mu^\gamma(T_s)
    :=
    \widehat\mu_s^\gamma.
\end{equation}

Define the deterministic radial envelopes directly on the post-capture clock:
\begin{align}
\label{eq:empirical-gd-envelopes}
    M_-^{\rm gd}(T)
    &:=
    \left\{
        (r_-^{\rm gd})^4
        +
        c_-^{\rm gd}
        (T-T_0)
    \right\}^{1/4},
    \\
    M_+^{\rm gd}(T)
    &:=
    \left\{
        (r_+^{\rm gd})^4
        +
        c_+^{\rm gd}
        (T-T_0)
    \right\}^{1/4}.
\end{align}
Thus
\begin{equation*}
    M_\pm^{\rm gd}(T_s)
    =
    \left\{
        (r_\pm^{\rm gd})^4
        +
        c_\pm^{\rm gd}
        \gamma(s-s_0)
    \right\}^{1/4}.
\end{equation*}

Set
\begin{equation}
\label{eq:empirical-gd-statistical-iteration}
    S_{\rm stat}(m,n,\gamma)
    :=
    s_0
    +
    \left\lfloor
        \frac{
            M_{\rm stat}^4
            -
            (r_+^{\rm gd})^4
        }{
            c_+^{\rm gd}\gamma
        }
    \right\rfloor,
\end{equation}
and define the corresponding deterministic statistical time
\begin{equation}
\label{eq:empirical-gd-statistical-time}
    T_{\rm stat}^{\rm gd}(m,n,\gamma)
    :=
    T_0
    +
    \gamma
    \left\{
        S_{\rm stat}(m,n,\gamma)-s_0
    \right\}.
\end{equation}

\subsection{Completion of the proof}

\begin{proof}[Proof of Theorem~\ref{thm:main-empirical}]
Increase $A_0$ if necessary so that $A_{m,n}\ge A_0$ implies
$m\ge m_0$, $n\ge n_0$, and $A_{m,n}\ge q_0$. Hence
Proposition~\ref{prop:main-simple-pretraining-regime-gd} applies.

Let $\mathcal E_{\rm ent}$ denote the compact approximation event from
Proposition~\ref{prop:compact-vector-field-approximation} on $K_0$.
On $\mathcal E_{\rm ent}$, the compact error is below
$\varepsilon_{\rm ent}^{\rm gd}$, so
Lemma~\ref{lem:empirical-gd-compact-tracking} gives
\[
    \widehat\theta_{s_0}^\gamma
    \in
    \operatorname{int}(\mathcal G),
    \qquad
    r_-^{\rm gd}
    \le
    \widehat\mu_{s_0}^\gamma
    \le
    r_+^{\rm gd}.
\]
Since $r_-^{\rm gd}>\mu_*$ and $M_{\rm stat}>2r_+^{\rm gd}$, there is a
unique $j_{\rm ent}\in\{0,\ldots,J_{\rm stat}\}$ such that
\[
    R_{j_{\rm ent}}
    \le
    \widehat\mu_{s_0}^\gamma
    <
    2R_{j_{\rm ent}},
\]
hence
$\widehat\theta_{s_0}^\gamma\in\mathcal S_{R_{j_{\rm ent}}}$.

For every deterministic shell $0\le j\le J_{\rm stat}$, let
$\mathcal E_j^{\rm dir}$ and $\mathcal E_j^{\rm sp}$ denote the events
from Theorem~\ref{thm:shellwise-directional-approximation} and
Proposition~\ref{prop:shellwise-speed-approximation}, respectively, each
with failure probability $p_j/2$. Proposition~\ref{prop:main-simple-pretraining-regime-gd}
verifies the required deterministic smallness conditions. Define
\[
    \mathcal E_{\rm shell}
    :=
    \bigcap_{j=0}^{J_{\rm stat}}
    \left(
        \mathcal E_j^{\rm dir}
        \cap
        \mathcal E_j^{\rm sp}
    \right),
    \qquad
    \mathcal E_{\rm train}
    :=
    \mathcal E_{\rm ent}\cap\mathcal E_{\rm shell}.
\]
Then the union bound gives
\[
    \Prob(\mathcal E_{\rm train})
    \ge
    1-p_{\rm train}^{\rm gd}(m,n)
    \ge
    1-Ce^{-cm}-C(mn)^{-c}.
\]

On $\mathcal E_{\rm train}$, all certified shells are available before
the data-dependent entry index $j_{\rm ent}$ is identified.
Proposition~\ref{prop:empirical-gd-one-shell} therefore applies at
$R_{j_{\rm ent}}$. Its handoff property places the landing iterate in
$\mathcal S_{R_{j_{\rm ent}+1}}$ whenever
$j_{\rm ent}<J_{\rm stat}$. Iterating this argument propagates the
empirical path through every subsequent certified shell while preserving
$\mathcal G$, up to the terminal radius
$M_{\rm stat}=2R_{J_{\rm stat}}$.

Iterating the one-shell fourth-power bounds from the entry radii gives, at every certified grid time $T$, $M_-^{\rm gd}(T)\le\widehat\mu^\gamma(T)\le M_+^{\rm gd}(T)$. By the definition of $T_{\rm stat}^{\rm gd}$, this holds simultaneously for every $T\in[T_0,T_{\rm stat}^{\rm gd}]\cap\mathbb T_\gamma$, proving \eqref{eq:main-empirical-horizon}.

It remains to identify the late statistical-scale window. Evaluating the
lower envelope at the deterministic endpoint $T_{\rm stat}^{\rm gd}$
and using $\lfloor x\rfloor\ge x-1$ gives
\[
 \{M_-^{\rm gd}(T_{\rm stat}^{\rm gd})\}^4
 \ge (r_-^{\rm gd})^4+\frac{c_-^{\rm gd}}{c_+^{\rm gd}}\{M_{\rm stat}^4-(r_+^{\rm gd})^4\}-c_-^{\rm gd}\gamma.
\]
The entry radii and $\gamma$ are bounded independently of $(m,n)$, so after increasing the fixed threshold $A_0$, the right-hand side is at least $\frac{c_-^{\rm gd}}{2c_+^{\rm gd}}M_{\rm stat}^4$. Thus $M_-^{\rm gd}(T_{\rm stat}^{\rm gd})\ge c_{\rm term}M_{\rm stat}$ with $c_{\rm term}:=(c_-^{\rm gd}/(2c_+^{\rm gd}))^{1/4}$. Choosing the main-text constant $c_{\rm rate}<c_{\rm term}$ makes $T_{\rm rate}^{\rm gd}$ well defined. For every $T\in\mathcal T_{\rm stat}(m,n,\gamma)$, monotonicity of the envelopes gives $c_{\rm rate}M_{\rm stat}\le\widehat\mu^\gamma(T)\le M_{\rm stat}$. Since $M_{\rm stat}\asymp A_{m,n}$, this is \eqref{eq:main-empirical-scale}; Lemma~\ref{lem:log-prediction-bound} gives the accompanying $C/A_{m,n}$ population prediction bound. Finally, the fourth-power envelopes imply $T_{\rm rate}^{\rm gd}-T_0\asymp T_{\rm stat}^{\rm gd}-T_0\asymp A_{m,n}^4$.
\end{proof}

\section{Fresh-Context Transfer from Frozen Checkpoints}
\label{sec:one-head-interior-transfer}

The empirical-GD theorem above controls every iterate in the complete deterministic post-capture training-time window. We keep the post-capture time $T_0$ exactly as it was fixed in Section~\ref{sec:empirical-pretraining}; in particular, none of the quantities $s_0$, $K_0$, $r_\pm^{\rm gd}$, the compact-entry constants, or the shellwise pretraining thresholds are redefined in this section.

We formulate fresh transfer for an arbitrary checkpoint selected from a later deterministic subwindow. The checkpoint is chosen using pretraining information only and is frozen before the fresh threshold, the fresh-context size, the confidence level, or the fresh observations are revealed.

Fix
\begin{equation}
\label{eq:one-head-transfer-margin}
    \eta
    \in
    (0,\Delta).
\end{equation}

Consider the safe one-head family
\begin{equation}
\label{eq:one-head-safe-interior}
    \alpha^{\rm fr}
    \in
    [
        \rho+\eta,
        \ell
    ],
    \qquad
    \ell:=1-\rho.
\end{equation}
Conditional on $\alpha^{\rm fr}$, draw an independent clean fresh context
\begin{equation}
\label{eq:one-head-fresh-context}
    X_i^{\rm fr}
    \stackrel{\mathrm{i.i.d.}}{\sim}
    \Unif(0,1),
    \qquad
    Y_i^{\rm fr}
    :=
    \mathbf 1
    \{
        X_i^{\rm fr}
        \le
        \alpha^{\rm fr}
    \},
    \qquad
    i=1,\ldots,N.
\end{equation}
For any frozen $\theta=(\lambda,\mu)$, define
\begin{equation}
\label{eq:one-head-fresh-predictor}
    F_N^{\rm fr}(\theta)
    :=
    \frac{
        \sum_{i=1}^N
        X_i^{\rm fr}
        e^{
            \lambda Y_i^{\rm fr}
            +
            \mu X_i^{\rm fr}
        }
    }{
        \sum_{i=1}^N
        e^{
            \lambda Y_i^{\rm fr}
            +
            \mu X_i^{\rm fr}
        }
    }.
\end{equation}

\subsection{Fixed-parameter fresh-context bounds}
\label{subsec:one-head-two-regimes}

The fresh-context proof uses two complementary arguments.

Below the self-normalized resolution scale, the fixed positive margin $\alpha^{\rm fr}-\rho\ge\eta$ implies that a positive observation sufficiently close to the threshold has a fixed normalized score advantage over every negative observation. Order statistics and deterministic softmax geometry then give an all-$N$ bound with no effective-sample-size assumption.

For sufficiently large contexts, the normalized weighted empirical moments resolve the frozen concentration scale. A self-normalized concentration inequality then gives the sharper approximation
\begin{equation*}
    F_N^{\rm fr}(\theta)
    =
    F(\theta;\alpha^{\rm fr})
    +
    O_{\Prob}
    \left(
        \sqrt{
            \frac1{\mu N}
        }
        +
        \frac1N
    \right).
\end{equation*}
The second argument removes the artificial growth of the deterministic $\log N/\mu$ softmax approximation when $N$ is very large.

\subsubsection{Deterministic softmax geometry}
\label{appsubsec:one-head-softmax-geometry}

\begin{lemma}[Softmax approximation to a maximum]
\label{lem:one-head-softmax-maximum}
Let $z_1,\ldots,z_M\in[0,1]$, $M\ge1$, and define
\begin{equation*}
    \overline z_\mu
    :=
    \frac{
        \sum_{i=1}^M
        z_i e^{\mu z_i}
    }{
        \sum_{i=1}^M
        e^{\mu z_i}
    }.
\end{equation*}
Then
\begin{equation}
\label{eq:one-head-softmax-maximum}
    0
    \le
    \max_{1\le i\le M}z_i
    -
    \overline z_\mu
    \le
    \frac{
        \log M
    }{
        \mu
    }.
\end{equation}
\end{lemma}

\begin{proof}
Let $w_i := { e^{\mu z_i} }/{ \sum_{j=1}^M e^{\mu z_j} }$. The Shannon entropy $H(w):=-\sum_iw_i\log w_i$ satisfies $0\le H(w)\le\log M$, while
\begin{equation*}
    H(w)
    =
    \log\left(
        \sum_{j=1}^M
        e^{\mu z_j}
    \right)
    -
    \mu
    \sum_{i=1}^M
    w_i z_i.
\end{equation*}
Since $\log\left( \sum_{j=1}^M e^{\mu z_j} \right) \ge \mu\max_i z_i$, the upper bound follows. The lower bound holds because $\overline z_\mu$ is a convex combination of the $z_i$.
\end{proof}

\begin{lemma}[One-head class-margin bound]
\label{lem:one-head-class-margin}
Fix a clean labeled context
$\mathcal D = \{(x_i,y_i)\}_{i=1}^N$
with at least one positive-label observation, and define
$M_1 := \max_{y_i=1}x_i.$
For $\theta=(\lambda,\mu)$ with $\mu>0$, write $k=\lambda/\mu$. If the negative-label class is nonempty, define
$M_0 := \max_{y_i=0}x_i, \ \Gamma_+ := k+M_1-M_0.$
If the negative-label class is empty, set $\Gamma_+:=+\infty$. Then, for every threshold $\alpha$ consistent with the clean labels,
\begin{equation}
\label{eq:one-head-class-margin}
    \left|
        F_N(\theta;\mathcal D)
        -
        \alpha
    \right|
    \le
    (\alpha-M_1)
    +
    \frac{
        \log(N+1)
    }{
        \mu
    }
    +
    \min\left\{
        1,
        Ne^{-\mu\Gamma_+}
    \right\}.
\end{equation}
\end{lemma}

\begin{proof}
Let $W_1$ and $W_0$ denote the total unnormalized attention weights in the positive and negative classes, respectively. Within the positive class, the common factor $e^\lambda$ cancels after normalization. Hence Lemma~\ref{lem:one-head-softmax-maximum} gives
\begin{equation*}
    0
    \le
    M_1
    -
    \frac{
        \sum_{y_i=1}
        x_i e^{\mu x_i}
    }{
        \sum_{y_i=1}
        e^{\mu x_i}
    }
    \le
    \frac{
        \log N
    }{
        \mu
    }.
\end{equation*}
The positive-class softmax average therefore lies in
$\left[ M_1-\frac{\log N}{\mu}, \, M_1 \right] \subseteq (-\infty,\alpha].$

If the negative class is nonempty,
$W_1 \ge e^{\lambda+\mu M_1}, \ W_0 \le N e^{\mu M_0},$
so
$\frac{ W_0 }{ W_1 } \le Ne^{-\mu\Gamma_+}.$
Thus the total negative-class softmax mass is at most
$\min\left\{ 1, Ne^{-\mu\Gamma_+} \right\}.$
Since every feature value lies in $[0,1]$, adding the negative class can move the output by at most this mass. Combining this with
$\alpha-M_1$
and the within-positive-class softmax error proves \eqref{eq:one-head-class-margin}. If the negative class is empty, the last term vanishes and the same argument applies.
\end{proof}

\begin{lemma}[Absorption of the wrong-class term]
\label{lem:one-head-wrong-class-absorption}
For every fixed $c>0$, there exist constants $C_c,\mu_c>0$ such that, for every $N\ge1$ and every $\mu\ge\mu_c$,
\begin{equation}
\label{eq:one-head-wrong-class-absorption}
    \min\left\{
        1,
        Ne^{-c\mu}
    \right\}
    \le
    C_c
    \frac{
        1+\log(N+1)
    }{
        \mu
    }.
\end{equation}
\end{lemma}

\begin{proof}
If
$\log(N+1) \ge \frac{c\mu}{2},$
then
$\frac{ 1+\log(N+1) }{ \mu } \ge \frac c2,$
so the right-hand side dominates $1$ after choosing $C_c\ge2/c$. If instead
$\log(N+1) < \frac{c\mu}{2},$
then
$Ne^{-c\mu} \le e^{-c\mu/2} \le \frac{ C_c }{ \mu }$
after increasing $\mu_c$ and $C_c$. This proves the claim.
\end{proof}

\begin{proposition}[Order statistic one-head bound]
\label{prop:one-head-fixed-interior}
Fix $\eta\in(0,\Delta)$. There exist constants $C_\eta,\mu_\eta>0$, depending only on the fixed structural constants and $\eta$, such that the following holds. For every
\begin{equation*}
    \theta=(\lambda,\mu)
    \in
    \mathcal G,
    \qquad
    \mu\ge\mu_\eta,
    \qquad
    \alpha\in[\rho+\eta,\ell],
\end{equation*}
every integer $N\ge1$, and every $u\ge1$,
\begin{equation}
\label{eq:one-head-fixed-interior-bound}
    \Prob_{\rm fr}
    \left(
        \left|
            F_N^{\rm fr}(\theta)-\alpha
        \right|
        >
        C_\eta
        \left\{
            \frac{u}{N}
            +
            \frac{
                1+\log(N+1)
            }{
                \mu
            }
        \right\}
    \right)
    \le
    e^{-u}.
\end{equation}
The constants are uniform over the displayed parameter and threshold ranges, and no lower-bound condition on $N$ is imposed.
\end{proposition}

\begin{proof}[Proof of Proposition~\ref{prop:one-head-fixed-interior}]
Fix
$\alpha \in [\rho+\eta,\ell].$
Let $\mu_{\eta/2}^{\rm wc}$ denote the threshold $\mu_c$ from Lemma~\ref{lem:one-head-wrong-class-absorption} with $c=\eta/2$. Increase $\mu_\eta$ once and for all so that
\begin{equation}
\label{eq:one-head-k-above-ell}
    \mu\ge\mu_\eta,
    \qquad
    \theta\in\mathcal G
    \quad\Longrightarrow\quad
    k
    \ge
    \ell,
\end{equation}
and
\begin{equation}
\label{eq:one-head-wrong-class-threshold-incorporated}
    \mu_\eta
    \ge
    \mu_{\eta/2}^{\rm wc}.
\end{equation}
For example, it is enough to impose
$\mu_\eta \ge \max \left\{ \mu_*, e^{-c_0}, \mu_{\eta/2}^{\rm wc} \right\},$
because
$k \ge b_{c_0}(\mu) = \ell + \frac{ \log\mu+c_0 }{ \mu }.$

First suppose
\begin{equation}
\label{eq:one-head-nontrivial-order-regime}
    \frac{u}{N}
    \le
    \frac\eta2.
\end{equation}
Define the threshold-neighborhood event
\begin{equation}
\label{eq:one-head-near-positive-event}
    \mathcal E_u
    :=
    \left\{
        \exists i\le N:
        X_i^{\rm fr}
        \in
        \left(
            \alpha-\frac{u}{N},
            \alpha
        \right]
    \right\}.
\end{equation}
The interval in \eqref{eq:one-head-near-positive-event} has length $u/N$ and is contained in $(0,\alpha]$ under \eqref{eq:one-head-nontrivial-order-regime}. Therefore
\begin{equation}
\label{eq:one-head-order-event-failure}
    \Prob_{\rm fr}(\mathcal E_u^c)
    =
    \left(
        1-\frac{u}{N}
    \right)^N
    \le
    e^{-u}.
\end{equation}

On $\mathcal E_u$, the positive-label class is nonempty. Only now define
$M_1 := \max\{ X_i^{\rm fr}: Y_i^{\rm fr}=1 \}.$
Then
\begin{equation}
\label{eq:one-head-M1-spacing}
    0
    \le
    \alpha-M_1
    \le
    \frac{u}{N}.
\end{equation}
If the negative class is nonempty, let
$M_0 := \max\{ X_i^{\rm fr}: Y_i^{\rm fr}=0 \}.$
Since $M_0\le1$, \eqref{eq:one-head-k-above-ell} and \eqref{eq:one-head-M1-spacing} imply
\begin{align}
\label{eq:one-head-positive-score-margin}
    \Gamma_+
    &=
    k+M_1-M_0
    \notag\\
    &\ge
    \ell
    +
    \alpha
    -
    \frac{u}{N}
    -
    1
    \notag\\
    &=
    \alpha-\rho-\frac{u}{N}
    \notag\\
    &\ge
    \frac\eta2.
\end{align}
If the negative class is empty, set $\Gamma_+=+\infty$ as in Lemma~\ref{lem:one-head-class-margin}.

Applying Lemmas \ref{lem:one-head-class-margin} and \ref{lem:one-head-wrong-class-absorption} with $c=\eta/2$ gives, on $\mathcal E_u$,
\begin{equation*}
    \left|
        F_N^{\rm fr}(\theta)-\alpha
    \right|
    \le
    C_\eta
    \left\{
        \frac{u}{N}
        +
        \frac{
            1+\log(N+1)
        }{
            \mu
        }
    \right\}.
\end{equation*}
Together with \eqref{eq:one-head-order-event-failure}, this proves the desired probability bound in the regime \eqref{eq:one-head-nontrivial-order-regime}.

It remains to consider ${u}/{N} > \eta/2.$ Since $F_N^{\rm fr}(\theta), \alpha \in (0,1),$ one always has $\left| F_N^{\rm fr}(\theta)-\alpha \right| \le 1.$
After taking $C_\eta\ge2/\eta$, $C_\eta \frac{u}{N} > 1,$ so the bad event in \eqref{eq:one-head-fixed-interior-bound} is empty. This completes the proof for every $N\ge1$.
\end{proof}

\begin{proposition}[Large-context self-normalized one-head bound]
\label{prop:one-head-large-context}
There exist constants $C_{\rm fr},\mu_{\rm fr}>0$, depending only on the fixed structural constants, such that the following holds. For every
\begin{equation*}
    \theta=(\lambda,\mu)
    \in
    \mathcal G,
    \qquad
    \mu\ge\mu_{\rm fr},
    \qquad
    \alpha\in[\rho,\ell],
\end{equation*}
every $u\ge1$, and every
\begin{equation}
\label{eq:one-head-large-context-condition}
    N
    \ge
    C_{\rm fr}\mu u,
\end{equation}
one has
\begin{equation}
\label{eq:one-head-large-context-bound}
    \Prob_{\rm fr}
    \left(
        \left|
            F_N^{\rm fr}(\theta)-\alpha
        \right|
        >
        C_{\rm fr}
        \left\{
            \frac1\mu
            +
            \sqrt{
                \frac{u}{\mu N}
            }
            +
            \frac{u}{N}
        \right\}
    \right)
    \le
    4e^{-u}.
\end{equation}
In particular, under \eqref{eq:one-head-large-context-condition},
\begin{equation}
\label{eq:one-head-large-context-simplified}
    \left|
        F_N^{\rm fr}(\theta)-\alpha
    \right|
    \le
    C_{\rm fr}
    \left\{
        \frac1\mu
        +
        \frac{u}{N}
    \right\}
\end{equation}
on the same event after increasing $C_{\rm fr}$. Unlike Proposition~\ref{prop:one-head-fixed-interior}, this proposition is valid on the full closed threshold interval $[\rho,\ell]$, but it requires the fresh context to resolve the scale $\mu$.
\end{proposition}

\begin{proof}[Proof of Proposition~\ref{prop:one-head-large-context}]
Fix
$\theta=(\lambda,\mu) \in \mathcal G, \ \alpha\in[\rho,\ell].$
Let
\begin{equation*}
    W_{\theta,\alpha}(X,Y)
    :=
    \frac{
        e^{\lambda Y+\mu X}
    }{
        \E[
            e^{\lambda Y+\mu X}
            \mid
            \alpha
        ]
    },
\end{equation*}
and define
$\phi_F := W_{\theta,\alpha} \bigl( X-F(\theta;\alpha) \bigr).$
Then
$P_\alpha W_{\theta,\alpha} = 1, \ P_\alpha\phi_F = 0,$
and the fresh-context error has the exact self-normalized representation
\begin{equation}
\label{eq:one-head-self-normalized-representation}
    F_N^{\rm fr}(\theta)
    -
    F(\theta;\alpha)
    =
    \frac{
        P_N\phi_F
    }{
        P_NW_{\theta,\alpha}
    }.
\end{equation}

By Lemma~\ref{lem:normalized-weight-geometry}, uniformly over $\theta\in\mathcal G$ and $\alpha\in[\rho,\ell]$,
$\|W_{\theta,\alpha}\|_\infty \le C\mu,$
$P_\alpha \bigl( W_{\theta,\alpha}-1 \bigr)^2 \le C\mu,$
$\|\phi_F\|_\infty \le C, \ P_\alpha\phi_F^2 \le \frac{C}{\mu}.$
Thus, in Lemma~\ref{lem:self-normalized-ratio-high-probability},
we may take
$B\le C\mu$, $\sigma_A^2\le C\mu$, $b\le C$, and
$s_2\le C/\mu$. After increasing $C_{\rm fr}$, the condition
$N\ge C_{\rm fr}\mu u$ implies
\[
    1\le u
    \le
    c\,\frac{N}{B+\sigma_A^2+1},
\]
so Lemma~\ref{lem:self-normalized-ratio-high-probability} applies to
\eqref{eq:one-head-self-normalized-representation}. Hence, with
probability at least $1-4e^{-u}$,
\begin{equation}
\label{eq:one-head-large-context-fluctuation}
    \left|
        F_N^{\rm fr}(\theta)
        -
        F(\theta;\alpha)
    \right|
    \le
    C_{\rm fr}
    \left\{
        \sqrt{
            \frac{
                u
            }{
                \mu N
            }
        }
        +
        \frac{
            u
        }{
            N
        }
    \right\}.
\end{equation}
The population prediction bound in $\mathcal G$ gives
$\left| F(\theta;\alpha)-\alpha \right| \le \frac{ C_{\rm bias} }{ \mu }.$
The triangle inequality proves \eqref{eq:one-head-large-context-bound}.

Finally, $N\ge C_{\rm fr}\mu u$ implies
$\sqrt{ \frac{ u }{ \mu N } } \le \frac{ C }{ \mu },$
which gives \eqref{eq:one-head-large-context-simplified} after enlarging $C_{\rm fr}$.
\end{proof}

Define $L_\mu(u) := 1 + \log \left( 1+C_{\rm fr}\mu u \right)$ and $\Phi(\mu,u):={L_\mu(u)}/{\mu}$. 
For every fixed $u\ge1$, the map $\mu\mapsto\Phi(\mu,u)$ is strictly decreasing on $(0,\infty)$.

\subsubsection{Proof of the main-text fixed-estimator proposition}
\label{app:proof-fixed-estimator}
\begin{proof}[Proof of Proposition~\ref{prop:fixed-estimator-v3}]
Fix
$\alpha \in [\rho+\eta,\ell], \ u\ge1.$
There are two cases.

If
$N \ge C_{\rm fr}\mu u,$
Proposition~\ref{prop:one-head-large-context} gives, with failure probability at most $4e^{-u}$,
\begin{equation*}
    \left|
        F_N^{\rm fr}(\theta)-\alpha
    \right|
    \le
    C
    \left\{
        \frac1\mu
        +
        \frac{u}{N}
    \right\}
    \le
    C_\eta
    \left\{
        \frac{
            L_\mu(u)
        }{
            \mu
        }
        +
        \frac{u}{N}
    \right\}.
\end{equation*}

If instead
$N < C_{\rm fr}\mu u,$
Proposition~\ref{prop:one-head-fixed-interior} gives, with failure probability at most $e^{-u}$,
\begin{equation*}
    \left|
        F_N^{\rm fr}(\theta)-\alpha
    \right|
    \le
    C_\eta
    \left\{
        \frac{u}{N}
        +
        \frac{
            1+\log(N+1)
        }{
            \mu
        }
    \right\}.
\end{equation*}
Since
$N+1 \le 1+C_{\rm fr}\mu u,$
after a harmless enlargement of the constant,
$1+\log(N+1) \le L_\mu(u).$
Thus the bound in \eqref{eq:fixed-estimator-v3} holds in both regimes. In either case the failure probability is at most $4e^{-u}$.
\end{proof}

\begin{lemma}[Freeze-then-transfer conditioning]
\label{lem:freeze-conditioning}
Let $\mathcal F_{\rm pre}$ be the pretraining sigma-field and let $\widehat\theta$ be $\mathcal F_{\rm pre}$-measurable. Suppose that on a pretraining event $\mathcal E_{\rm train}$ one has $\widehat\theta\in\Theta_{\rm good}$, and suppose that for every deterministic $\theta\in\Theta_{\rm good}$ a fresh-context event $\mathcal A(\theta)$ satisfies $\Prob_{\rm fr}(\mathcal A(\theta))\ge1-\delta$, with constants uniform over $\Theta_{\rm good}$. Then, on $\mathcal E_{\rm train}$, $\Prob_{\rm fr}(\mathcal A(\widehat\theta)\mid\mathcal F_{\rm pre})\ge1-\delta$.
\end{lemma}
\begin{proof}
Condition on $\mathcal F_{\rm pre}$. The value of $\widehat\theta$ is then fixed and belongs to $\Theta_{\rm good}$ on $\mathcal E_{\rm train}$, so the deterministic fixed-parameter bound applies directly.
\end{proof}

Lemma~\ref{lem:freeze-conditioning} is the only probabilistic step needed to pass from a fixed-parameter fresh-context theorem to a deterministic or pretraining-measurable checkpoint. This makes the quantifier order explicit and avoids claiming a single fresh-context event that is simultaneous over all checkpoints.

\subsection{Freeze-then-transfer theorem}
\label{appsubsec:freeze-then-transfer}

\begin{proof}[Proof of Theorem~\ref{thm:main-transfer}]
Work on the training event $\mathcal E_{\rm train}$ from Theorem~\ref{thm:main-empirical}. For every $T\in\mathcal T_{\rm stat}(m,n,\gamma)$, that theorem gives $\widehat\theta^\gamma(T)\in\mathcal G$ and $c_A A_{m,n}\le\widehat\mu^\gamma(T)\le C_A A_{m,n}$. Increasing the fixed threshold $q_\eta^{\rm fr}$ ensures that $c_A A_{m,n}$ satisfies the sufficiently-large-$\mu$ condition in Proposition~\ref{prop:fixed-estimator-v3}.

Let $\widehat T$ be deterministic or $\mathcal F_{\rm pre}$-measurable with values in $\mathcal T_{\rm stat}$ on $\mathcal E_{\rm train}$, and set $\widehat\theta:=\widehat\theta^\gamma(\widehat T)$. Conditional on $\mathcal F_{\rm pre}$, this checkpoint is fixed. Lemma~\ref{lem:freeze-conditioning} and Proposition~\ref{prop:fixed-estimator-v3} therefore give, for every fixed $\alpha^{\rm fr}\in[\rho+\eta,\ell]$, every $N\ge1$, and every $u\ge1$, conditional fresh probability at least $1-4e^{-u}$ and
\[
 |F_N^{\rm fr}(\widehat\theta)-\alpha^{\rm fr}|
 \le C_\eta\left[\frac{1+\log(1+C_{\rm fr}\widehat\mu^\gamma(\widehat T)u)}{\widehat\mu^\gamma(\widehat T)}+\frac{u}{N}\right].
\]
This is \eqref{eq:main-transfer-conditional}.
\end{proof}

\subsection{Full closed interval once the fresh context resolves the scale}
\label{subsec:one-head-full-boundary-large-N}

\begin{corollary}[Full-boundary transfer along the certified time window]
\label{cor:one-head-full-boundary-large-N}
Under the assumptions of Theorem~\ref{thm:main-transfer}, fix $T\in\mathcal T_{\rm stat}(m,n,\gamma)$ and $u\ge1$. If
\begin{equation}
\label{eq:one-head-full-boundary-window-condition}
    N
    \ge
    C_{\rm fr}
    M_+^{\rm gd}(T)
    u,
\end{equation}
then, for every fixed $\alpha^{\rm fr}\in[\rho,\ell]$, with joint probability at least
$1 - p_{\rm train}^{\rm gd}(m,n) - 4e^{-u},$
\begin{equation}
\label{eq:one-head-full-boundary-window-bound}
    \left|
        F_N^{\rm fr}
        \left(
            \widehat\theta^\gamma(T)
        \right)
        -
        \alpha^{\rm fr}
    \right|
    \le
    C
    \left[
        \frac1{
            M_-^{\rm gd}(T)
        }
        +
        \sqrt{
            \frac{
                u
            }{
                M_-^{\rm gd}(T)
                N
            }
        }
        +
        \frac{
            u
        }{
            N
        }
    \right].
\end{equation}

If additionally $T$ lies in the statistical-scale training-time window and
\begin{equation}
\label{eq:one-head-full-boundary-statistical-window-condition}
    N
    \ge
    C_{\rm fr}
    C_A
    A_{m,n}
    u,
\end{equation}
then
\begin{equation}
\label{eq:one-head-full-boundary-large-N-bound}
    \left|
        F_N^{\rm fr}
        \left(
            \widehat\theta^\gamma(T)
        \right)
        -
        \alpha^{\rm fr}
    \right|
    \le
    C
    \left[
        \frac1{
            A_{m,n}
        }
        +
        \sqrt{
            \frac{
                u
            }{
                A_{m,n}
                N
            }
        }
        +
        \frac{
            u
        }{
            N
        }
    \right].
\end{equation}
Thus the frozen estimator is consistent at the lower endpoint when the certified scale and the fresh-context size diverge under the stated relation. The remaining obstruction concerns a uniform all-$N$ vanishing theorem for one frozen raw head.
\end{corollary}

\begin{proof}
Fix $T\in\mathcal T_{\rm stat}(m,n,\gamma)$ and work on $\mathcal E_{\rm train}$. Then $\widehat\theta^\gamma(T)\in\mathcal G$ and $M_-^{\rm gd}(T)\le\widehat\mu^\gamma(T)\le M_+^{\rm gd}(T)$. If $N\ge C_{\rm fr}M_+^{\rm gd}(T)u$, Proposition~\ref{prop:one-head-large-context} applies uniformly for every fixed $\alpha^{\rm fr}\in[\rho,\ell]$. Replacing $\widehat\mu^\gamma(T)$ by the lower envelope gives \eqref{eq:one-head-full-boundary-window-bound}. On the statistical-scale window, $M_-^{\rm gd}(T)\ge c_A A_{m,n}$ and $M_+^{\rm gd}(T)\le C_A A_{m,n}$, which gives \eqref{eq:one-head-full-boundary-large-N-bound} under \eqref{eq:one-head-full-boundary-statistical-window-condition}.
\end{proof}

\begin{remark}[Heuristic lower-endpoint obstruction for the all-context-length theorem]
\label{rem:one-head-interior-scope}
Fix a certified time $T$ and write
\begin{equation*}
    \widehat\theta^\gamma(T)
    =
    (
        \widehat\lambda^\gamma(T),
        \widehat\mu^\gamma(T)
    ),
    \qquad
    \widehat k^\gamma(T)
    :=
    \frac{
        \widehat\lambda^\gamma(T)
    }{
        \widehat\mu^\gamma(T)
    }.
\end{equation*}
Suppose both fresh label classes are present and define
\begin{equation*}
    M_1
    :=
    \max
    \{
        X_i^{\rm fr}:
        Y_i^{\rm fr}=1
    \},
    \qquad
    M_0
    :=
    \max
    \{
        X_i^{\rm fr}:
        Y_i^{\rm fr}=0
    \},
\end{equation*}
together with
$G_L := \alpha^{\rm fr}-M_1, \ G_R := 1-M_0.$
The normalized positive-minus-negative maximum-score gap is
\begin{equation}
\label{eq:one-head-score-gap-decomposition}
    \Delta_N(T)
    :=
    \widehat k^\gamma(T)
    +
    M_1
    -
    M_0
    =
    (
        \alpha^{\rm fr}-\rho
    )
    +
    (
        \widehat k^\gamma(T)-\ell
    )
    +
    G_R-G_L.
\end{equation}
On the safe interior, $\alpha^{\rm fr}-\rho\ge\eta$ provides a fixed deterministic margin.

At the lower endpoint $\alpha^{\rm fr}=\rho$, this fixed term vanishes. Inside the logarithmic tube,
\begin{equation*}
    \widehat k^\gamma(T)-\ell
    =
    O
    \left(
        \frac{
            \log
            \widehat\mu^\gamma(T)
        }{
            \widehat\mu^\gamma(T)
        }
    \right),
\end{equation*}
whereas fresh endpoint spacings fluctuate on the $N^{-1}$ scale. In the intermediate regime
$1 \ll N \ll { \widehat\mu^\gamma(T) }/{ \log \widehat\mu^\gamma(T) },$
the stochastic term $\widehat\mu^\gamma(T)(G_R-G_L)$ has scale $\widehat\mu^\gamma(T)/N$ and can dominate the learned logarithmic advantage. This comparison indicates that the negative class can win the maximum-score competition with nonvanishing probability and can therefore place order-one mass on the wrong class and potentially produce order-one prediction error. We use this observation only as a heuristic obstruction here; a formal impossibility statement would require a separate lower-bound proposition.

If instead $\alpha^{\rm fr}\sim\Unif([\rho,\ell])$ and probability is also taken over the fresh task, then
$\Prob \left( \alpha^{\rm fr}-\rho \le \eta \right) = { \eta }/{ \Delta }.$
Thus one may exclude a boundary layer in a joint-over-task theorem, but that boundary-layer probability must explicitly enter the final failure probability or risk. This obstruction also motivates studying reflected multi-expert models in future work. Reflection can repair physical-boundary coverage, but a routed construction must additionally control finite-context fluctuations.
\end{remark}
\section{Experimental Details and Additional Diagnostics}
\label{sec:experimental-details}
\subsection{Reproducibility Protocol}
Population calculations use analytic feature-coordinate integrals and 160-point Gauss-Legendre quadrature over the task threshold. Empirical experiments use exact full-batch gradients on independently sampled pretraining corpora. Every finite-pretraining iteration index is first computed from the deterministic population-GD reference path; it is therefore fixed before the corresponding empirical corpus is sampled. For transfer, one checkpoint is frozen per pretraining corpus and reused for every fresh threshold, context realization, and context size $N$.

Unless stated otherwise, a ``worst-grid pointwise 95th-percentile error'' is computed by taking the 95th percentile over independent fresh contexts at each tested threshold and then maximizing over the finite threshold grid. The metric provides pointwise coverage on that grid; simultaneous coverage over the continuum lies outside its definition. Code is available at \url{https://anonymous.4open.science/r/ICl_Decision_Stump-B5F3}.
\subsection{Detailed finite-step and transfer diagnostics}
The sufficient-stopping test uses $(128,128)$, $(256,256)$, $(512,512)$, $(128,512)$, $(512,128)$, $(256,512)$, and $(512,256)$, with 30 independent corpora per configuration. The aggregate radius and prediction-error outcomes are reported in Section~\ref{sec:experiments}. A stricter diagnostic based on one concrete finite-radius instantiation of the asymptotic logarithmic tube holds throughout the window in 104 of 210 runs. The theorem leaves its finite-radius threshold and comparison constants existential, so this operational tube serves as a descriptive diagnostic. An auxiliary axis sweep with six corpora per configuration remains reliable under the operational tracking criterion up to the largest tested reference radius, $\mu_{\rm ref}=60$, for all seven tested pairs, including values well beyond $A_{m,n}^{\rm proxy}$.

The separate end-to-end sweep used to generate Figure~\ref{fig:main-pretraining-transfer} uses 10, 8, 6, and 4 corpora for $m=n\in\{64,128,256,512\}$, respectively, and 6, 6, 5, and 5 corpora for $(m,n)\in\{(128,512),(512,128),(256,512),(512,256)\}$. These are the corpus counts for the displayed trajectory and frozen-transfer experiment; the sufficient-stopping diagnostic above uses 210 corpora.

Figure~\ref{fig:app-finite-pretraining} supplements the main continuation plot with the discrete-time clock and a common-shell heatmap comparing task-level and within-task sampling error.  The continuation panel in Figure~\ref{fig:main-pretraining-transfer} also records behavior beyond the certified proxy radius; failure-threshold estimation lies outside the scope of this sweep.

\begin{figure}[t]
    \centering
    \includegraphics[width=0.42\textwidth,trim={0 0 590bp 0},clip]{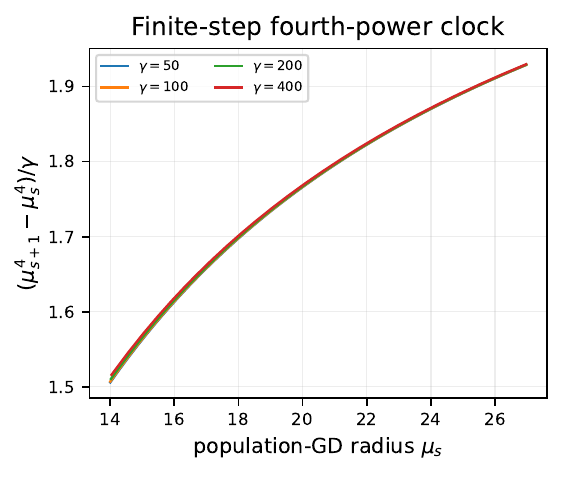}
    \hfill
    \includegraphics[width=0.42\textwidth,trim={575bp 0 0 0},clip]{figures/exp_finite_pretraining.pdf}
    \caption{Additional finite-step and finite-pretraining diagnostics. Population GD preserves the normalized fourth-power increment across step sizes (left), while a full $(m,n)$ grid separates task and within-task sampling effects at a common shell (right). The omitted continuation panel is reported in Figure~\ref{fig:main-pretraining-transfer}.}
    \label{fig:app-finite-pretraining}
\end{figure}

Figure~\ref{fig:app-frozen-transfer} expands the transfer summary.  Thin curves retain corpus-level pretraining variability, while each dotted line marks the population-resolution floor of the same frozen checkpoint.  The large-context panel isolates the inverse learned-resolution dependence.  At $N=4096$, the axis swaps $(128,512)$ and $(512,128)$ give errors $0.128$ and $0.124$, while $(256,512)$ and $(512,256)$ give $0.0592$ and $0.0586$.  After fitting an additive floor, the pre-floor excess-error slopes against $N$ range from $-0.87$ to $-0.94$ across the four balanced corpus sizes.

\begin{figure}[t]
    \centering
    \includegraphics[width=0.82\textwidth]{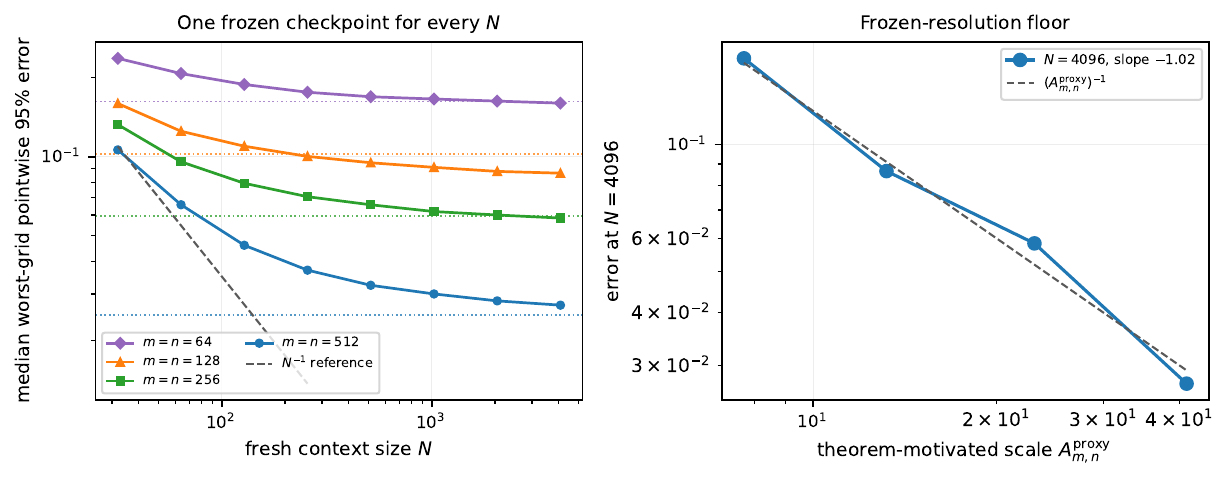}
    \caption{Detailed freeze-then-transfer results.  One checkpoint is selected without fresh-task information and reused for every $N$ (left).  At large $N$, the remaining error decreases approximately inversely with the theorem-motivated pretraining scale (right).}
    \label{fig:app-frozen-transfer}
\end{figure}

The initialization sweep uses 13 evenly spaced ratios $k_0\in[0.1,1.1]$ and 19 initial radii: zero and 18 positive values geometrically spaced from $10^{-3}$ to $30$. The design-and-noise scope check uses six independently sampled pretraining corpora per condition, eight threshold values, and 300 fresh contexts per threshold. The two nonuniform designs are $\operatorname{Beta}(2,2)$ and $\operatorname{Beta}(1/2,1/2)$ and are used for both pretraining and fresh-context evaluation. Global noise flips each affected label independently with probability $0.01$, while threshold-localized fresh noise flips with probability $1/2$ on $|X-\alpha|\le1/\widehat\mu$. The right panel of Figure~\ref{fig:app-scope-checks} displays these matched-design and label-corruption conditions. Bands retain variation across the six independently sampled pretraining corpora.

\begin{figure}[t]
    \centering
    \includegraphics[width=\textwidth]{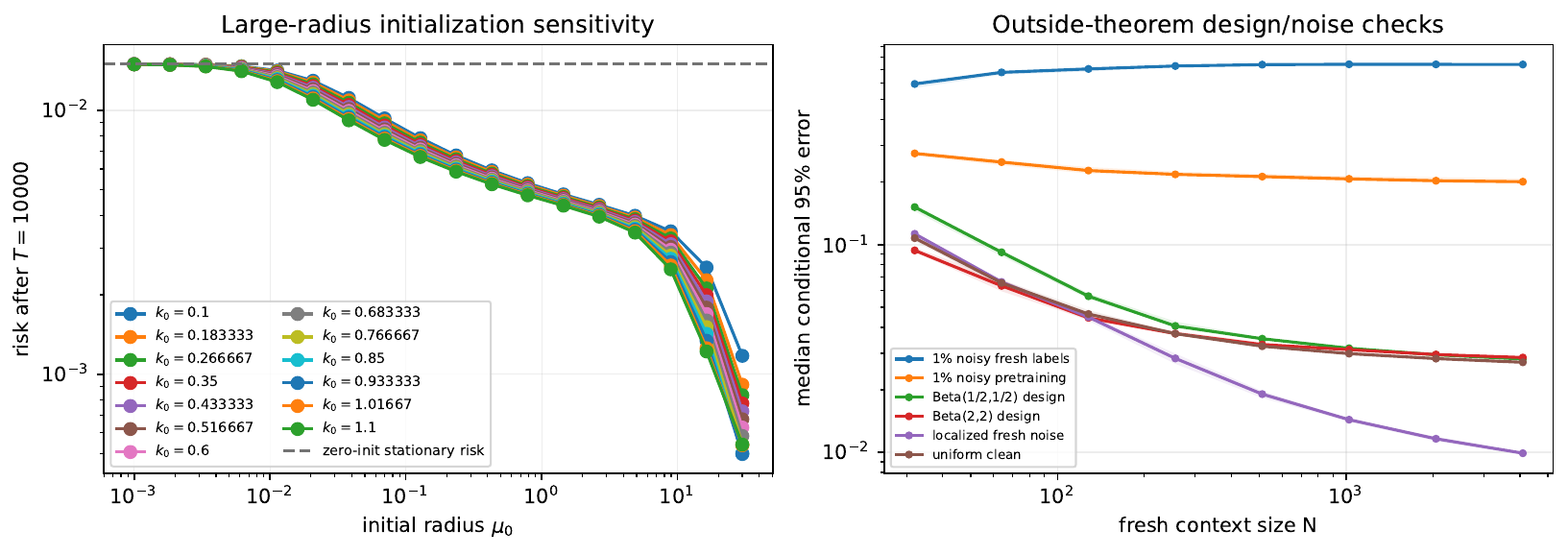}
    \caption{Scope checks.  The common-time initialization sweep exhibits the zero-gradient delay at small radius (left).  Two nonuniform designs remain close to the uniform baseline, while one-percent global label corruption changes the behavior substantially (right).  Noise and design experiments are exploratory deviations from the theorem.}
    \label{fig:app-scope-checks}
\end{figure}
\textbf{Additional population and initialization checks.}
The main population experiment follows one complete transient.  Figure~\ref{fig:app-population-robustness} repeats the post-capture diagnostic across several task interiors and initial states and separately checks numerical quadrature and ODE convergence.  The stabilization of $\mu\sup_\alpha|F-\alpha|$ is insensitive to these tested choices, while the numerical discrepancies are negligible at the quadrature order used in the main experiment.

\begin{figure}[t]
    \centering
    \includegraphics[width=0.94\textwidth]{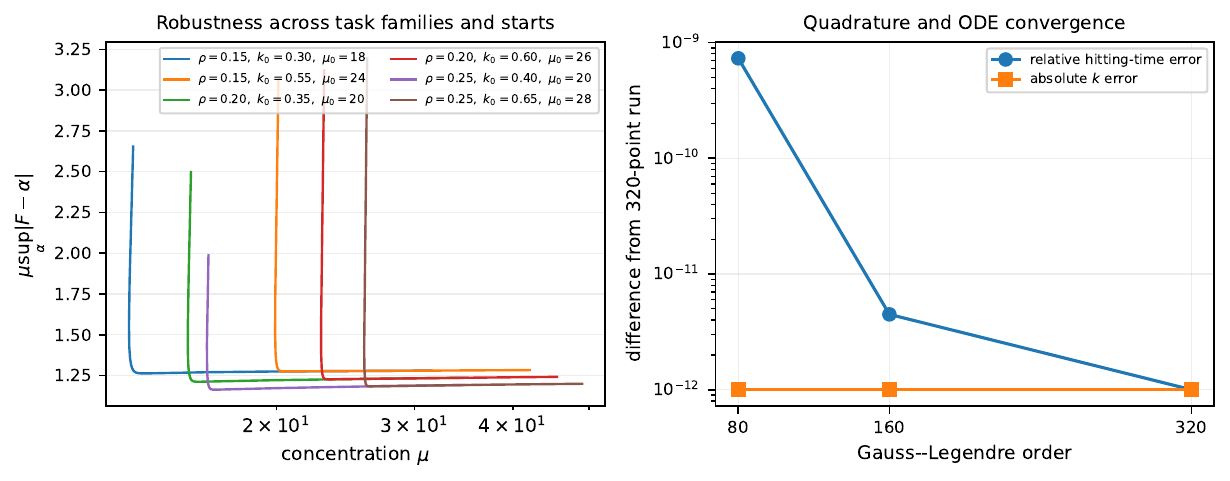}
    \caption{Population robustness and numerical convergence.  The scaled prediction error stabilizes across the tested task interiors and initial states (left).  Increasing the Gauss--Legendre order confirms convergence of the hitting time and ratio trajectory (right).}
    \label{fig:app-population-robustness}
\end{figure}
The large-radius assumption has a visible dynamical role.  At the origin the attention weights are uniform and symmetry makes the gradient vanish.  Figure~\ref{fig:app-initialization} shows linear vanishing of the gradient near this stationary point and delayed escape from nearby initializations.  Within the fixed training budget, performance is highly sensitive to the initial radius; the asymptotic behavior of every nonzero small initialization remains unresolved.

\begin{figure}[t]
    \centering
    \includegraphics[width=\textwidth]{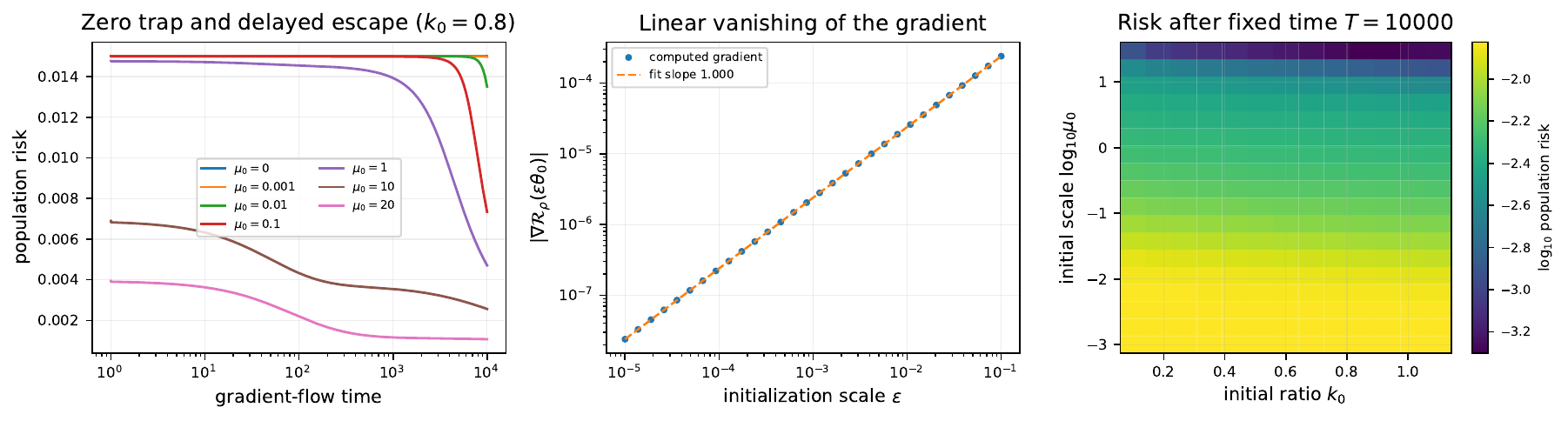}
    \caption{Initialization sensitivity.  Exactly zero initialization is stationary and nearby starts escape slowly (left); the gradient norm vanishes linearly with initialization scale (center); and the risk after a common time budget depends strongly on the initial radius (right).}
    \label{fig:app-initialization}
\end{figure}
\subsection{Exploratory label-noise}
\label{app:noise-robustness}
To test the robustness of a pretrained model, we use a frozen-checkpoint protocol: the checkpoint is selected before the noisy fresh context is observed and the same parameters are reused for every context size $N$. This is our primary robustness experiment. As an auxiliary mechanism-level diagnostic, we also evaluate a resolution-matched sequence with $\mu=N/8$ to ask whether corruption localized at the learned resolution preserves the predicted scaling. In this sequence, localized noise flips labels with probability $1/2$ on $|X-\alpha|\le1/\mu$, the tapered condition uses flip probability $\frac12\{1-(\mu|X-\alpha|)^2\}_+$, and global Massart noise flips labels independently with probability $0.01$. This sequence changes the parameters with $N$ and therefore tracks a family of models; the frozen-checkpoint protocol evaluates a single pretrained model.

For the frozen one-head test, we use $\widehat\mu=40$ and $\widehat k=0.86472$ for every context size.  Clean and threshold-localized conditions approach checkpoint-dependent population floors, whereas global Massart flips produce order-one error (Figure~\ref{fig:app-one-head-frozen-noise}).

\begin{figure}[t]
    \centering
    \includegraphics[width=0.72\textwidth]{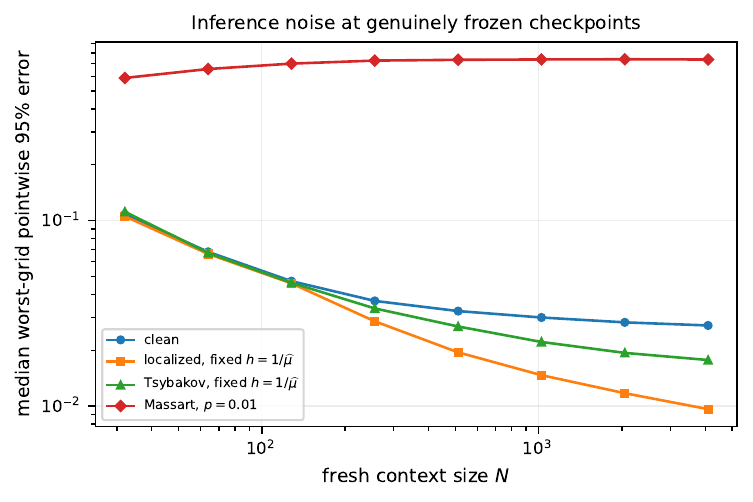}
    \caption{Inference-time noise at a genuinely frozen one-head checkpoint. Clean and threshold-localized conditions improve with fresh-context size and then approach finite-resolution floors.  Homogeneous one-percent Massart flips cause persistent order-one error.}
    \label{fig:app-one-head-frozen-noise}
\end{figure}

Along the resolution-matched sequence, clean, localized, and Tsybakov errors decrease approximately as $N^{-1}$ on the safe interior, while the fixed one-head orientation retains its physical-boundary obstruction.  We use this coupled scaling as a mechanism-level diagnostic of resolution-compatible corruption across the model sequence.

\begin{figure}[t]
    \centering
    \includegraphics[width=0.94\textwidth]{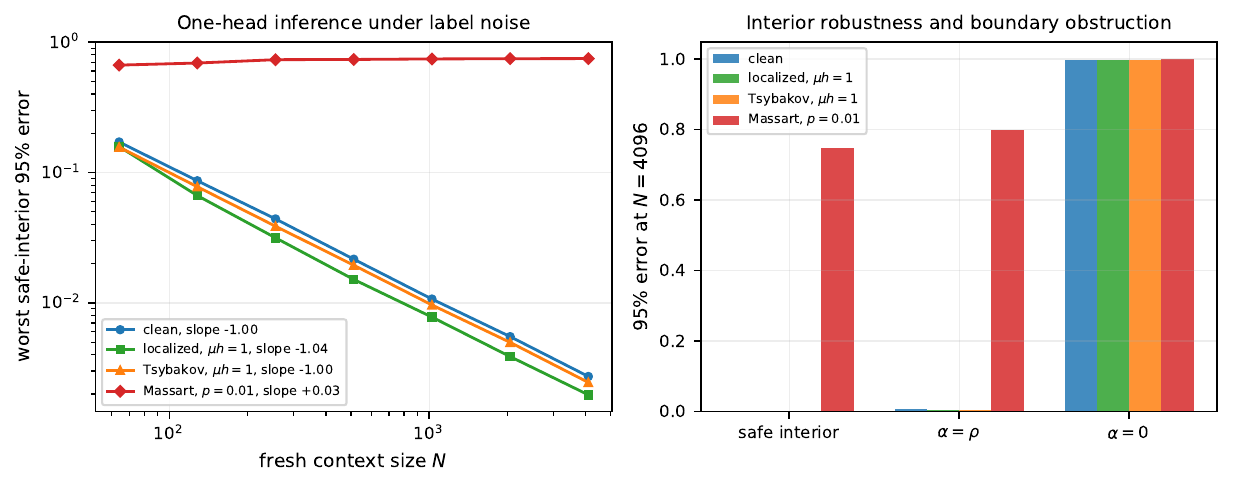}
    \caption{Resolution-matched one-head noise stress test.  Localized and Tsybakov corruption preserve the clean interior trend when $\mu=N/8$ (left), while the exact one-head boundary obstruction persists across context sizes and localized corruption (right).}
    \label{fig:app-one-head-noise}
\end{figure}
\end{document}